\documentclass[twoside,11pt]{article}
\usepackage{jmlr2e}
\usepackage{natbib,subfigure}
\usepackage{pdfsync}
\usepackage{lipsum,booktabs}
\usepackage{amsmath,mathrsfs,amssymb,amsfonts,bm}
\usepackage{paralist}
\usepackage{rotating}
\usepackage{pdflscape}
\usepackage{tcolorbox}
\usepackage{hyperref,url,dsfont,nicefrac}
\hypersetup{
    colorlinks,
    breaklinks,
    linkcolor = blue,
    citecolor = blue,
    urlcolor  = blue,
}
\allowdisplaybreaks[2]
\usepackage{appendix}
\usepackage{multirow,makecell}
\usepackage{graphicx,color} 
\usepackage{caption2}
\usepackage{standalone}     
\usepackage{preview}
\usepackage{xcolor}
\usepackage{setspace}
\usepackage{lmodern}
\usepackage{algorithmic,algorithm}

\renewcommand{\tilde}{\widetilde}
\renewcommand{\hat}{\widehat}

\def \B {\mathbb{B}}
\def \B {\mathcal{B}}

\def \D {\mathcal{D}}
\def \E {\mathbb{E}}

\def \H {\mathcal{H}}
\def \I {\mathcal{I}}

\def \O {\mathcal{O}}

\def \R {\mathbb{R}}

\def \T {\top}

\def \X {\mathcal{X}}
\def \Y {\mathcal{Y}}

\def \a {\mathbf{a}}

\def \c {\mathbf{c}}

\def \e {\mathbf{e}}

\def \m {\mathbf{m}}

\def \p {\mathbf{p}}
\def \q {\mathbf{q}}

\def \u {\mathbf{u}}

\def \x {\mathbf{x}}
\def \y {\mathbf{y}}

\def \xh {\hat{\x}}
\def \xb {\bar{\x}}

\def \Rcal {\mathcal{R}}

\def \ellb {\boldsymbol{\ell}}

\usepackage{inconsolata} 

\def \base {\mathtt{base}\mbox{-}\mathtt{regret}}
\def \meta {\mathtt{meta}\mbox{-}\mathtt{regret}}
\def \epsilon {\varepsilon}

\usepackage{mathtools}
\let\norm\undefined 
\DeclarePairedDelimiter\norm{\lVert}{\rVert}
\DeclarePairedDelimiter\abs{\lvert}{\rvert}
\newcommand\inner[2]{\langle #1, #2 \rangle}

\newcommand \term[1]{\mathtt{term}~(\mathtt{#1})}

\def \p {\boldsymbol{p}}
\def \m {\boldsymbol{m}}

\DeclareMathOperator*{\poly}{poly}

\DeclareMathOperator*{\argmin}{arg\,min}

\let\proof\relax
  
\usepackage{amsthm}
\let\proof\relax

\newenvironment{proof}{\par\noindent{\textbf{Proof}\ }}{\hfill\BlackBox\\[2mm]}

\newtheorem{myThm}{Theorem}

\newtheorem{myLemma}{Lemma}

\theoremstyle{definition}
\newtheorem{myAssum}{Assumption}

\newtheorem{myRemark}{Remark}

\newtheorem{myInstance}{Instance}

\def\endenv{\hfill\raisebox{1pt}{$\P$}\smallskip}

\definecolor{wine_red}{RGB}{228,48,64}
\definecolor{DSgray}{cmyk}{0,1,0,0}

\newcommand \Div[2]{\D_{\Rcal}(#1, #2)}
\def \Rcal {\psi}

\def \q {\boldsymbol{q}}
\def \cb {\boldsymbol{c}}
\def \psib {\boldsymbol{\psi}}
\def \Psib {\boldsymbol{\Psi}}

\usepackage{lastpage}
\jmlrheading{25}{2024}{1-\pageref{LastPage}}{7/21; Revised
11/23}{3/24}{21-0748}{Peng Zhao, Yu-Jie Zhang, Lijun Zhang, and Zhi-Hua Zhou}

\ShortHeadings{Adaptivity and Non-stationarity: Problem-dependent Dynamic Regret for OCO}{Zhao, Zhang, Zhang, Zhou}
\firstpageno{1}

\begin{document}

\title{Adaptivity and Non-stationarity: Problem-dependent Dynamic Regret for Online Convex Optimization}

\author{\name Peng Zhao \email zhaop@lamda.nju.edu.cn \\
       \name Yu-Jie Zhang \email zhangyj@lamda.nju.edu.cn \\
       \name Lijun Zhang \email zhanglj@lamda.nju.edu.cn \\
       \name Zhi-Hua Zhou \email zhouzh@lamda.nju.edu.cn \\
       \addr National Key Laboratory for Novel Software Technology\\
       Nanjing University, Nanjing 210023, China}

\editor{Csaba Szepesvari}

\maketitle

\begin{abstract} 
We investigate online convex optimization in non-stationary environments and choose \emph{dynamic regret} as the performance measure, defined as the difference between cumulative loss incurred by the online algorithm and that of any feasible comparator sequence. Let $T$ be the time horizon and $P_T$ be the path length that essentially reflects the non-stationarity of environments, the state-of-the-art dynamic regret is $\O(\sqrt{T(1+P_T)})$. Although this bound is proved to be minimax optimal for convex functions, in this paper, we demonstrate that it is possible to further enhance the guarantee for some easy problem instances, particularly when online functions are smooth. Specifically, we introduce novel online algorithms that can exploit smoothness and replace the dependence on $T$ in dynamic regret with \mbox{\emph{problem-dependent}} quantities: the variation in gradients of loss functions, the cumulative loss of the comparator sequence, and the minimum of these two terms. These quantities are at most $\O(T)$ while could be much smaller in benign environments. Therefore, our results are adaptive to the intrinsic difficulty of the problem, since the bounds are tighter than existing results for easy problems and meanwhile safeguard the same rate in the worst case. Notably, our proposed algorithms can achieve favorable dynamic regret with only \emph{one} gradient per iteration, sharing the same gradient query complexity as the static regret minimization methods. To accomplish this, we introduce the \emph{collaborative online ensemble} framework. The proposed framework employs a two-layer online ensemble to handle non-stationarity, and uses optimistic online learning and further introduces crucial correction terms to enable effective collaboration within the meta-base two layers, thereby attaining adaptivity. We believe the framework can be useful for broader problems.
\end{abstract}



\section{Introduction}
In many real-world applications, data are inherently accumulated over time, and thus it is of great importance to develop a learning system that updates in an online fashion. Online Convex Optimization (OCO)~\citep{book'16:Hazan-OCO,arxiv'19:online-learning-modern-intro} is a powerful paradigm for learning in such scenarios, which can be regarded as an iterative game between a player and an adversary. At iteration $t$, the player chooses a decision vector $\x_t$ from a convex set $\X \subseteq \R^d$. Subsequently, the adversary discloses a convex function $f_t: \X \mapsto \R$, and the player incurs a loss denoted by $f_t(\x_t)$. The standard performance measure is the (static) \emph{regret}~\citep{ICML'03:zinkvich},
\begin{equation}
  \label{eq:static-regret}
  \mbox{S-Regret}_T = \sum_{t=1}^T f_t(\x_t) - \min_{\x\in \mathcal{X}} \sum_{t=1}^T f_t(\x),
\end{equation}
which is the difference between cumulative loss incurred by the online algorithm and that of the best decision in hindsight. The rationale behind such a metric is that the best fixed decision in hindsight is reasonably good over all the iterations. However, this might be too optimistic and may not hold in non-stationary environments, where data are evolving and the optimal decision is drifting over time. To address this limitation, \emph{dynamic regret} is proposed to compete with changing comparators $\u_1,\dots,\u_T\in\mathcal{X}$,
\begin{equation}
  \label{eq:universal-dynamic-regret}
      \mbox{D-Regret}_T(\u_1,\dots,\u_T) = \sum_{t=1}^T f_t(\x_t) -  \sum_{t=1}^T f_t(\u_t),
\end{equation}
which draws considerable attention recently~\citep{NIPS'18:Zhang-Ader,NIPS'20:sword,ICML'20:Ashok,JMLR'21:BCO,COLT'21:baby-exp-concave,NeurIPS:2021:Zhang:A,NeurIPS'22:efficient,JMLR'23:memory}. The measure is also called the \emph{universal} dynamic regret (or \emph{general} dynamic regret), in the sense that it gives a universal guarantee that holds against \emph{any} comparator sequence. Note that the static regret~\eqref{eq:static-regret} can be viewed as its special form by choosing comparators as the fixed best decision in hindsight. Moreover, a variant appeared frequently in the literature is called the \emph{worst-case dynamic regret}~\citep{OR'15:dynamic-function-VT,AISTATS'15:dynamic-optimistic,CDC'16:dynamic-sc,ICML'16:Yang-smooth,NIPS'16:Wei-non-stationary-expert,NIPS'17:zhang-dynamic-sc-smooth,NIPS'19:Wangyuxiang,AAAI'20:Jianjun,AISTATS'20:restart,AISTATS:2020:Zhang,UAI'20:simple,L4DC'21:sc_smooth}, defined as 
\begin{equation}
    \label{eq:worst-case-dynamic-regret}
    \mbox{D-Regret}_T(\x^*_1,\ldots,\x^*_T) = \sum_{t=1}^T f_t(\x_t) -  \sum_{t=1}^T f_t(\x^*_t),
\end{equation}
which specializes the general form~\eqref{eq:universal-dynamic-regret} with comparators $\u_t = \x_t^* \in \argmin_{\x\in \mathcal{X}} f_t(\x)$. Therefore, universal dynamic regret is very general and can include the static regret~\eqref{eq:static-regret} and the worst-case dynamic regret~\eqref{eq:worst-case-dynamic-regret} as special cases by different instantiations of comparators. We further remark that the worst-case dynamic regret is often too pessimistic, whereas the universal one is more adaptive to non-stationary environments. Actually, changes of online functions usually come from two sources: sampling randomness and environmental non-stationarity, with the latter being the primary concern in non-stationary online learning. Optimizing the worst-case dynamic regret can be problematic in certain scenarios. For instance, consider a stochastic optimization task where $f_t$'s are independently randomly sampled from the same distribution. Then, minimizing the worst-case dynamic regret is not suitable and can lead to overfitting~\citep{NIPS'18:Zhang-Ader}, as the minimizer of online function may significantly deviate from the minimizer of the expected function due to the sampling randomness. By contrast, since universal dynamic regret can accommodate any feasible comparator sequence, it can automatically  adapt to underlying distribution shifts.

There are many studies on the worst-case dynamic regret~\citep{OR'15:dynamic-function-VT,AISTATS'15:dynamic-optimistic,CDC'16:dynamic-sc,ICML'16:Yang-smooth,NIPS'17:zhang-dynamic-sc-smooth,ICML'18:zhang-dynamic-adaptive,NIPS'19:Wangyuxiang,UAI'20:simple,L4DC'21:sc_smooth}, but only few results are known for the universal dynamic regret.~\citet{ICML'03:zinkvich} shows that online gradient descent (OGD) with a step size $\eta>0$ achieves an $\O(({1+P_T})/{\eta}+\eta T)$ universal dynamic regret, where  
\begin{equation}
  \label{eq:path-length}
  P_T = \sum_{t=2}^{T} \norm{\u_{t} - \u_{t-1}}_2
\end{equation}
is the path length of comparators $\u_1,\ldots,\u_T$ and thus reflects the non-stationarity of environments. If the path length $P_T$ were known, one could choose the optimal step size $\eta_* = \Theta(\sqrt{(1+P_T)/T})$ and attain an $\O(\sqrt{T(1+P_T)})$ dynamic regret. However, this path length quantity is hard to know since the universal dynamic regret aims to provide guarantees against any feasible comparator sequence. The step size $\eta = \Theta (1/\sqrt{T})$ commonly used in static regret would lead to an inferior $\O(\sqrt{T}(1+P_T))$ bound, which exhibits a large gap from the favorable bound with an oracle step size tuning.~\citet{NIPS'18:Zhang-Ader} resolve the issue by proposing a novel online algorithm to search the optimal step size $\eta_*$, attaining an $\O(\sqrt{T(1+P_T)})$ universal dynamic regret, and they also establish an $\Omega(\sqrt{T(1+P_T)})$ lower bound to show the minimax optimality.

Although the rate is minimax optimal for convex functions, we would like to design algorithms with \emph{problem-dependent} regret guarantees beyond the worst-case analysis~\citep{roughgarden_2021}. Specifically, we aim to enhance the guarantee for some easy problem instances, particularly when the online functions are smooth, by replacing the dependence on $T$ by certain problem-dependent quantities that are $\O(T)$ in the worst case while could be much smaller in benign environments. For static regret mininimization, existing studies can attain such results like small-loss bounds~\citep{NIPS'10:smooth} and gradient-variation bounds~\citep{COLT'12:variation-Yang}. Thus, a natural question arises \emph{whether it is possible to achieve similar problem-dependent guarantees for universal dynamic regret?}

\paragraph{Our results.} In this paper, extending
our preliminary conference version~\citep{NIPS'20:sword}, we provide an affirmative answer by designing online algorithms with problem-dependent dynamic regret bounds. Specifically, we focus on the following two problem-dependent quantities: the gradient variation of online functions $V_T$, and the cumulative loss of the comparator sequence $F_T$, defined as
\begin{equation}
  \label{eq:gradient-variation}
  V_T = \sum_{t=2}^{T} \sup_{\x\in \X} \norm{\nabla f_t(\x) - \nabla f_{t-1}(\x)}_2^2, \text{  and  }~F_T = \sum_{t=1}^{T} f_t(\u_t).
\end{equation}
The two problem-dependent quantities are both at most $\O(T)$ under standard assumptions of online learning, while could be much smaller in easier problem instances. We propose two novel online algorithms called \textsf{Sword} and \textsf{Sword++} (``Sword'' is short for \underline{S}moothness-a\underline{w}are \underline{o}nline lea\underline{r}ning with \underline{d}ynamic regret) that are suitable for different feedback models. Our algorithms are online ensemble methods~\citep{book'12:ensemble-zhou,thesis:zhao2021-eng}, which admit a two-layer structure with a meta-algorithm running over a group of base-learners. We prove that they enjoy an $\O(\sqrt{(1+P_T + \min\{V_T,F_T\})(1+P_T)})$ dynamic regret, achieving gradient-variation and small-loss bounds simultaneously. Compared to the $\O(\sqrt{T(1+P_T)})$ minimax rate, our result replaces the dependence on $T$ by the problem-dependent quantity $P_T + \min\{V_T,F_T\}$. Our bounds become much tighter when the problem is easy, such as when both $P_T$ and $V_T$ (or $F_T$) exhibit sublinear growth in $T$. Meanwhile, our regret bounds can safeguard the same guarantee in the worst case. Hence, our results are adaptive to the intrinsic difficulty of problem instances as well as the non-stationarity of environments. 

Our first algorithm, \textsf{Sword}, achieves the favorable problem-dependent guarantees under the \emph{multi-gradient feedback model}, where the player can query gradient information multiple times at each round. This algorithm is conceptually simple, yet it requires a gradient query complexity of $\O(\log T)$ at each round. Our second algorithm, \textsf{Sword++}, improves upon this by necessitating only \emph{one} gradient per iteration, despite using a two-layer online ensemble structure. Therefore, Sword++ is not only computationally efficient but also more attractive due to its reduced feedback requirements --- it is particularly suitable for the \emph{one-gradient feedback model}, in which the player receives only the gradient $\nabla f_t(\x_t)$ after submitting the decision $\x_t$. Therefore, Sword++ has the potential to be extended to more constrained bandit feedback models.

\paragraph{Technical contributions.} Note that existing studies have demonstrated that the worst-case dynamic regret can benefit from smoothness~\citep{ICML'16:Yang-smooth,NIPS'17:zhang-dynamic-sc-smooth,L4DC'21:sc_smooth}. However, their analyses do not apply to our concerned universal dynamic regret, because we cannot exploit the optimality condition of comparators $\u_1,\ldots,\u_T$, in stark contrast with the worst-case dynamic regret analysis. To address this, we propose an adaptive \emph{online ensemble} method to hedge non-stationarity while extracting adaptivity. Our method incorporates a meta-base two-layer ensemble to hedge the non-stationarity and employs optimistic online learning for adaptive reuse of historical gradient information. Two crucial novel ingredients are designed to achieve favorable problem-dependent guarantees.
\begin{itemize}
  \item We introduce optimistic online mirror descent (\textsc{Optimistic OMD}) as a unified building block for the algorithm design of dynamic regret minimization at both meta and base levels.\footnote{For the meta-algorithm, we only care about its static regret, which is essentially a special case of the universal dynamic regret. When the meta-algorithm implements Hedge-style updates with changing learning rates, it aligns more with optimistic FTRL (Follow-The-Regularized-Leader) instead of optimistic OMD. However, we prioritize discussing optimistic OMD (hence using a fixed learning rate), intentionally to better illustrate the core ideas of our regret analysis and algorithm design.} We present generic and completely modular analysis for the dynamic regret of Optimistic OMD, where the \emph{negative term} is essential especially for achieving problem-dependent dynamic regret guarantees.
  \item We implement an adaptive online ensemble method that combines optimistic online learning for attaining adaptivity with a meta-base structure to hedge non-stationarity. A key innovation is the emphasis on \emph{collaboration} within the online ensemble. In our \emph{collaborative online ensemble} framework, we introduce a novel decision-deviation correction term in algorithm design and simultaneously exploit the negative term in regret analysis, facilitating effective collaboration within two layers, which is crucial for achieving desired problem-dependent bounds with only one gradient per iteration.
\end{itemize}
We emphasize that these ingredients are particularly important for achieving gradient-variation dynamic regret, which we will demonstrate to be more fundamental than the small-loss bound. In particular, our proposed Sword++ algorithm effectively utilizes negative terms and introduces correction terms to ensure effective collaboration within the two layers. The overall framework of collaborative online ensemble is summarized in Section~\ref{sec:framework-collaborative-OE}, and we believe that the proposed framework has the potential for broader online learning problems.

\paragraph{Organization.} The rest is structured as follows. Section~\ref{sec:related-work} briefly reviews  related works. In Section~\ref{sec:problem-setup-OMD}, we introduce the problem setup and the optimistic online mirror descent framework, where a general dynamic regret analysis is provided. Section~\ref{sec:gradient-variation} establishes the gradient-variation dynamic regret bounds under two different gradient feedback models. Section~\ref{sec:framework-collaborative-OE} illustrates our proposed collaborative online ensemble framework, which is very general and useful for attaining problem-dependent dynamic regret. Section~\ref{sec:implications} provides some additional results regarding implications, significance and a lower bound. The major proofs are presented in Section~\ref{sec:appendix-analysis}. Furthermore, Section~\ref{sec:experiemnts} reports the experiments. Finally, we conclude the paper in Section~\ref{sec:conclusion}. Omitted proofs are provided in the appendix.

\section{Related Work}
\label{sec:related-work}
In this section, we present a brief overview of both static and dynamic regret minimization in the context of online convex optimization. Additionally, we provide more discussions on the subsequent studies after the preprint of our manuscript is publicly available.

\subsection{Static Regret}
\label{sec:related-work-static-regret}
Static regret has been extensively studied in online convex optimization. Let $T$ be the time horizon and $d$ be the dimension, there exist online algorithms with static regret bounded by $\O(\sqrt{T})$, $\O(d\log T)$, and $\O(\log T)$ for convex, exponentially concave, and strongly convex functions, respectively~\citep{ICML'03:zinkvich,journals/ml/HazanAK07}. These results are proved to be minimax optimal~\citep{conf/colt/AbernethyBRT08}. More results can be found in the seminal books~\citep{book'12:Shai-OCO,book'16:Hazan-OCO} and references therein.

In addition to exploiting the convexity of functions, there are studies improving static regret by incorporating smoothness, whose main proposal is to replace the dependence on $T$ by problem-dependent quantities. Such problem-dependent bounds enjoy many benign properties, in particular, they can safeguard the worst-case minimax rate yet can be much tighter in easier problem instances. There are two representative problem-dependent bounds --- small-loss bound~\citep{NIPS'10:smooth} and gradient-variation bound~\citep{COLT'12:variation-Yang}. 

Small-loss bounds are first introduced in the context of prediction with expert advice~\citep{journals/iandc/LittlestoneW94,JCSS'97:boosting}, which replace the dependence on $T$ by cumulative loss of the best expert. Later,~\citet{NIPS'10:smooth} show that in the online convex optimization setting, OGD with a certain step size scheme can achieve an $\O(\sqrt{1 + F^*_T})$ small-loss regret bound when the online convex functions are smooth and non-negative, where $F^*_T$ is the cumulative loss of the best decision in hindsight, namely, $F^*_T = \sum_{t=1}^{T} f_t(\x^*)$ with $\x^*$ chosen as the offline minimizer. The key technical ingredient is to exploit the self-bounding property of smooth functions. Gradient-variation bounds are introduced by~\citet{COLT'12:variation-Yang}, rooting in the development of second-order bounds for prediction with expert advice~\citep{COLT'05:second-order-Hedge} and online convex optimization~\citep{COLT'08:Hazan-variation}. For convex and smooth functions, \citet{COLT'12:variation-Yang} establish an $\O(\sqrt{1 + V_T})$ gradient-variation regret bound, where $V_T = \sum_{t=2}^{T} \sup_{\x\in \X} \norm{\nabla f_t(\x) - \nabla f_{t-1}(\x)}_2^2$ measures the cumulative gradient variation. Gradient-variation bounds are particularly favored in slowly changing environments where online functions evolve gradually. Furthermore, the techniques developed for gradient-variation regret bounds have a profound connection to many other learning problems, including repeated games~\citep{NIPS'15:fast-rate-game} and stochastic optimization~\citep{NIPS'22:SEA-model}.

In addition, problem-dependent static regret bounds are also studied in the bandit online learning setting, including gradient-variation bounds for two-point bandit convex optimization~\citep{COLT'13:Chiang}, as well as small-loss bounds for multi-armed bandits~\citep{ALT'06:Green,COLT'18:adaptive-bandits,NIPS'20:unbiased-bandits}, linear bandits~\citep{NIPS'20:unbiased-bandits}, semi-bandits~\citep{COLT'15:Neu-small-bandits}, graph bandits~\citep{COLT'18:Lykouris-small-loss,COLT'20:small-graph}, and contextual bandits~\citep{ICML'18:Maga,NIPS'21:dylan-first-order}, etc.

Finally, we mention that problem-dependent regret minimization falls under the wider umbrella of \emph{adaptive online convex optimization}~\citep{COLT'10:adaptiveOCO,COLT'10:AdaGrad}, with more recent explorations discussed in~\citep{JMLR'17:adaptive-online,TCS'20:unified-framework,NIPS'20:Ashok-matrixRegret} and the monograph~\citep{arxiv'19:online-learning-modern-intro}. However, in addition to developing problem-dependent bounds, this field also covers \emph{data-dependent} bounds. A caveat is that these data-dependent bounds (or called ``\emph{algorithm-dependent} bounds'') might be influenced not only by the complexity of the problem instance but also by the dynamics of the algorithm itself. This can be sometimes undesired, particularly when the data-dependent quantity is not appropriated defined, potentially leading to a misleading representation of the learning problem's difficulty.
For instance, if the regret upper bound depends on a data-dependent quantity like $\sum_{t=2}^T \norm{\nabla f_t(\x_t) - \nabla f_{t-1}(\x_{t-1})}_2^2$ (rather than the problem-dependent quantity $\sum_{t=2}^T \sup_{\x \in \X}\norm{\nabla f_t(\x) - \nabla f_{t-1}(\x)}_2^2$), the regret bound becomes affected by the algorithm's decision sequence $\x_1,\ldots,\x_T$. This will misleadingly lead to large bounds in scenarios where the function sequence is constant ($f_1=\ldots=f_T = f$) but the decision sequence is unstable.

\subsection{Dynamic Regret}
\label{sec:related-work-dynamic-regret}
Dynamic regret enforces the player to compete with time-varying comparators and thus is favored in online learning in open and non-stationary environments~\citep{book/mit/sugiyama2012machine,TKDE'21:DFOP,NSR'22:Zhou-OpenML}. The notion of dynamic regret is sometimes referred to as tracking regret/switching regret/shifting regret in the prediction with expert advice setting~\citep{journals/ml/HerbsterW98,JMLR'01:Herbster,JMLR'02:bousquet-dynamic,conf/nips/Cesa-BianchiGLS12,ICML'16:GyorgyS-shiftregret}. It is known that in the worst case, a sublinear dynamic regret is not attainable unless imposing certain regularities on the comparator sequence or the function sequence~\citep{OR'15:dynamic-function-VT,AISTATS'15:dynamic-optimistic}. This paper focuses on the most common regularity called the \emph{path length} $P_T = \sum_{t=2}^{T} \norm{\u_{t-1} - \u_{t}}_2$ introduced by~\citet{ICML'03:zinkvich}, which measures fluctuation of the comparators. We simply focus on the Euclidean norm throughout this paper, and it is straightforward to extend the notions and results to general primal-dual norms. Other regularities include the squared path length $S_T = \sum_{t=2}^{T} \norm{\u_{t-1} - \u_{t}}_2^2$ introduced by~\citet{NIPS'17:zhang-dynamic-sc-smooth}, and the function variation introduced by~\citet{OR'15:dynamic-function-VT} that measures the cumulative variation with respect to the function value and is defined as $V^f_T = \sum_{t=2}^{T} \sup_{\x\in \X} \abs{f_{t-1}(\x) - f_t(\x)}$.

There are two kinds of dynamic regret notions in the previous studies. The universal dynamic regret, as defined in~\eqref{eq:universal-dynamic-regret}, aims to compare with any feasible comparator sequence, while the worst-case dynamic regret defined in~\eqref{eq:worst-case-dynamic-regret} specifies the comparator sequence to be the sequence of minimizers of online functions. In the following, we present the related works respectively. Notice that we will use notations  $P_T$ and $S_T$ for path length and squared path length of the comparator sequence $\{\u_t\}_{t=1,\ldots,T}$, while adopt the notations $P_T^*$ and $S_T^*$ for that of the sequence $\{\x_t^*\}_{t=1,\ldots,T}$ where $\x^*_t$ is one of minimizers of the online function $f_t$, namely, $P_T^* = \sum_{t=2}^{T} \norm{\x^*_{t-1} - \x^*_{t}}_2$ and $S_T^* = \sum_{t=2}^{T} \norm{\x^*_{t-1} - \x^*_{t}}^2_2$.

\paragraph{Universal dynamic regret.} The pioneering work of~\citet{ICML'03:zinkvich} demonstrates that online gradient descent (OGD) enjoys an $\O(\sqrt{T}(1+P_T))$ universal dynamic regret, which holds against any feasible comparator sequence. Nevertheless, the result is far from the $\Omega(\sqrt{T(1+P_T)})$ lower bound established by~\citet{NIPS'18:Zhang-Ader}, who further close the gap by proposing a novel online algorithm that attains an optimal rate of $\O(\sqrt{T(1+P_T)})$ for convex functions~\citep{NIPS'18:Zhang-Ader}. Our work further exploits the easiness of the problem instances and achieves problem-dependent regret guarantees, hence better than the minimax rate. \citet{JMLR'21:BCO} study the universal dynamic regret for bandit convex optimization under both one-point and two-point feedback models. Concurrent to our conference version paper~\citep{NIPS'20:sword}, \citet{ICML'20:Ashok} proposes a novel online algorithm that achieves the same order of minimax optimal dynamic regret for convex functions as~\citep{NIPS'18:Zhang-Ader}, yet without relying on using a meta-algorithm hedging over a group of base learners. Instead, their method employs the combination strategy developed in parameter-free online learning~\citep{COLT'18:black-box-reduction,COLT'19:easy-combine}. Note that it may be possible to modify the algorithm of~\citet{ICML'20:Ashok} to achieve small-loss bounds; however, attaining gradient-variation bounds would be generally challenging, especially under the one-gradient feedback model. More specifically, it is not hard to modify their framework to incorporate optimistic online learning, but one usually needs to exploit additional negative terms to convert the optimistic quantity $\norm{\nabla f_t(\x_t) - \nabla f_{t-1}(\x_{t-1})}_2^2$ to gradient variation $\sup_{\x \in \X}\norm{\nabla f_t(\x) - \nabla f_{t-1}(\x)}_2^2$, to eliminate the difference between decisions $\x_t$ and $\x_{t-1}$. Our algorithms, based on the collaborative online ensemble framework, involve a careful exploitation of negative terms in the regret analysis of both meta and base algorithms, alongside introducing correction terms in the algorithm design. However, as far as we can see, with only one gradient feedback per round, it is challenging for the framework of~\citet{ICML'20:Ashok} to achieve the gradient-variation bound due to the lack of negative terms in their regret analysis.

\paragraph{Worst-case dynamic regret.} There are many efforts devoted to studying the worst-case dynamic regret. \citet{ICML'16:Yang-smooth} prove that OGD enjoys an $\O(\sqrt{T(1 + P_T^*)})$ worst-case dynamic regret for convex functions when the path length $P_T^*$ is known. For strongly convex and smooth functions,~\citet{CDC'16:dynamic-sc} show that an $\O(P_T^*)$ dynamic regret is achievable, and~\citet{NIPS'17:zhang-dynamic-sc-smooth} further propose the online multiple gradient descent algorithm with an $\O(\min\{P_T^*,S_T^*\})$ guarantee. \citet{ICML'16:Yang-smooth} show that $\O(P_T^*)$ rate is attainable for convex and smooth functions, provided that all the minimizers $\x_t^*$'s lie in the interior of the domain $\X$. The above results mainly use the (squared) path length as the non-stationarity measure, which measures the cumulative variation of the comparator sequence. In another line of research, researchers use the variation with respect to the function values as the measure.~\citet{OR'15:dynamic-function-VT} show that OGD with a restarting strategy attains an $\O(T^{2/3}{V_T^{f 1/3})}$ regret for convex functions when the function variation $V^f_T$ is available, which is improved to $\O(T^{1/3}{V_T^{f 2/3}})$ for $1$-dim square loss~\citep{NIPS'19:Wangyuxiang}. \citet{OR'19:V_T-pq} extend the results of~\citet{OR'15:dynamic-function-VT} to more general function-variation measures capable of capturing local temporal and spatial changes. To take advantage of variations in both comparator sequences and function values, \citet{L4DC'21:sc_smooth} provide an improved analysis for online multiple gradient descent and prove an $\O(\min\{P_T^*,S_T^*, V_T^f\})$ worst-case dynamic regret for strongly convex and smooth functions. For convex and smooth functions, it is also demonstrated that a simple greedy strategy, i.e., $\x_{t+1} = \x_t^* \in \argmin_{\x \in \X} f_t(\x)$, can effectively optimize the worst-case dynamic regret, as shown in~\citep[Section 4.2]{L4DC'21:sc_smooth}.

\subsection{More Discussions}

\paragraph{Subsequent works for dynamic regret minimization.} There are many developments for dynamic regret minimization after our work became publicly available~\citep{JMLR:sword++}, and we briefly mention a few here. For exp-concave or strongly convex online functions, optimal dynamic regret can be obtained by algorithms minimizing strongly adaptive regre~\citep{COLT'21:baby-exp-concave,AISTATS'22:sc-proper}. Dynamic regret of decision-theoretic online learning is substantially explored, including online non-stochastic control~\citep{AISTATS'22:scream,AISTATS'22:adaptive-control,NeurIPS'22:LQR-dynamic}, online MDPs~\citep{NIPS'20:dynamic-fei,ICML'22:mdp,NeurIPS'23:linearMDP}, and online games~\citep{ICML'22:TVgame,ICML'23:SMontoneGame,ICLR'23:meta-games}. Furthermore, related techniques have been applied to online label shift~\citep{NeurIPS'22:label_shift} and online covariate shift~\citep{NeurIPS'23:covariate_shift}. The efficiency issue regarding the projection complexity of two-layer online ensemble is considered in~\citep{NeurIPS'22:efficient}.

\paragraph{Subsequent works employing the collaborative online ensemble.} A pivotal technique in our paper is the collaborative online ensemble framework, which effectively facilitates the \emph{collaboration} between meta and base layers by incorporating correction terms in the algorithm design and exploiting negative terms in the regret analysis. We have found this collaboration  crucial for a variety of problems involving deploying a two-layer structure. We mention two particular examples raised in the literature after our result became available, including game theory~\citep{ICML'22:TVgame,ICML'23:SMontoneGame} and an intermediate model for bridging stochastic and adversarial optimization~\citep{ICML'23:OMD4SEA,arxiv'23:OMD4SEA-journal}.
\begin{itemize}
	\item \citet{ICML'22:TVgame} investigate time-varying zero-sum games, introducing individual regret, dynamic NE-regret, and duality gap as the joint performance measures to guide algorithmic design. To handle multiple performance requirements, they deploy a two-layer algorithm for each player, demonstrating that the overall algorithm enjoys favorable regret guarantees. A vital component in their algorithm is to facilitate  collaborations between the meta and base layers. This is again achieved by injecting correction terms in base level and exploiting negative terms in regret analysis, as well as leveraging the unique structure of the zero-sum games. These results are further generalized to strongly monotone games~\citep{ICML'23:SMontoneGame}. 
	\item \citet{ICML'23:OMD4SEA} investigate the Stochastically Extended Adversarial (SEA) model, initially proposed in~\citep{NIPS'22:SEA-model}, serving as an intermediate model to bridge stochastic and adversarial convex optimization. They enhance the theoretical guarantees of~\citep{NIPS'22:SEA-model} by a careful analysis using  optimistic online mirror descent. Furthermore, they generalize the results by considering dynamic regret minimization for the SEA model, accommodating potential distribution shifts. Consequently, they implement a two-layer algorithm similar to Sword++, ensuring a favorable regret. Note that in the SEA model, due to the dependence issue of the random variables, it is necessary to use the collaborative online ensemble like \textsf{Sword++}. As highlighted in~\citep[Remark 10]{arxiv'23:OMD4SEA-journal}, deploying an algorithm similar to \textsf{Sword} to the SEA model will fail to yield desired regret bounds, primarily due to the dependence issue introduced by employing intermediate decisions in the meta-base structure.
\end{itemize}

\section{Problem Setup and Algorithmic Framework}
\label{sec:problem-setup-OMD}
In this section, we first formally state the problem setup, then introduce the foundational algorithmic framework for dynamic regret minimization, and finally list several assumptions that might be used in the theoretical analysis.

\subsection{Problem Setup}
\label{sec:problem}
Online Convex Optimization (OCO) can be modeled as an iterated game between the player and the environments. At iteration $t \in [T]$, the player first chooses the decision $\x_t$ from a convex feasible set $\X \subseteq \R^d$, then the environments reveal the loss function $f_t: \X \mapsto \R$ and the player suffers the loss $f_t(\x_t)$ and observes a certain information about the function $f_t(\cdot)$.\footnote{One may also understand this by defining $f_t$ over the entire $\R^d$ space while  constraining the decisions to the feasible domain $\X \subseteq \R^d$.} According to the revealed information, the online learning problems are typically classified into \emph{full-information} online learning and \emph{partial-information} online learning (or sometimes called \emph{bandit} online learning). In this paper, we focus on the full-information one, which can be further categorized into the following two setups:
\begin{itemize}
    \item[(i)] \textbf{multi-gradient feedback}: the player can access the entire gradient function $\nabla f_t(\cdot)$ and thus can evaluate the gradient multiple times at each round;
    \item[(ii)] \textbf{one-gradient feedback}: the player can observe the gradient information $\nabla f_t(\x_t)$ after submitting the decision $\x_t$ at each round.
\end{itemize}
In Section~\ref{sec:solution1}, we develop the \textsf{Sword} algorithm, which achieves the gradient-variation dynamic regret under the multi-gradient feedback model. In Section~\ref{sec:solution2}, we present an improved algorithm called \textsf{Sword++} that can achieve the same dynamic regret guarantee (up to constants) under the more challenging one-gradient feedback model. 

To handle non-stationary environments, we focus on the \emph{dynamic regret} measure, which compares the online algorithm to a sequence of time-varying comparators $\u_1,\ldots,\u_T \in \X$, as defined in~\eqref{eq:universal-dynamic-regret}. An upper bound of dynamic regret should be a function of comparators, and typically the bound depends on some regularities that measure the fluctuation of the comparator sequence, such as the path length $P_T = \sum_{t=2}^{T} \norm{\u_t - \u_{t-1}}_2$. Throughout the paper, we focus on the Euclidean norm for simplicity, and it is straightforward to extend our results to general primal-dual norms.

In addition to the regret measure, we further consider the \emph{gradient query complexity}. Note that algorithms designed for the multi-gradient feedback model may query the gradients for multiple times at each round. However, most algorithms designed for the static regret minimization only require \emph{one gradient per iteration}, namely, using $\nabla f_t(\x_t)$ for the next update only. Therefore, it is more desirable to achieve the favorable regret guarantees under the one-gradient feedback model. In other words, our aim is to develop first-order methods for dynamic regret minimization that require only one gradient query per iteration.

\subsection{Optimistic Online Mirror Descent}
\label{sec:OMD-framework}
We employ the algorithmic framework of Optimistic Online Mirror Descent (\textsc{Optimistic OMD})~\citep{COLT'12:variation-Yang,conf/colt/RakhlinS13} as a general building block for designing algorithms for non-stationary online learning. Optimistic OMD is an algorithmic realization of optimistic online learning. Compared to the standard online learning setup, the player receives an additional element at each round: an optimistic vector $M_t \in \R^d$. This vector acts as a predictive hint or an optimistic estimate of the upcoming gradient, thereby called ``optimistic vector'' or simply ``optimism''. Optimistic OMD starts from the initial point $\xh_1 \in \X$ and performs the following two-step updates at each round:
\begin{equation}
    \label{eq:optimistic-OMD-fix}
    \begin{split}
        \x_{t} = {} & \argmin_{\x \in \X}~\big\{\eta_{t} \inner{M_{t}}{\x} + \D_{\psi}(\x,\xh_{t})\big\}, \vspace{2mm}\\
    \xh_{t+1} = {} & \argmin_{\x \in \X}~\big\{\eta_t \inner{\nabla f_{t}(\x_{t})}{\x} + \D_{\psi}(\x,\xh_{t})\big\},
    \end{split}
\end{equation}
which firstly updates by the optimistic vector $M_t$ and then updates by the received gradient $\nabla f_t(\x_t)$. In above, $\eta_t > 0$ is a (potentially) time-varying step size, and $\Div{\cdot}{\cdot}$ denotes the Bregman divergence associated with the regularizer $\psi$ defined as $\Div{\x}{\y} = \psi(\x) - \psi(\y) - \inner{\nabla \psi(\y)}{\x -\y}$. We may assume the regularizer to be $\sigma$-strongly convex with respect to the norm $\|\cdot\|$, i.e., $\psi(\y) \geq \psi(\x) + \inner{\nabla \psi(\x)}{\y - \x} + \frac{\sigma}{2} \norm{\y-\x}^2$ holds for any $\x,\y \in \X$. We have the following general result regarding dynamic regret of optimistic OMD.
\begin{myThm}
\label{thm:dynamic-regret-OMD-generic} 
Suppose that the regularizer $\psi: \X \mapsto \R$ is 1-strongly convex with respect to the norm $\|\cdot\|$, and let $\|\cdot\|_*$ be the dual norm of $\|\cdot\|$. The dynamic regret of \textsc{Optimistic OMD} whose update rule is specified in~\eqref{eq:optimistic-OMD-fix} is bounded as follows: 
\begin{equation}
\label{eq:dynamic-regret-OMD-generic}
\begin{split}
\sum_{t=1}^{T} f_t(\x_t) - \sum_{t=1}^{T} f_t(\u_t) \leq \sum_{t=1}^{T}\eta_t \norm{\nabla f_t(\x_t) - M_t}_*^2 {} & +  \sum_{t=1}^{T} \frac{1}{\eta_t} \Big( \Div{\u_t}{\xh_t} - \Div{\u_t}{\xh_{t+1}}\Big) \\
{} & -  \sum_{t=1}^{T} \frac{1}{\eta_t}\Big( \Div{\xh_{t+1}}{\x_t} + \Div{\x_t}{\xh_{t}} \Big),
\end{split}
\end{equation}
which holds for any comparator sequence $\u_1,\ldots,\u_T \in \X$.
\end{myThm}

\begin{myRemark}
\label{remark:OMD}
The dynamic regret upper bound in Theorem~\ref{thm:dynamic-regret-OMD-generic} consists of three terms:
\begin{enumerate}
\item[(i)] the first term $\sum_{t=1}^{T} \eta_t \norm{\nabla f_t(\x_t) - M_t}_*^2$ is the \emph{adaptivity} term that measures the deviation between the gradient and optimistic vector;
\item[(ii)] the second term can be restructured as $\sum_{t=2}^{T} \big(\frac{1}{\eta_t} \Div{\u_t}{\xh_t} - \frac{1}{\eta_{t-1}}\Div{\u_{t-1}}{\xh_{t}}\big)$, hence  reflecting the \emph{non-stationarity} of environments;
\item[(iii)] the last one $- \sum_{t=1}^{T} \frac{1}{\eta_t}\big( \Div{\xh_{t+1}}{\x_t} + \Div{\x_t}{\xh_{t}} \big)$ is the \emph{negative} term, which can be greatly useful for problem-dependent bounds, particularly the gradient-variation one.
\end{enumerate}
Moreover, we emphasize that the above regret guarantee is very general due to the flexibility in choosing the regularizer $\psi$ and comparators $\u_1,\ldots,\u_T$ as well as the setting of optimistic vectors $M_1,\ldots,M_T$. For example, by choosing the negative-entropy regularizer and competing with the best fixed prediction, the result recovers the static regret bound of Optimistic Hedge~\citep{NIPS'15:fast-rate-game}; by choosing the Euclidean regularizer and setting the optimistic vectors as all zero vectors as well as competing with time-varying compactors, it recovers the dynamic regret bound of Online Gradient Descent~\citep{ICML'03:zinkvich}. The versatility of this optimistic OMD framework motivates us to use it as a unified building block for both algorithm design and theoretical analysis.
\endenv \end{myRemark}

\subsection{Assumptions}
\label{sec:assumption}
In this part, we list several common assumptions that might be used in the theorems.
\begin{myAssum}
\label{assumption:bounded-gradient}
The norm of the gradients of online functions over the domain $\X$ is bounded by $G$, i.e., $\norm{\nabla f_t(\x)}_2 \leq G$, for all $\x \in \X$ and $t \in [T]$.
\end{myAssum}

\begin{myAssum}
\label{assumption:bounded-domain}
The domain $\X \subseteq \R^d$ contains the origin $\mathbf{0}$, and the diameter of the domain $\X$ is at most $D$, i.e., $\norm{\x -\x'}_2 \leq D$ for any $\x, \x' \in \X$.
\end{myAssum}

\begin{myAssum}
\label{assumption:smoothness}
All the online functions are $L$-smooth, i.e., $\norm{\nabla f_t(\x)-\nabla f_t(\x')}_2 \leq L \norm{\x-\x'}_2$  for any $\x, \x' \in \R^d$ and $t \in [T]$.
\end{myAssum}

\begin{myAssum}
\label{assumption:non-negative}
All the online functions are non-negative over $\R^d$.
\end{myAssum}

We have the following remarks regarding the assumptions. The general dynamic regret analysis of Optimistic OMD (Theorem~\ref{thm:dynamic-regret-OMD-generic}) does not require the smoothness assumption. Nevertheless, this assumption is crucial for achieving problem-dependent dynamic regret bounds. In fact, smoothness has been demonstrated to be essential even in the static regret analysis for first-order methods to achieve gradient-variation bounds, as evidenced in Lemma~9 of~\citet{COLT'12:variation-Yang} and Theorem~1 of~\citet{ML'14:variation-Yang}. Therefore, throughout the paper we focus on the problem-dependent dynamic regret of convex and smooth functions. Note that Assumption~\ref{assumption:non-negative} requires non-negativity outside the domain $\X$, which is a precondition for establishing the self-bounding property for smooth functions, see Lemma~3.1 of~\citet{NIPS'10:smooth} and Lemma 13.2 of~\citet{lecture-note:Ashok}. 

Finally, we mention that following previous studies~\citep{ALT'12:closer-adaptive-regret,COLT'15:Luo-AdaNormalHedge}, we treat double logarithmic factors in $T$ as a constant. More concretely, our usage of the $\O(\cdot)$-notation emphasizes the dependence on the time horizon $T$ while hiding the $\log \log T$ factors, and also highlights the dependence on path length $P_T$, as well as the problem-dependent gradient-variation quantity $V_T$ and small-loss quantity $F_T$.

\section{Gradient-Variation Dynamic Regret}
\label{sec:gradient-variation}
Our paper aims to develop online algorithms that can \emph{simultaneously} achieve problem-dependent dynamic regret bounds, scaling with two quantities: the gradient-variation term $V_T$ and the small-loss term $F_T$, as defined in~\eqref{eq:gradient-variation}. As we will demonstrate in the next section, the gradient-variation bound is more fundamental than the small-loss bound. Consequently, we start by focusing on the gradient-variation dynamic regret in this section. In Section~\ref{sec:implications}, we will then present the small-loss bound and the best-of-both-worlds bound (i.e., achieving gradient-variation and small-loss bounds simultaneously) as direct implications.

\subsection{A Gentle Start}
In the study of static regret,~\citet{COLT'12:variation-Yang} propose the online extra-gradient descent (OEGD) algorithm and prove that the algorithm enjoys gradient-variation static regret. Specifically, OEGD starts from $\xh_1 \in \X$ and then updates by
\begin{equation}
  \label{alg:OEGD}
  \x_{t} = \Pi_{\X}\left[\xh_{t} - \eta \nabla f_{t-1}(\x_{t-1})\right],~~ \xh_{t+1} = \Pi_{\X}\left[\xh_{t}-\eta\nabla f_{t}(\x_{t})\right],
\end{equation}
where we define $f_0(\x_0) = \mathbf{0}$ and $\Pi_{\X}[\cdot]$ the Euclidean projection onto the nearest point in $\X$.  We consider a fixed step size $\eta>0$ for simplicity. For convex and smooth functions,~\citet{COLT'12:variation-Yang} prove that OEGD enjoys an $\O(\sqrt{1 + V_T})$ gradient-variation static regret. 

Actually, OEGD can be viewed as a specialization of  Optimistic OMD~\eqref{eq:optimistic-OMD-fix} presented in Section~\ref{sec:OMD-framework}, by choosing the regularizer $\Rcal(\x) = \frac{1}{2}\norm{\x}_2^2$ and the optimistic vector $M_{t} = \nabla f_{t-1}(\x_{t-1})$ as well as a fixed step size $\eta > 0$. Therefore,  Theorem~\ref{thm:dynamic-regret-OMD-generic} directly implies the following dynamic regret upper bound for OEGD, with proof in Appendix~\ref{sec:appendix-proof-OMD}.
\begin{myLemma}
\label{lemma:OEGD-variation}
Under Assumptions~\ref{assumption:bounded-gradient},~\ref{assumption:bounded-domain}, and~\ref{assumption:smoothness}, by choosing $\eta \leq \frac{1}{4L}$, dynamic regret of OEGD (namely, \textsc{Optimistic OMD} with $\Rcal(\x) = \frac{1}{2}\norm{\x}_2^2$ and $M_t = \nabla f_{t-1}(\x_{t-1})$) satisfies
\begin{equation}
\label{eq:variation-regret-smooth}
\sum_{t=1}^{T} f_t(\x_t) - \sum_{t=1}^{T} f_t(\u_t) \leq \eta (G^2 + 2V_T) + \frac{1}{2\eta}(D^2 + 2DP_T)
\end{equation}
for any comparator sequence $\u_1,\ldots,\u_T \in \X$. 
\end{myLemma}
Lemma~\ref{lemma:OEGD-variation} immediately implies a static regret bound. By choosing comparators as the best decision in hindsight $\u_1 = \ldots = \u_T \in \argmin_{\x \in \X} \sum_{t=1}^{T} f_t(\x)$, we have $P_T = 0$ and thereby obtain the existing result~\citep[Theorem 11]{COLT'12:variation-Yang}: $\sum_{t=1}^{T} f_t(\x_t) - \min_{\x \in \X} \sum_{t=1}^{T} f_t(\x) \leq \eta (G^2 + 2V_T) + \frac{D^2}{2\eta} = \O(\sqrt{1+V_T})$ when setting the step size $\eta = \min\{\sqrt{D^2/(G^2 + 2V_T)}, 1/(4L)\}$. Note that the requirement of $V_T$ in tuning can be removed by  doubling trick~\citep{JACM'97:doubling-trick} or self-confident tuning~\citep{JCSS'02:Auer-self-confident}. 

However, it is more complicated when competing with a sequence of time-varying comparators. Lemma~\ref{lemma:OEGD-variation} suggests that it is crucial to tune the step size to balance non-stationarity (path length $P_T$) and adaptivity (gradient-variation $V_T$) for achieving a tight dynamic regret bound. Ideally, the optimal tuning is $\eta^* = \sqrt{(D^2 + 2DP_T)/(2G^2 + 2V_T)}$, but this requires the prior information of $P_T$ and $V_T$ that are generally unavailable. We note that $V_T$ is empirically observable in the sense that at round $t \in [T]$ one can observe its internal estimate $V_t = \sum_{s=2}^t \sup_{\x \in \X}\norm{\nabla f_s(\x) - \nabla f_{s-1}(\x)}_2^2$. By contrast, $P_T = \sum_{t=2}^T \norm{\u_t - \u_{t-1}}_2$ cannot be known or approximated during the learning process. The ideal best comparator sequence, which tightens the upper bound of cumulative loss, satisfies the condition $\sum_{t=1}^T f_t(\x_t) \leq \min_{\u_1,\ldots,\u_T} \{\sum_{t=1}^T f_t(\u_t) + \mathcal{R}_T(P_T,V_T)\}$, where $\mathcal{R}_T(P_T,V_T)$ denotes the dynamic regret upper bound. Universal dynamic regret does not specifically target this optimal comparator sequence but aims to adapt to all feasible comparators, making the choice of $\u_1,\ldots,\u_T$  arbitrarily and entirely unknown within the feasible domain. 

In summary, while the self-confident tuning can be used to remove the dependence on the unknown gradient variation $V_T$, it cannot address the unknown path length $P_T$. In fact, this is the fundamental problem of non-stationary online learning --- how to deal with uncertainty due to unknown environmental non-stationarity, captured by path length of comparators in dynamic regret minimization.

To simultaneously handle the uncertainty arising from adaptivity and non-stationarity, in addition to using optimistic online learning to reuse the historical gradients, we design an adaptive \emph{online ensemble} method~\citep{book'12:ensemble-zhou} that can hedge the non-stationarity while extracting the adaptivity. Our approach deploys a two-layer meta-base structure, in which multiple base-learners are maintained simultaneously and a meta-algorithm is used to track the best one. More concretely, inspired by the recent advance in learning with multiple learning rates~\citep{NIPS'16:MetaGrad,JMLR'21:metagrad}, we first construct a pool of candidate step sizes to discretize possible range of the optimal step size, and then initialize multiple base-learners denoted by $\B_1,\ldots,\B_N$. Each base-learner $\B_i$ returns her own prediction $\x_{t,i}$ by running the base-algorithm with a step size $\eta_i$ from the pool. Finally, those predictions from base-learners are combined by a meta-algorithm to produce the final output $\x_t =\sum_{i=1}^{N} p_{t,i} \x_{t,i}$, where $\p_{t} \in \Delta_N$ is the weight from the meta-algorithm.

Due to the meta-base structure of the above procedures, we can naturally decompose dynamic regret into the following two parts:
\begin{equation}
  \label{eq:meta-base-decompose}
  \mbox{D-Regret}_T = \sum_{t=1}^T f_t(\x_t) - \sum_{t=1}^{T} f_t(\u_t) = \underbrace{\sum_{t=1}^T f_t(\x_t) - f_t(\x_{t,i})}_{\meta} + \underbrace{\sum_{t=1}^T f_t(\x_{t,i}) - f_t(\u_t)}_{\base},
\end{equation}
where $\{\x_t\}_{t=1,\ldots,T}$ denotes the final output sequence, and $\{\x_{t,i}\}_{t=1,\ldots,T}$ is the prediction sequence of base-learner $\B_i$. Notably, the decomposition holds for any base-learner's index $i \in [N]$. The first part is the difference between cumulative loss of the final output sequence and that of the prediction sequence of base-learner $\B_i$, which is introduced by the meta-algorithm and thus named as \emph{meta-regret}; the second part is the dynamic regret of base-learner $\B_i$ and therefore called \emph{base-regret}. As a result, we need to make the meta-regret and base-regret scaling with $V_T$ to achieve the desired gradient-variation dynamic regret.

In the following, we present two solutions. The first solution, developed in our conference paper~\citep{NIPS'20:sword}, is conceptually simpler but requires $N = \O(\log T)$ gradient queries at each round, making it suitable only for the multi-gradient feedback model. The second solution is an improved algorithm based on a refined analysis of the problem's structure, which attains the same dynamic regret guarantee with only one gradient per iteration and hence suits for the more challenging one-gradient feedback model. Recall that the definitions of the multi/one-gradient feedback models are presented in Section~\ref{sec:problem}.
\subsection{Multi-Gradient Feedback: Sword}
\label{sec:solution1}
Our approach, \textsf{Sword}, implements a meta-base online ensemble structure, in which multiple base-learners are initiated simultaneously (denoted by $\B_1,\ldots,\B_N$) and the intermediate predictions of all the base-learners are combined by a meta-algorithm to produce the final output. Below, we describe the specific settings of the base-algorithm and meta-algorithm.

For the base-algorithm, we simply employ the OEGD algorithm~\citep{COLT'12:variation-Yang}, where the base-learner $\B_i$ shall update her local decision $\{\x_{t,i}\}_{t=1,\ldots,T}$ by 
\begin{equation}
  \label{eq:OEGD-actual-loss}
  \x_{t,i} = \Pi_{\X}\left[\xh_{t,i} - \eta_{i} \nabla f_{t-1}(\x_{t-1,i})\right],~~ \xh_{t+1,i} = \Pi_{\X}\left[\xh_{t,i}-\eta_{i} \nabla f_{t}(\x_{t,i})\right],
\end{equation}
where $\eta_i > 0$ is the associated step size from the step size pool $\H = \{\eta_1,\ldots,\eta_N\}$ and the number of base-learner is chosen as $N = \O(\log T)$. Lemma~\ref{lemma:OEGD-variation} ensures an upper bound of base-regret scaling with the gradient variation, i.e., $\sum_{t=1}^{T} f_t(\x_{t,i}) - \sum_{t=1}^{T} f_t(\u_t) \leq \O(\eta_i (1+V_T) + P_T/\eta_i)$, whenever the step size satisfies $\eta_i \leq 1/(4L)$. The caveat is that each base-learner requires her own gradient direction for the update, so we need the gradient information of $\{\nabla f_t(\x_{t,i})\}_{i=1,\ldots,N}$ at round $t \in [T]$. Notably, the gradient query complexity is $N = \O(\log T)$ per round, which means the method developed in this part only suits for the multi-gradient feedback model. In Section~\ref{sec:solution2}, we will further design an improved algorithm applicable for the one-gradient feedback model.

The main difficulty lies in the design and analysis of an appropriate meta-algorithm. In order to be compatible to the gradient-variation base-regret, the meta-algorithm is required to incur a problem-dependent meta-regret of order $\O(\sqrt{V_T\log N})$. However, the meta-algorithms used in existing studies~\citep{NIPS'16:MetaGrad,NIPS'18:Zhang-Ader} cannot satisfy the requirements. For example, the vanilla Hedge suffers an $\O(\sqrt{T\log N})$ meta-regret, which is problem-independent and thus not suitable for us. To this end, we introduce the optimistic Hedge~\citep{NIPS'15:fast-rate-game} by exploiting the optimistic online learning, and further design a carefully designed optimism  specifically for our problem. 

Consider the problem of prediction with expert advice. At the beginning of iteration $(t+1)$, in addition to the loss vector $\ellb_t \in \R^N$ returned by the experts, the player can receive an optimism $\m_{t+1} \in \R^N$. Optimistic Hedge updates the weight vector $\p_{t+1} \in \Delta_{N}$ by
\begin{equation}
\label{eq:OptimisticHedge}
p_{t+1,i} \propto \exp\left(-\epsilon\Big(\sum_{s=1}^{t} \ell_{s,i} + m_{t+1,i}\Big)\right), \quad \forall i\in[N].
\end{equation}
Here, $\epsilon > 0$ is the learning rate of the meta-algorithm and we consider a fixed learning rate for simplicity.\footnote{We adopt the terminology ``learning rate'' for the meta-algorithm of our approach following the convention in the prediction with expert advice, and use ``step size'' for the general online convex optimization.} The optimism $\m_{t+1} \in \R^N$ can be interpreted as an optimistic guess of the loss of round $(t+1)$, and we thus incorporate it into the cumulative loss for update. It is well known that Optimistic Hedge can be regarded as an instance of Optimistic OMD with the negative-entropy regularizer, as mentioned in Remark~\ref{remark:OMD}. Therefore, the general result of Theorem~\ref{thm:dynamic-regret-OMD-generic} implies the following static regret bound of Optimistic Hedge, and the proof can be found in Appendix~\ref{sec:appendix-proof-OMD}. Notably, the negative term shown in~\eqref{eq:regret-optimistic-Hedge} will be of great importance in the algorithm design and regret analysis.
\begin{myLemma}
\label{lemma:OptimisticHedge}
The regret of Optimistic Hedge with a fixed learning rate $\epsilon > 0 $ to any expert $i \in [N]$ is at most
\begin{equation}
  \label{eq:regret-optimistic-Hedge}
  \sum_{t=1}^{T} \inner{\p_t}{\ellb_t} - \sum_{t=1}^{T} \ell_{t,i} \leq \epsilon \sum_{t=1}^{T} \norm{\ellb_t - \m_t}_{\infty}^2 + \frac{\ln N}{\epsilon} - \frac{1}{4\epsilon}\sum_{t=2}^{T} \norm{\p_{t} - \p_{t-1}}_1^2.
\end{equation}
Let $D_T = \sum_{t=1}^{T} \norm{\ellb_t - \m_{t}}_{\infty}^2$ measure the deviation between optimism and gradient. With a proper learning rate tuning scheme, Optimistic Hedge enjoys an $\O(\sqrt{D_T \log N})$ meta-regret.
\end{myLemma}

The framework of optimistic online learning is very powerful for designing adaptive methods, in that the adaptivity quantity $D_T = \sum_{t=1}^{T} \norm{\ellb_t - \m_{t}}_{\infty}^2$ is very general and can be specialized flexibly with different configurations of feedback loss $\ellb_t$ and optimism $\m_t$. To achieve the desired gradient-variation dynamic regret, we need to investigate the online ensemble structure carefully. To this end, we specialize Optimistic Hedge in the following way to make the meta-regret compatible with the desired gradient-variation quantity.
\begin{itemize}
    \item The feedback loss $\ellb_t \in \R^N$ is set as the linearized surrogate loss:
    \begin{equation}
      \label{eq:Sword-loss}
      \ell_{t,i} = \inner{\nabla f_t(\x_t)}{\x_{t,i}},~ \forall t \in [T] \mbox{ and } \forall i \in [N].
    \end{equation}
    \item The optimism $\m_t \in \R^N$ is set with a careful design: $\m_1 = \mathbf{0}$ and
    \begin{equation}
      \label{eq:Sword-optimism}
      m_{t,i} = \inner{\nabla f_{t-1}(\bar{\x}_{t})}{\x_{t,i}},~ \forall t \in [T] \mbox{ and } \forall i \in [N],  \text{   where } \bar{\x}_{t} = \sum_{i=1}^{N} p_{t-1,i} \x_{t,i}. 
    \end{equation}  
\end{itemize}
We will explain the motivation of such designs in Remark~\ref{remark:1}. Note that this optimism is legitimate as the instrumental variable $\bar{\x}_{t}$ only uses the information of $\p_{t-1}$ and local decisions $\{\x_{t,i}\}_{i=1,\ldots,N}$ at time $t$. The meta-algorithm updates the weight $\p_{t+1} \in \R^N $ by
\begin{equation}
\label{eq:VariationHedge}
p_{t+1,i} \propto \exp\left(-\epsilon\Big(\sum_{s=1}^{t} \inner{\nabla f_s(\x_s)}{\x_{s,i}} + \inner{\nabla f_{t}(\bar{\x}_{t+1})}{\x_{t+1,i}}\Big)\right), \quad \forall i\in[N].
\end{equation}
Algorithm~\ref{alg:Sword-meta} summarizes detailed procedures of the meta-algorithm, which in conjunction with the base-algorithm of Algorithm~\ref{alg:Sword-base} yields the Sword algorithm. We also discuss the gradient query complexity of the overall algorithm. In each round $t \in [T]$, at the base level, the algorithm computes the gradient information $\nabla f_t(\x_{t,i})$ for all $i \in [N]$. At the meta level, it additionally requires the gradients $\nabla f_t(\x_{t})$ and $\nabla f_{t}(\bar{\x}_{t+1})$. Consequently, the total gradient query complexity per round is $N+2 = \O(\log T)$.

\begin{figure}[!t]
\begin{minipage}{0.49\textwidth}
\begin{algorithm}[H]
   \caption{\textsf{Sword}: meta-algorithm}
   \label{alg:Sword-meta}
\begin{algorithmic}[1]
  \REQUIRE{step size pool $\H$; learning rate $\epsilon$}
  \STATE{Initialization: $\forall i\in [N], p_{0,i} = 1/N$}
    \FOR{$t=1$ {\bfseries to} $T$}
      \STATE Receive $\x_{t+1,i}$ from base-learner $\B_i$
      \STATE Update weight $p_{t+1,i}$ by~\eqref{eq:VariationHedge}
      \STATE Predict $\x_{t+1} = \sum_{i=1}^{N} p_{t+1,i} \x_{t+1,i}$
    \ENDFOR
\end{algorithmic}
\end{algorithm}
\end{minipage}
\hspace{2mm}
\begin{minipage}{0.49\textwidth}
\begin{algorithm}[H]
   \caption{\textsf{Sword}: base-algorithm}
   \label{alg:Sword-base}
\begin{algorithmic}[1]
  \REQUIRE{step size $\eta_i \in \H$}
  \STATE{Let $\hat{\x}_{1,i}, \x_{1,i}$ be any point in $\mathcal{X}$}
    \FOR{$t=1$ {\bfseries to} $T$}
      \STATE $\xh_{t+1,i} = \Pi_{\mathcal{X}}\big[\xh_{t,i} - \eta_i \nabla f_{t}(\x_{t,i})\big]$ 
      \STATE $\x_{t+1,i} = \Pi_{\mathcal{X}}\big[\xh_{t+1,i} - \eta_i \nabla f_{t}(\x_{t,i})\big]$
      \STATE Send $\x_{t+1,i}$ to the meta-algorithm     
    \ENDFOR
\end{algorithmic}
\end{algorithm}
\end{minipage}
\end{figure}

\begin{myRemark}[design of optimism]
\label{remark:1}
The design of optimism in~\eqref{eq:Sword-optimism}, particularly the construction of the instrumental variable $\xb_{t}$, is crucial and is the most challenging part in this method.  Our design carefully leverages the structure of two-layer online ensemble methods, specifically, the goal of designing optimism is to approximate the current gradient $\nabla f_t(\x_t)$ (which is unknown) via the available knowledge till round $t$. We propose to use $\nabla f_{t-1}(\bar{\x}_{t})$ as the approximation, and the difference of online functions delivers the gradient-variation term $\sup_{\x \in \X} \norm{f_t(\x) - f_{t-1}(\x)}_2^2$, while the difference between $\x_t$ and $\bar{\x}_t$ can be upper bounded by the decision variation of the meta-algorithm, 
\begin{equation}
  \label{eq:Sword-swicthing-cost}
  \norm{\x_t - \xb_t}_2^2 = \left\Vert \sum_{i=1}^{N} (p_{t,i} - p_{t-1,i})\x_{t,i}\right\Vert_2^2 \leq \Big(\sum_{i=1}^{N} \abs{p_{t,i} - p_{t-1,i}} \norm{\x_{t,i}}_2\Big)^2 \leq D^2 \norm{\p_t - \p_{t-1}}_1^2,
\end{equation}
which can be eliminated by the \emph{negative term} in the regret bound of Optimistic Hedge as shown in~\eqref{eq:regret-optimistic-Hedge}, providing with a suitable setting for the learning rate of the meta-algorithm. As such, the adaptivity quantity $D_T$ can be converted to the desired gradient variation $V_T$ plus the decision variation of the meta-algorithm, concretely,
\begin{equation}
\label{eq:Sword-adaptivity-expand}
\begin{split}
   \norm{\ellb_t - \m_t}_{\infty}^2 \overset{\eqref{eq:Sword-optimism}}{=} {} &  \max\nolimits_{i\in[N]} \inner{\nabla f_t(\x_t) - \nabla f_{t-1}(\xb_t)}{\x_{t,i}}^2\\
   \leq {} & D^2 \norm{\nabla f_t(\x_t) - \nabla f_{t-1}(\xb_t)}_2^2 \\
   \leq {} & 2 D^2 (\norm{\nabla f_t(\x_t) - \nabla f_{t-1}(\x_t)}_2^2 + \norm{\nabla f_{t-1}(\x_t) - \nabla f_{t-1}(\xb_t)}_2^2) \\
   \leq {} & 2 D^2 \sup\nolimits_{\x \in \X}\norm{\nabla f_t(\x) - \nabla f_{t-1}(\x)}_2^2 + 2 D^2 L^2\norm{\x_t - \xb_t}_2^2\\
   \leq {} & 2 D^2 \sup\nolimits_{\x \in \X}\norm{\nabla f_t(\x) - \nabla f_{t-1}(\x)}_2^2 + 2 D^4 L^2\norm{\p_t - \p_{t-1}}_1^2,
\end{split}
\end{equation}
where the derivation makes use of the boundedness of the feasible domain, triangle inequality, and the smoothness of online functions. The last term will be canceled by the negative term in the meta-regret, then we obtain the desired gradient-variation regret guarantee.
\endenv \end{myRemark}

The following theorem shows that the meta-regret is at most $\O(\sqrt{(1+V_T) \log N})$, which is nicely compatible to the attained base-regret. The proof can be found in Section~\ref{sec:proof-Sword-meta}.
\begin{myThm}
\label{thm:variation-meta-regret}
Under Assumptions~\ref{assumption:bounded-gradient},~\ref{assumption:bounded-domain}, and~\ref{assumption:smoothness}, by setting the learning rate of the meta-algorithm~\eqref{eq:VariationHedge} optimally as $\epsilon = \min\{1/(4D^2L),\sqrt{(\ln N)/(2D^2(G^2 + V_T))}\}$, the meta-regret of Sword (Algorithm~\ref{alg:Sword-meta}) is at most 
\begin{equation*}
  \sum_{t=1}^T f_t(\x_t) - \sum_{t=1}^Tf_t(\x_{t,i}) \leq 2D\sqrt{2(G^2 + V_T)\ln N} + 8D^2L\ln N = \O\Big( \sqrt{(1+V_T) \log N} \Big).
\end{equation*}
\end{myThm}
Note that the optimal learning rate tuning of the meta-algorithm requires the knowledge of gradient variation $V_T = \sum_{t=2}^T \sup_{\x \in \X}\norm{\nabla f_t(\x) - \nabla f_{t-1}(\x)}_2^2$. The undesired demand can be removed by the self-confident tuning method~\citep{JCSS'02:Auer-self-confident}, which employs a time-varying learning rate scheme for the meta-algorithm's update based on internal estimates, roughly, $p_{t+1,i} \propto \exp\big(-\epsilon_t(\sum_{s=1}^{t} \ell_{s,i} + m_{t+1,i})\big), \forall i\in[N]$ with $\epsilon_t = \O(1/\sqrt{1+V_{t}})$ with an internal estimate $V_{t} = \sum_{s=2}^{t} \sup_{\x \in \X}\norm{\nabla f_s(\x) - \nabla f_{s-1}(\x)}_2^2$. Besides, notice that this $V_{t}$ is actually not easy to calculate due to the computation of instantaneous variation $\sup_{\x \in \X} \norm{\nabla f_t(\x) - \nabla f_{t-1}(\x)}_2^2$, which is the difference of convex functions programming and is not easy to solve even with the explicit form of functions $f_t$ and $f_{t-1}$. Fortunately, we can use an alternative twisted quantity $\bar{V}_T = \sum_{t=2}^{T} \norm{\nabla f_t(\x_t) - \nabla f_{t-1}(\x_{t-1})}_2^2$ for the learning rate configuration and also achieve the same regret bound via a refined analysis. Then, it suffices to perform the self-confident tuning over $\bar{V}_T$ by monitoring the corresponding internal estimate $\bar{V}_t = \sum_{s=2}^{t} \norm{\nabla f_s(\x_s) - \nabla f_{s-1}(\x_{s-1})}_2^2$, which avoids the burdensome calculations of inner optimization problems and thereby significantly streamlines the computational efforts paid for the adaptive learning rate tuning. A caveat of this Optimistic Hedge update when implemented time-varying learning rates is that it essentially is a special case of Optimistic FTRL rather than Optimistic OMD. For a more thorough technical discussion, see Remark~\ref{remark:optimisticHedge-LR} in Appendix~\ref{appendix:adaptive-LR}.

So far, the obtained base-regret bound (Lemma~\ref{lemma:OEGD-variation}) and meta-regret bound (Theorem~\ref{thm:variation-meta-regret}) are both adaptive to the gradient variation, and we can simply combine them to achieve the final gradient-variation dynamic regret as stated in Theorem~\ref{thm:dynamic-var}, providing with an appropriate candidate step size pool. The proof is provided in Section~\ref{sec:proof-Sword-dynamic-regret}.   
\begin{myThm}
\label{thm:dynamic-var}
Under Assumptions~\ref{assumption:bounded-gradient},~\ref{assumption:bounded-domain}, and~\ref{assumption:smoothness}, set the pool of candidate step sizes $\H$ as 
\begin{equation}
  \label{eq:step-size-pool-variation}
  \H = \left\{\eta_i = \min\bigg\{\frac{1}{4L}, 2^{i-1}\sqrt{\frac{D^2}{8G^2T}}\bigg\} \mid  i \in [N]\right\},
\end{equation}
where $N = \lceil 2^{-1} \log_2(G^2T/(2D^2L^2))\rceil + 1$ is the number of candidate step sizes; and set the learning rate of the meta-algorithm as $\epsilon = \min\{1/(4D^2L),\sqrt{(\ln N)/(2D^2(G^2 + V_T))}\}$.
Then, Sword enjoys the following dynamic regret against any comparators $\u_1,\ldots,\u_T \in \X$,
\begin{align*}
\label{eq:dynamic-regret-variation}
\sum_{t=1}^T f_t(\x_{t}) - \sum_{t=1}^T f_t(\u_t) \leq \O \Big(\sqrt{(1 + P_T + V_T)(1 + P_T)}\Big).
\end{align*}
\end{myThm}
\begin{myRemark}
Compared with the existing $\O(\sqrt{T(1+P_T)})$ dynamic regret~\citep{NIPS'18:Zhang-Ader}, our result is more adaptive in the sense that it replaces $T$ by the \emph{problem-dependent} quantity $P_T + V_T$. Therefore, the bound will be much tighter in easy problems, for example when both $V_T$ and $P_T$ are $o(T)$. Meanwhile, it safeguards the same minimax rate, since both quantities are at most $\O(T)$. Furthermore, because the \emph{universal} dynamic regret studied in this paper holds against any comparator sequence, it specializes the static regret by setting all comparators as the best fixed decision in hindsight, i.e., $\u_1=\ldots=\u_T=\x^* \in \argmin_{\x \in \X} \sum_{t=1}^{T} f_t(\x)$. Under such a circumstance, the path length $P_T = \sum_{t=2}^{T} \norm{\u_{t-1} - \u_t}_2$ becomes zero, so the regret bound in Theorem~\ref{thm:dynamic-var} implies an $\O (\sqrt{1 + V_T})$ gradient-variation static regret bound, recovering the result of~\citet{COLT'12:variation-Yang}.
\endenv \end{myRemark}
\subsection{One-Gradient Feedback: Sword++}
\label{sec:solution2}
So far, we have designed an online algorithm (Sword) with the gradient-variation dynamic regret. While it achieves a favorable regret guarantee, one caveat is that Sword runs $N = \O(\log T)$ base-learners simultaneously and each base-learner requires her own gradient direction for the update. Consequently, the overall algorithm necessitates $\O(\log T)$ gradient queries at each iteration, making it time-consuming and only applicable to the multi-gradient feedback model. In contrast,  algorithms designed for static regret minimization typically work well under the more challenging one-gradient feedback model, namely, they only require the gradient information $\nabla f_t(\x_t)$ for updates. Given this, it is natural to ask whether it is possible to design online algorithms that can achieve the same dynamic regret guarantee as Sword \emph{using only one gradient query per iteration}, making them applicable to the one-gradient feedback online learning.

We resolve the question affirmatively. The new algorithm, called \textsf{Sword++}, also implements an online ensemble structure. Compared to Sword presented in Section~\ref{sec:solution1}, the key novel ingredient is the framework of \emph{collaborative online ensemble}. We carefully introduce correction terms to the online loss and optimism, forming a biased surrogate loss and a surrogate optimism, which are then fed to the meta-algorithm. By further exploiting the negative terms in the meta and base levels, the overall algorithm ensures effective \mbox{\emph{collaboration}} within the meta and base two layers, thereby achieving the favorable gradient-variation dynamic regret with only one gradient query per iteration.

In the following, we describe the details of Sword++. The algorithm maintains multiple base-learners denoted by $\B_1,\ldots,\B_N$, which are performed with different step sizes and then combined by a meta-algorithm to track the best one. An exponential step size grid is adopted as the schedule, denoted by $\H = \{\eta_i = c \cdot 2^{i} \mid i \in [N]\}$ with $N = \O(\log T)$ for some constant $c > 0$ (usually scaling with $\poly(1/T)$), whose specific setting will be given later.

\paragraph{Base-algorithm.} Instead of performing updates over the original loss $f_t$ as shown in~\eqref{eq:OEGD-actual-loss}, all the base-learners of Sword++ update over the \emph{linearized surrogate loss} $g_t:\X \mapsto \R$ defined $g_t(\x) = \inner{\nabla f_t(\x_t)}{\x}$, and moreover, the optimism is chosen as $M_t = \nabla g_{t-1}(\x_{t-1,i})$ for each base-learner $\B_i$ with $i \in [N]$. By definition, we have $\nabla g_t(\x_{t,i}) = \nabla f_t(\x_t)$, so each base-learner $\B_i$ essentially performs the following update at each iteration:
\begin{equation}
    \label{eq:variation-OGD-base}
    \x_{t,i} = \Pi_{\X}\left[\xh_{t,i} - \eta_i \nabla f_{t-1}(\x_{t-1})\right],~~ \xh_{t+1,i} = \Pi_{\X}\left[\xh_{t,i} - \eta_i \nabla f_{t}(\x_{t})\right].
\end{equation}
Using above update steps, we no longer need to evaluate the gradient $\nabla f_t(\x_{t,i})$ over the local decisions for every base-learner, as was done by Sword (see its update rule in~\eqref{eq:OEGD-actual-loss}). Instead, a single call of $\nabla f_t(\x_t)$ is sufficient at each round for the update in Sword++. 

We note that although the linearized trick has previously been employed in the meta-base structure for achieving the minimax dynamic regret $\O(\sqrt{T(1+P_T)})$ with one gradient per iteration~\citep{NIPS'18:Zhang-Ader}, this modification alone is far from enough to obtain a problem-dependent dynamic regret. To see this, we can check the regret of the base-learner updated with the surrogate loss $g_t(\x)$. A similar argument of Lemma~\ref{lemma:OEGD-variation} shows that the base-regret over the linearized loss $(\#) \triangleq \sum_{t=1}^T g_t(\x_{t,i}) - \sum_{t=1}^T g_t(\u_t)$ satisfies
\begin{align}
\label{eq:sword++-challenge}
(\#) \leq \eta_i(G^2+2V_T) + \frac{D^2+2DP_T}{2\eta_i} + 2\eta_i L^2\sum_{t=2}^T\norm{\x_t-\x_{t-1}}^2_2 - \frac{1}{4\eta_i}\sum_{t=2}^T\norm{\x_{t,i}-\x_{t-1,i}}^2_2.
\end{align}
In the analysis of Sword, because the gradient $\nabla f_t(\x_{t,i})$ is evaluated on every base-learner's own decision $\x_{t,i}$, the additional positive term (the third one) is $2\eta_iL^2\sum_{t=2}^T\norm{\x_{t,i}-\x_{t-1,i}}_2^2$, which can be cancelled by the negative term $-\sum_{t=2}^T\norm{\x_{t,i}-\x_{t-1,i}}_2^2/(4\eta_i)$ whenever the step size is set appropriately. However, when the base-learner updates her decision over the surrogate loss, the additional positive term becomes $2\eta_iL^2 \norm{\x_t-\x_{t-1}}_2^2$, which \emph{cannot} be handled by the negative term of any base-learner. Thus, more advanced mechanisms are required to achieve problem-dependent dynamic regret under the one-gradient query model. 

To tackle the difficulty, our primary idea is to facilitate \emph{collaboration} between the meta and base levels. Specifically, we aim to leverage negative terms from both levels to handle the positive term. However, it turns out that the positive term cannot be entirely offset by the combined negative terms from meta and base levels. To address this issue, we introduce correction terms to the feedback loss and optimism in the meta-algorithm. This generates a new negative term that, together with the negative term from the meta level, effectively cancels out the positive term. Nevertheless, another new positive term emerges due to the injected correction, which we ensure can be managed by the negative term from the base level. As a result, the overall undesired positive term is finally addressed.

The above forms the main idea of our proposed \emph{collaborative online ensemble} framework. The term ``collaboration'' refers to the interplay between meta and base layers. Indeed, on their own, neither the base level nor the meta level can achieve a gradient-variation base/meta regret; each incurs some additional positive terms. This positive term necessitates the negative terms from the other layer to help effectively cancel it out.

In the following, we describe the details of the meta-algorithm. We will provide a brief explanation of the design of corrections in Remark~\ref{remark:collaboration} and offer a more comprehensive elaboration on the general framework of collaborative online ensemble in Section~\ref{sec:framework-collaborative-OE}.

\paragraph{Meta-algorithm.} We still employ Optimistic Hedge as the meta-algorithm, but additionally require  innovative designs in feedback loss and optimism. Specifically, instead of simply using the linearized surrogate loss $\ell_{t,i}\triangleq\inner{\nabla f_t(\x_t)}{\x_{t,i}}$ as the feedback loss like Sword (see the update rule in~\eqref{eq:Sword-loss}), we carefully construct the surrogate loss in the following way and send it to the meta-algorithm. 
\begin{itemize}
 \item The feedback loss $\ellb_t \in \R^N$ is constructed as follows: for each $i \in [N]$, $\ell_{1,i} = \inner{\nabla f_1(\x_1)}{\x_{1,i}}$ and for $t \geq 2$, it composes the linearized surrogate loss $\inner{\nabla f_t(\x_t)}{\x_{t,i}}$ with a \emph{decision-deviation correction term}, namely,
    \begin{equation}
    \label{eq:meta-surrogate-loss}
    \ell_{t,i} = \inner{\nabla f_t(\x_t)}{\x_{t,i}} + \lambda \norm{\x_{t,i} - \x_{t-1,i}}_2^2.
\end{equation}
 \item The optimism $\m_t \in \R^N$ is similarly configured as follows: $\m_{1} = \mathbf{0}$ and for $t \geq 2$ and $i \in [N]$, the optimism also admits a \emph{decision-deviation correction term}, namely,
    \begin{equation}
        \label{eq:meta-optimism}
        m_{t,i} = \inner{M_t}{\x_{t,i}} + \lambda \norm{\x_{t,i} - \x_{t-1,i}}_2^2 = \inner{\nabla f_{t-1}(\x_{t-1})}{\x_{t,i}} + \lambda \norm{\x_{t,i} - \x_{t-1,i}}_2^2.
    \end{equation}
\end{itemize}
Both feedback loss and optimism admit an additional correction term $\lambda \norm{\x_{t,i} - \x_{t-1,i}}_2^2$ that measures the stability of the local decisions returned by the base-learner, where $\lambda > 0$ is the correction coefficient to be determined later. We will explain soon in Remark~\ref{remark:collaboration} on the crucial role and design motivation of this correction. Overall, the meta-algorithm of Sword++ updates the weight $\p_{t+1} \in \R^N $ as follows: for any $i \in [N]$,
\begin{equation}
    \label{eq:variation-Hedge-meta}
    p_{t+1,i} \propto \exp\left(-\epsilon \Big(\sum_{s=1}^{t} \ell_{s,i} + m_{t+1,i} \Big)\right),
\end{equation} 
where $\epsilon > 0$ is a (for simplicity) fixed learning rate. Notably, the update only requires the gradient information of $\nabla f_t(\x_t)$ and thus is feasible for the one-gradient feedback model.

\begin{myRemark}[design of correction term]
\label{remark:collaboration}
We emphasize that the correction term $\lambda \norm{\x_{t,i} - \x_{t-1,i}}_2^2$, appearing in the construction of both feedback loss and optimism, is crucial for the design and is the most challenging part in this method. We briefly explain the motivation. As mentioned earlier in~\eqref{eq:sword++-challenge}, the use of linearized surrogate loss $g_t(\x)$ will introduce an additional term ${\sum_{t=2}^T}\norm{\x_t-\x_{t-1}}_2^2$, which \emph{cannot} be directly canceled by the negative term of any base-regret, namely, $-\sum_{t=2}^T\norm{\x_{t,i}-\x_{t-1,i}}_2^2$. To address the difficulty, we scrutinize the positive term and find that actually it can be further expanded as:
\begin{align*}
\norm{\x_t-\x_{t-1}}_2^2 ={}&\left\lVert{\sum_{i=1}^Np_{t,i}\x_{t,i}-\sum_{i=1}^Np_{t-1,i}\x_{t-1,i}}\right\rVert_2^2\notag\\
\leq {}&2\left\lVert{\sum_{i=1}^Np_{t,i}\x_{t,i}-\sum_{i=1}^Np_{t,i}\x_{t-1,i}}\right\rVert_2^2 + 2\left\lVert{\sum_{i=1}^Np_{t,i}\x_{t-1,i}-\sum_{i=1}^Np_{t-1,i}\x_{t-1,i}}\right\rVert_2^2\notag\\
\leq {}& 2\left(\sum_{i=1}^N p_{t,i} \norm{\x_{t,i}-\x_{t-1,i}}_2\right)^2 + 2 \left(\sum_{i=1}^N  \abs{p_{t,i}-p_{t-1,i}}  \norm{\x_{t-1,i}}_2\right)^2\notag\\
\leq {}&2\sum_{i=1}^Np_{t,i}\norm{\x_{t,i}-\x_{t-1,i}}_2^2 + 2D^2\norm{\p_t-\p_{t-1}}_1^2,
\end{align*}
which concludes that
\begin{equation}
  \label{eq:Swordpp-swicthing-cost}
  \sum_{t=2}^T \norm{\x_t - \x_{t-1}}_2^2 \leq 2\sum_{t=2}^T\sum_{i=1}^N p_{t,i}\norm{\x_{t,i}-\x_{t-1,i}}_2^2 + 2D^2\sum_{t=2}^T\norm{\p_t-\p_{t-1}}_1^2.
\end{equation}
The right hand side of~\eqref{eq:Swordpp-swicthing-cost} is a weighted combination of stability of base-learners  (hence called \emph{mixed stability}), and the second one is the stability of the meta-algorithm's weights (hence called \emph{meta stability}). We also similarly define $\sum_{t=2}^T \norm{\x_{t,i} - \x_{t-1,i}}_2^2$ as the \emph{base stability} (of the base learner $i \in [N]$). Clearly, the meta stability can be readily canceled by the negative term of meta-regret. However, it is challenging to address the first positive term, namely, the mixed stability. To overcome the difficulty, we \emph{algorithmically} add the decision-variation correction term in the feedback loss and optimism of the meta-algorithm, as well as leveraging the negative term of base-regret. The underlying intuition is \emph{to penalize base-learners with large decision variations, so as to ensure a small enough variation of final decisions}. As such, we have facilitated the collaborations between the base and meta levels --- the overall positive term ($\sum_{t=2}^T \norm{\x_t - \x_{t-1}}_2^2$) is jointly cancelled out by the negative term of meta-regret ($-\sum_{t=2}^T\norm{\p_t-\p_{t-1}}_1^2$) and the one due to the injected corrections ($-\sum_{t=2}^T\sum_{i=1}^N p_{t,i}\norm{\x_{t,i}-\x_{t-1,i}}_2^2$); and meanwhile, the injected corrections will introduce a new positive term ($\sum_{t=2}^T \norm{\x_{t,i}-\x_{t-1,i}}_2^2$), which can be further tackled by the negative term of base-regret ($-\sum_{t=2}^T \norm{\x_{t,i}-\x_{t-1,i}}_2^2$). A notable characteristic is that the positive terms of meta/base/mixed stability cannot be cancelled solely by the negative terms within their respective layer. Instead, they necessitate additional negative terms, either from regret analysis or algorithmic corrections, to help effectively cancel out. Only through such collaborations within the two-layer online ensembles can the proposed Sword++ algorithm attain the desired gradient-variation dynamic regret, utilizing only one gradient per iteration. A general presentation of this collaborative online ensemble will be provided in Section~\ref{sec:framework-collaborative-OE}.
\endenv \end{myRemark}

We summarize the procedures of Sword++ in Algorithm~\ref{alg:Swordpp-meta} (meta-algorithm) and Algorithm~\ref{alg:Swordpp-base} (base-algorithm). Though multiple base-learners are performed with different step sizes to tackle the uncertainty of non-stationary environments, Sword++ requires the gradient information of $\nabla f_t(\x_t)$ only at round $t$ and then broadcasts it to all the base-learners for local update. Therefore, Sword++ is feasible for the one-gradient feedback model. Moreover, the algorithm provably achieves the same gradient-variation dynamic regret as Sword up to constants, shown in Theorem~\ref{thm:variation-one-gradient}, whose proof is presented in Section~\ref{sec:proof-Swordpp-dynamic-regret}.
\begin{myThm}
\label{thm:variation-one-gradient}
Under Assumptions~\ref{assumption:bounded-gradient},~\ref{assumption:bounded-domain}, and~\ref{assumption:smoothness}, set the pool of candidate step sizes $\H$ as 
\begin{equation}
  \label{eq:step-size-pool-one-gradient}
  \H = \left\{\eta_i = \min\bigg\{\frac{1}{8L}, \sqrt{\frac{D^2}{8G^2T}}\cdot 2^{i-1}\bigg\} \mid i\in [N]\right\},
\end{equation}
where $N = \lceil 2^{-1} \log_2(G^2T/(8D^2L^2))\rceil + 1$ is the number of candidate step sizes; further set the correction coefficient as $\lambda = 2L$ and  the learning rate of the meta-algorithm as $\epsilon = \min\big\{1/(8D^2L), \sqrt{(\ln N)/(D^2(\norm{\nabla f_1(\x_1)}_2^2 + \bar{V}_T})) \big\}$.
Then, Sword++ satisfies
\begin{align*}
   \sum_{t=1}^T f_t(\x_t) - \sum_{t=1}^T f_t(\u_t) \leq \O\left(\sqrt{(1+P_T+V_T)(1+P_T)}\right)
\end{align*}
for any comparator sequence $\u_1,\ldots,\u_T \in \X$. In above, $\bar{V}_T = \sum_{t=2}^T\norm{\nabla f_t(\x_t) - \nabla f_{t-1}(\x_{t-1})}_2^2$ is the variant of gradient variation $V_T$.
\end{myThm}

\begin{figure}[!t]
\begin{minipage}{0.49\textwidth}
\begin{algorithm}[H]
   \caption{\textsf{Sword++}: meta-algorithm}
   \label{alg:Swordpp-meta}
\begin{algorithmic}[1]
  \REQUIRE{step size pool $\H$; learning rate $\epsilon$}
  \STATE{Initialization: $\forall i\in [N], p_{0,i} = 1/N$}
    \FOR{$t=1$ {\bfseries to} $T$}
      \STATE Receive $\x_{t+1,i}$ from base-learner $\B_i$
      \STATE Update weight $p_{t+1,i}$ by~\eqref{eq:meta-surrogate-loss}--\eqref{eq:variation-Hedge-meta}
      \STATE Predict $\x_{t+1} = \sum_{i=1}^{N} p_{t+1,i} \x_{t+1,i}$
    \ENDFOR
\end{algorithmic}
\end{algorithm}
\end{minipage}
\hspace{2mm}
\begin{minipage}{0.49\textwidth}
\begin{algorithm}[H]
   \caption{\textsf{Sword++}: base-algorithm}
   \label{alg:Swordpp-base}
\begin{algorithmic}[1]
  \REQUIRE{step size $\eta_i \in \H$}
  \STATE{Let $\hat{\x}_{1,i}, \x_{1}$ be any point in $\mathcal{X}$}
    \FOR{$t=1$ {\bfseries to} $T$}
      \STATE $\xh_{t+1,i} = \Pi_{\X}[\xh_{t,i} - \eta_i \nabla f_t(\x_t)]$ 
      \STATE $\x_{t+1,i} = \Pi_{\X}[\xh_{t+1,i} - \eta_i  \nabla f_t(\x_t)]$
      \STATE Send $\x_{t+1,i}$ to the meta-algorithm
    \ENDFOR
\end{algorithmic}
\end{algorithm}
\end{minipage}
\end{figure}

Note that the learning rate tuning of the meta-algorithm requires the knowledge of $\bar{V}_T$. Yet, this unpleasant dependence can be removed by performing the self-confident tuning over $\bar{V}_T$ by monitoring the internal estimate $\bar{V}_t = \sum_{s=2}^{t} \norm{\nabla f_s(\x_s) - \nabla f_{s-1}(\x_{s-1})}_2^2$. Importantly, this adaptive learning rate tuning can be easily realized under the one-gradient feedback model, namely, only $\nabla f_t(\x_t)$ available at round $t$. To avoid clutter, we here stick to a fixed learning rate instead of a time-varying one, which is more convenient to demonstrate the collaboration between meta and base layers in the regret analysis (also see Remark~\ref{remark:collaboration-proof} in the proof of Theorem~\ref{thm:variation-one-gradient}). We also defer an adaptive learning rate version to Appendix~\ref{appendix:adaptive-LR}.

Up to now, we have shown that it is possible to design online methods to achieve stronger guarantees than static methods under the challenging one-gradient feedback online learning, and meanwhile suffer no computational overhead in terms of the gradient query complexity. 
\section{A General Framework: Collaborative Online Ensemble}
\label{sec:framework-collaborative-OE}
In this section, we formally introduce the proposed \emph{collaborative online ensemble} framework, a general algorithmic template designed to achieve (problem-dependent) dynamic regret guarantees. This framework is particularly crucial for attaining gradient-variation bounds. As will be demonstrated shortly, our proposed Sword (in Section~\ref{sec:solution1}) and Sword++ (in Section~\ref{sec:solution2}) can both be considered as specific instantiations.

\subsection{Algorithmic Template}
\label{sec:online-ensemble-template}
We focus on the standard OCO setup as specified in Section~\ref{sec:problem}. At iteration $t \in [T]$, the player first chooses the decision $\x_t \in \X$, then the environments reveal the loss function $f_t: \X \mapsto \R$. Subsequently, the player suffers the loss $f_t(\x_t)$ and observes certain gradient information of $\nabla f_t(\cdot)$ according to the feedback model.

The overall algorithmic template implements a meta-base two-layer online ensemble. There are three crucial ingredients in collaborative online ensemble: (i) the surrogate loss, (ii) the surrogate optimism, and (iii) the correction terms. Additionally, the negative terms, hidden in the analysis, play a significant role within the framework. To better present the algorithmic template, we introduce the following notations:
\begin{itemize}
\item for the base-algorithm, let $g_t^{\mathtt{base}}: \X \mapsto \R$ be the base surrogate loss and $h_t^{\mathtt{base}}: \X \mapsto \R$ be the base surrogate optimism;
\item for the meta-algorithm, let $g_t^{\mathtt{meta}}: \X \mapsto \R$ be the meta surrogate loss and $h_t^{\mathtt{meta}}: \X \mapsto \R$ be the meta surrogate optimism, and let $\cb_t \in \R^N$ be the correction term.
\end{itemize}

The base-algorithm updates the decisions $\{\x_{t,i}\}_{i=1}^N$ by Optimistic OGD over the base surrogate loss and optimism, that is,
\begin{align}
\label{eq:general-framework-base}
  \x_{t,i} = \Pi_{\X}\left[\xh_{t,i} - \eta_{i} \nabla h_t^{\mathtt{base}}(\x_{t-1,i})\right],~~ \xh_{t+1,i} = \Pi_{\X}\left[\xh_{t,i}-\eta_{i} \nabla g_t^{\mathtt{base}}(\x_{t,i})\right],
\end{align}
where $\eta_i>0$ is a fixed step size specified by the step size pool $\H = \{\eta_1,\ldots,\eta_N\}$. Subsequently, the player makes the final decision at this round by $\x_t = \sum_{i=1}^N p_{t,i}\x_{t,i}$. 

The meta-algorithm will then update the weight $\p_{t+1} \in \Delta_N$ by Optimistic Hedge over the feedback loss $\ellb_{t} \in \R^N$ and optimism $\m_t \in \R^N$, 
\begin{equation}
    \label{eq:general-framework-meta}
    p_{t+1,i} \propto \exp\left(-\epsilon \Big(\sum_{s=1}^{t} \ell_{s,i} + m_{t+1,i} \Big)\right),
\end{equation}
where $\epsilon > 0$ is (for simplicity) chosen as a fixed learning rate of the meta-algorithm and the feedback loss $\ell_t \in \R^N$ and optimism $\m_t \in \R^N$ are defined as
\begin{equation}
    \label{eq:general-framework-meta-loss-optimism}
    \ell_{t,i} = g_t^{\mathtt{meta}}(\x_{t,i}) + \lambda c_{t,i}, \textnormal{  and  } m_{t,i} = h_t^{\mathtt{meta}}(\x_{t,i}) + \lambda c_{t,i},
\end{equation}
where $\lambda \geq 0$ is the coefficient of the correction terms and $c_{t,i}$ is the $i$-th entry of $\cb_t$.

\begin{myRemark}
The meta-base updates in~\eqref{eq:general-framework-base} and~\eqref{eq:general-framework-meta} are quite versatile, as there are many options for constructing the surrogate (meta/base) loss, optimism, and the correction term. We remind that a feasible construction must adhere to the feedback model --- in the multi-gradient feedback model, the entire gradient function $\nabla f_t(\cdot)$ is available, while in the one-gradient feedback model, only the gradient $\nabla f_t(\x_t)$ is available to the player. In Section~\ref{sec:framework-instantiations}, we will present several concrete instantiations of the general algorithmic template, including the proposed Sword and Sword++ in the earlier subsections.
\endenv \end{myRemark}

\subsection{Instantiations}
\label{sec:framework-instantiations}
In this part, we present three instantiations of the general algorithmic template, including Sword, Sword++, and another important instantiation, which we refer to as \textsf{Sword.optimism}. For clarity, we provide a summary of these instantiations in Table~\ref{table:instantiations-summary}.

\begin{table}[!t]
\caption{Summary of three instantiations of the collaborative online ensemble framework, including Sword, Sword++, and Sword.optimism.}
\vspace{2mm}
\centering
\label{table:instantiations-summary}
\renewcommand*{\arraystretch}{1.25}
\resizebox{\textwidth}{!}{
\begin{tabular}{l|ccccc}
\hline

\hline
\multicolumn{1}{c|}{\textbf{Algorithm}} & $g_t^{\mathtt{base}}(\x)$      & $h_t^{\mathtt{base}}(\x)$              & $g_t^{\mathtt{meta}}(\x)$      & $h_t^{\mathtt{meta}}(\x)$                & $\cb_t$                                      \\ \hline
\textsf{Sword}              & $f_t(\x)$                      & $f_{t-1}(\x)$                          & $\inner{\nabla f_t(\x_t)}{\x}$ & $\inner{\nabla f_{t-1}(\bar{\x}_t)}{\x}$ & $\cb_t=\mathbf{0}$                                 \\
\textsf{Sword++}            & $\inner{\nabla f_t(\x_t)}{\x}$ & $\inner{\nabla f_{t-1}(\x_{t-1})}{\x}$ & $\inner{\nabla f_t(\x_t)}{\x}$ & $\inner{\nabla f_{t-1}(\x_{t-1})}{\x}$   & $c_{t,i} = \norm{\x_{t,i} - \x_{t-1,i}}_2^2$ \\
\textsf{Sword.optimism}     & $\inner{\nabla f_t(\x_t)}{\x}$ & $\inner{M_t}{\x}$                      & $\inner{\nabla f_t(\x_t)}{\x}$ & $\inner{M_t}{\x}$                        & $c_{t,i} = \norm{\x_{t,i} - \x_{t-1,i}}_2^2$ \\ \hline

\hline
\end{tabular}
}
\end{table}

\paragraph{Recovering Sword.} We instantiate the algorithmic template as follows: setting
\begin{itemize}
\item  base surrogate loss as $g_t^{\mathtt{base}}(\x) = f_t(\x)$ and base optimism as $h_t^{\mathtt{base}}(\x) = f_{t-1}(\x)$;
\item  meta surrogate loss as $g_t^{\mathtt{meta}}(\x) = \inner{\nabla f_t(\x_t)}{\x}$ and  meta optimism as $h_t^{\mathtt{meta}}(\x) = \inner{\nabla f_{t-1}(\bar{\x}_t)}{\x}$, as well as  correction term as $\cb_t = \mathbf{0}$.
\end{itemize}
Then, the template updates in the following way: the base-algorithm updates by
\begin{align*}
  \x_{t,i} = \Pi_{\X}\left[\xh_{t,i} - \eta_{i} \nabla f_{t-1}(\x_{t-1,i})\right],~~ \xh_{t+1,i} = \Pi_{\X}\left[\xh_{t,i}-\eta_{i}\nabla f_{t}(\x_{t,i})\right],
\end{align*}
and the meta-algorithm updates by
\begin{equation*}
    p_{t+1,i} \propto \exp\left(-\epsilon \Big(\sum_{s=1}^{t} \inner{\nabla f_s(\x_s)}{\x_{s,i}} + \inner{\nabla f_t(\bar{\x}_{t+1})}{\x_{t,i}} \Big)\right).
\end{equation*}
The update procedures precisely recover Sword as presented in Algorithms~\ref{alg:Sword-meta} and~\ref{alg:Sword-base}. Note that in Sword, there are no correction terms, since the gradient-variation dynamic regret bound is attained by guaranteeing gradient-variation meta-regret for the meta-algorithm and gradient-variation base-regret for the base-algorithm, respectively.

\paragraph{Recovering Sword$++$.} We instantiate the algorithmic template as follows: setting
\begin{itemize}
\item  base surrogate loss as $g_t^{\mathtt{base}}(\x) = \inner{\nabla f_t(\x_t)}{\x}$ and base optimism as $h_t^{\mathtt{base}}(\x) = \inner{\nabla f_{t-1}(\x_{t-1})}{\x}$;
\item meta surrogate loss as $g_t^{\mathtt{meta}}(\x) = \inner{\nabla f_t(\x_t)}{\x}$ and meta optimism as $h_t^{\mathtt{meta}}(\x) = \inner{\nabla f_{t-1}(\x_{t-1})}{\x}$, as well as  correction term $\cb_t$ as $c_{t,i} = \norm{\x_{t,i} - \x_{t-1,i}}_2^2$ with $\x_{0,1} = \mathbf{0}$.
\end{itemize}
Then, the template updates in the following way: the base-algorithm updates by
\begin{align*}
  \x_{t,i} = \Pi_{\X}\left[\xh_{t,i} - \eta_{i} \nabla f_{t-1}(\x_{t-1})\right],~~ \xh_{t+1,i} = \Pi_{\X}\left[\xh_{t,i}-\eta_{i}\nabla f_{t}(\x_{t})\right],
\end{align*}
and the meta-algorithm updates by
\begin{align*}
    p_{t+1,i} \propto \exp\bigg(-\epsilon \Big(\sum_{s=1}^{t} \inner{\nabla f_s(\x_s)}{\x_{s,i}} + \lambda \sum_{s=1}^{t+1} \norm{\x_{s,i} - \x_{s-1,i}}_2^2  + \inner{\nabla f_t(\x_{t})}{\x_{t+1,i}} \Big)\bigg).
\end{align*}
The update procedures precisely correspond to Sword++. We emphasize once more that the algorithmic updates only necessitate querying the gradient $\nabla f_t(\x_t)$ at each round $t \in [T]$.

\paragraph{Another important instantiation.} We further present another instantiation of the template that can be of independent interest. The resulting algorithm can achieve an optimistic dynamic regret bound of order $\O(\sqrt{A_T(1+P_T)})$, where $A_T = \sum_{t=1}^T \norm{\nabla f_t(\x_t) - M_t}_2^2$ measures the quality of the optimistic vectors $\{M_t\}_{t=1}^T$. We instantiate the algorithmic template as follows: setting
\begin{itemize}
\item  base surrogate loss $g_t^{\mathtt{base}}(\x) = \inner{\nabla f_t(\x_t)}{\x}$ and base optimism $h_t^{\mathtt{base}}(\x) = \inner{M_t}{\x}$;
\item meta surrogate loss $g_t^{\mathtt{meta}}(\x) = \inner{\nabla f_t(\x_t)}{\x}$ and meta optimism $h_t^{\mathtt{meta}}(\x) = \inner{M_t}{\x}$, as well as  correction term $\cb_t$ as $c_{t,i} = \norm{\x_{t,i} - \x_{t-1,i}}_2^2$ with $\x_{0,1} = \mathbf{0}$.
\end{itemize}
Then, the template updates in the following way: the base-algorithm updates by
\begin{align}
\label{eq:general-base}
  \x_{t,i} = \Pi_{\X}\left[\xh_{t,i} - \eta_{i} M_t\right],~~ \xh_{t+1,i} = \Pi_{\X}\left[\xh_{t,i}-\eta_{i}\nabla f_{t}(\x_{t})\right],
\end{align}
and the meta-algorithm updates by
\begin{align}
\label{eq:general-meta}
    p_{t+1,i} \propto \exp\bigg(-\epsilon \Big(\sum_{s=1}^{t} \inner{\nabla f_s(\x_s)}{\x_{s,i}} + \lambda \sum_{s=1}^{t+1} \norm{\x_{s,i} - \x_{s-1,i}}_2^2  + \inner{M_{t+1}}{\x_{t+1,i}} \Big)\bigg).
\end{align}
We refer to the above meta-base updates,~\eqref{eq:general-base} and~\eqref{eq:general-meta}, as \textsf{Sword.optimism}. Its dynamic regret analysis detailed in Section~\ref{sec:theoretical-guarantee}. Notice that by setting the optimism as the last-round gradient, specifically, $M_t = \nabla f_{t-1}(\x_{t-1})$, Sword.optimism recovers Sword++ exactly.

\subsection{Theoretical Guarantee} 
\label{sec:theoretical-guarantee}
In this part, we present the dynamic regret analysis for Sword.optimism, which is arguably the most general instantiation of the collaborative online ensemble template. It is straightforward to extend Theorem~\ref{thm:general-online-ensemble} for the general template presented in Section~\ref{sec:online-ensemble-template}, specifically, the meta-base updates in~\eqref{eq:general-framework-base} and~\eqref{eq:general-framework-meta}. However, the various variables in the general template may somewhat obscure the core ideas. Therefore, we choose to showcase the dynamic regret analysis for Sword.optimism, as its analysis effectively captures the essence and its algorithm is also sufficiently general (for instance, it can specialize Sword++).

\begin{myThm}
\label{thm:general-online-ensemble}
Under Assumptions~\ref{assumption:bounded-gradient} and~\ref{assumption:bounded-domain}, set the pool of candidate step sizes $\H$ as 
\begin{equation}
  \label{eq:step-size-pool-one-gradient-general}
  \H = \left\{\eta_i = \min\bigg\{\bar{\eta}, \sqrt{\frac{D^2}{8G^2T}}\cdot 2^{i-1}\bigg\} \mid i\in [N]\right\},
\end{equation}
where $N = \lceil 2^{-1} \log_2((8G^2T\bar{\eta}^2)/D^2)\rceil + 1$ is the number of candidate step sizes; further set the learning rate of the meta-algorithm as 
\begin{equation}
  \label{eq:lr-meta-solution2}
  \epsilon = \min\left\{ \bar{\epsilon}, \sqrt{\frac{\ln N}{D^2 \sum_{t=1}^T\norm{\nabla f_t(\x_t) - M_t}_2^2}} \right\}.
\end{equation}
Then, \textsf{Sword.optimism} satisfies that for any comparator sequence $\u_1,\ldots,\u_T \in \X$,
\begin{equation}
\label{eq:general-dynamic-M_T}
\begin{split}
\sum_{t=1}^T f_t(\x_t) - \sum_{t=1}^T f_t(\u_t) {}& \leq 2\sqrt{D^2(\ln N)A_T }+ 2 \sqrt{(D^2+2DP_T)A_T}\\
& + \frac{2\ln N}{\bar{\epsilon}} + \frac{2(D^2+2DP_T)}{\bar{\eta}}  + \left(\lambda-\frac{1}{4\bar{\eta}}\right) S_{x,i} - \frac{1}{4\bar{\epsilon}} S_p - \lambda S_{\mathrm{mix}}.
\end{split}
\end{equation}
In above, $A_T = \sum_{t=1}^{T} \norm{\nabla f_t(\x_t)-M_t}_2^2$ is the adaptivity term, $P_T = \sum_{t=2}^{T} \norm{\u_{t-1} - \u_{t}}_2$ is the path length, $S_{x,i} = \sum_{t=2}^T\norm{\x_{t,i}-\x_{t-1,i}}_2^2$, $S_p = \sum_{t=2}^T\norm{\p_t-\p_{t-1}}_1^2$, and $S_{\mathrm{mix}}= \sum_{t=2}^T\sum_{i=1}^Np_{t,i}\norm{\x_{t,i}-\x_{t-1,i}}_2^2$ are base stability, meta stability, and mixed stability.
\end{myThm}

The proof of Theorem~\ref{thm:general-online-ensemble} is presented in Section~\ref{sec:proof-theorem-ensemble}. Notice that by setting the correction coefficient $\lambda = 0$ and setting clipped parameters $\bar{\eta}$ and $\bar{\epsilon}$ as appropriate constants, Theorem~\ref{thm:general-online-ensemble} directly implies an $\O(\sqrt{A_T(1+P_T)})$ dynamic regret for Sword.optimism. 

As aforementioned, when setting the optimism as $M_t = \nabla f_{t-1}(\x_{t-1})$, Sword.optimism recovers Sword++. Consequently, Theorem~\ref{thm:general-online-ensemble} serves as a preliminary analysis for Sword++ by substituting $M_t = \nabla f_{t-1}(\x_{t-1})$ in the upper bound~\eqref{eq:general-dynamic-M_T}. By further combining the analysis of~\eqref{eq:Swordpp-swicthing-cost} in Remark~\ref{remark:collaboration}, we can then prove the gradient-variation bound of Theorem~\ref{thm:variation-one-gradient}, see the detailed argument in Section~\ref{sec:proof-Swordpp-dynamic-regret}. The key element is to effectively cancel out the additional positive term using negative terms and correction terms jointly, which are strategically introduced due to the collaboration between the meta and base levels.
\section{Implication, Significance, and Lower Bound}
\label{sec:implications}
In this section, we present several additional results, including the implication to small-loss dynamic regret, the implication to the worst-case dynamic regret, the significance of problem-dependent bounds, and a lower bound justification.  

\subsection{Implication to Small-Loss Dynamic Regret}
In this part, we investigate another problem-dependent quantity --- the cumulative loss of comparators defined as $F_T = \sum_{t=1}^{T} f_t(\u_t)$. 

In the conference version, we propose the Sword algorithm (presented in Section~\ref{sec:solution1}) to achieve the gradient-variation dynamic regret, and then propose a variant to attain the small-loss bound, which employs OGD as the base-algorithm and uses the vanilla Hedge with linearized surrogate loss as the meta-algorithm (i.e., choosing the optimistic vector $M_t = \mathbf{0}$ for both meta- and base-algorithms). In the current paper, we demonstrate that the improved algorithm Sword++ designed in Section~\ref{sec:solution2} itself provably achieves the small-loss dynamic regret \emph{without} any algorithmic modification. In fact, we have the following theorem regarding the small-loss bound of Sword++, whose proof is in Section~\ref{sec:proof-small-loss-bobw}.
\begin{myThm}
\label{thm:small-loss-one-gradient}
Set the parameters the same as those in Theorem~\ref{thm:variation-one-gradient}. Under Assumptions~\ref{assumption:bounded-gradient},~\ref{assumption:bounded-domain},~\ref{assumption:smoothness}, and~\ref{assumption:non-negative}, Sword++ satisfies that
\begin{align*}
   \sum_{t=1}^T f_t(\x_t) - \sum_{t=1}^T f_t(\u_t) \leq \O\left(\sqrt{(1+P_T+F_T)(1+P_T)}\right),
\end{align*}
and hence achieves the best-of-both-worlds guarantee:
\begin{align*}
   \sum_{t=1}^T f_t(\x_t) - \sum_{t=1}^T f_t(\u_t) \leq \O\left(\sqrt{(1+P_T+ \min\{V_T, F_T\})(1+P_T)}\right).
\end{align*}
The bounds hold for any comparator sequence $\u_1,\ldots,\u_T \in \X$.
\end{myThm}
Comparing with Theorem~\ref{thm:variation-one-gradient}, one more assumption (Assumption~\ref{assumption:non-negative}) is required. As mentioned, this non-negativity assumption is a precondition for establishing the self-bounding property for smooth functions~\citep{NIPS'10:smooth,lecture-note:Ashok}, and thus is commonly used in the small-loss analysis of online learning and stochastic optimization~\citep{NIPS'10:smooth,NIPS'11:minibatch,ICML'13:logT-projection,ICML19:Zhang-Adaptive-Smooth,COLT19:zhanglj-FASA}.

\begin{myRemark}
Our conference version~\citep{NIPS'20:sword} achieves the best-of-both-worlds bound in a different way, in which we use a heterogeneous model selection method of learning an optimism~\citep{conf/colt/RakhlinS13} since different optimistic vectors are used for the small-loss and gradient-variation bounds. As such, three algorithms (Sword$_\text{var}$, Sword$_\text{small}$, and Sword$_\text{best}$) are designed to achieve the three different bounds (gradient-variation, small-loss, and best-of-both-worlds bounds) respectively. By contrast, Theorem~\ref{thm:small-loss-one-gradient} indicates that the Sword++ algorithm can achieve \emph{all} the three problem-dependent dynamic regret bounds without any modifications, owing to its one-gradient query complexity property. This also, to some extent, demonstrates the fundamental importance of achieving gradient-variation bounds --- sometimes this can directly imply a small-loss bound in the analysis.
\endenv 
\end{myRemark}

\begin{myRemark}
\label{remark:discussion-VT-lower-bound}
Comparing to the $\O(\sqrt{T(1+P_T)})$ minimax rate, Theorem~\ref{thm:small-loss-one-gradient} replaces the dependence on $T$ by the problem-dependent quantity $P_T + \min\{V_T,F_T\}$ and thus achieves dual adaptivity in terms of both gradient variation $V_T$ and the small-loss quantity $F_T$. Furthermore, one may wonder whether it is possible to replace $T$ by $\min\{V_T,F_T\}$ only. This requires a lower bound argument and we only have a partial answer. Specifically, we prove in Theorem~\ref{thm:lower-bound} that no algorithm can achieve an $\O(\sqrt{(1+F_T)(1+P_T)})$ bound in general. Nevertheless, we fail to provide a similar reasoning for the gradient-variation bound. Indeed, we have the following conjectures. For the multi-gradient feedback model, we are inclined to believe that the $\O(\sqrt{(1+P_T + V_T)(1+P_T)})$ rate may not be optimal and it might be possible to achieve $\O(\sqrt{(1+V_T)(1+P_T)})$. For the one-gradient feedback model, we conjecture that our obtained rate has already been optimal. We leave this problem-dependent lower bound as an important future work to investigate. 
\endenv 
\end{myRemark}

\subsection{Implication to Worst-Case Dynamic Regret}
In this part, we present the implication of the universal dynamic regret to the worst-case dynamic regret. As discussed in Section~\ref{sec:related-work-dynamic-regret}, for the worst-case dynamic regret, there are two kinds of regularities: path length $P_T^* = \sum_{t=2}^T \norm{\x^*_{t-1} - \x_t^*}_2$ and function variation $V_T^f = \sum_{t=2}^T\sup_{\x\in\X}\vert f_{t-1}(\x) - f_{t}(\x)\vert$. The following theorem provides a reduction to both.
\begin{myThm}
\label{thm:implication-worst-case-bound}
Let $A_T \in \R_{+}$ be a certain adaptivity term. Suppose there exists an algorithm $\mathcal{A}$ that ensures the following guarantee: for any comparator sequence $\u_1,\ldots\u_T \in\X$ with path length $P_T = \sum_{t=2}^T \norm{\u_t-\u_{t-1}}_2$, it holds that 
\begin{align}
\textnormal{D-Regret}_T(\u_1,\dots,\u_T) \leq \sqrt{A_T (D+P_T)},\label{eq:implication-UDR}
\end{align}
for some constant $D>0$. Then $\mathcal{A}$ enjoys the worst-case dynamic regret bound: 
\begin{align}
\textnormal{D-Regret}_T(\x^*_1,\dots,\x^*_T) \leq 3\sqrt{DA_T} + \min\left\{\sqrt{A_TP^*_T},5D^{1/3}T^{1/3}A_T^{1/3}(V_T^f)^{1/3}\right\}. \label{eq:implication-WDR}
\end{align}
\end{myThm}
Theorem~\ref{thm:implication-worst-case-bound} demonstrates that an $\O(\sqrt{A_T(1+P_T)})$ universal dynamic regret bound can \emph{directly} imply an $\O(\sqrt{A_T}+\min\{\sqrt{A_TP_T^*},(TA_TV_T^f)^{1/3}\})$ worst-case dynamic regret bound. A typical choice of this adaptivity term is $A_T = \sum_{t=1}^{T} \norm{\nabla f_t(\x_t)-M_t}_2^2$ that measures the quality of optimistic gradient vectors $\{M_t\}_{t=1}^T$. Then, the implication matches the best-known optimistic worst-case dynamic regret bound presented in~\citep{AISTATS'15:dynamic-optimistic,UAI'20:simple}, taking the best of the path-length and function-variation regularities. It is worth noting that~\citet{AISTATS'15:dynamic-optimistic} achieve this result through a carefully designed  doubling trick scheme, which will introduce a potentially non-convex inner optimization $\sup_{\x \in \X} \abs{f_t(\x) - f_{t-1}(\x)}$ at iteration $t \in[T]$. In contrast, our Theorem~\ref{thm:implication-worst-case-bound} demonstrates that when the algorithm achieves an $\O(\sqrt{A_T(1+P_T)})$ universal dynamic regret, it automatically obtains the desired worst-case dynamic regret bounds. Notably, our proposed Sword.optimism algorithm (see the last instantiation in Section~\ref{sec:framework-instantiations}) already satisfies this requirement using the collaborative online ensemble framework.

The proof of Theorem~\ref{thm:implication-worst-case-bound} can be found in Section~\ref{sec:implication-worst-case-bound}.
Given the universal dynamic regret bound~\eqref{eq:implication-UDR}, one can immediately derive an $\O(\sqrt{A_T}+\sqrt{A_TP_T^*})$ worst-case path-length bound by setting $\u_t = \x_t^*$ for any $t \in [T]$, but it is less straightforward to obtain the $\O(\sqrt{A_T}+T^{1/3}A_T^{1/3}(V_T^f)^{1/3})$ function-variation bound. To achieve so, we need to introduce a reference comparator sequence that exhibits piecewise-stationary behavior. The desired function-variation bound is then achievable by optimally tuning the stationary length of the sequence during the analysis. The idea was introduced in~\citet[Appendix~A.2]{UAI'20:simple}, but an explicit reduction was not provided. We offer a clear presentation of the results. 

Moreover, in Theorem~\ref{thm:implication-worst-case-bound}, we focus on the $\O(\sqrt{A_T(1+P_T)})$ universal dynamic regret bound, which incorporates the general adaptivity term $A_T$. Using a similar analysis, we can also convert the gradient-variation/small-loss universal dynamic regret bounds, attained by Sword and Sword++, into the worst-case dynamic regret bounds. Details are omitted here.

\subsection{Significance of Problem-Dependent Bounds}
\label{sec:example}
In this part, we justify the significance of our problem-dependent dynamic regret bounds. We present two concrete problem instances to demonstrate that it is possible to achieve a \emph{constant} dynamic regret bound instead of the minimax rate $\O(\sqrt{T(1+P_T)})$.

We consider the quadratic loss function of the form $f_t(x) = \frac{1}{2}(a_t \cdot x-b_t)^2$, where $a_t \neq 0$ and $x \in \X \triangleq [-1,1]$. Clearly, the online function $f_t: \R \mapsto \R$ is convex and smooth. Denote by $T$ the time horizon. The coefficients $a_t$ and $b_t$ will be specified below in each instance.
\begin{myInstance}[{$V_T \ll F_T$}]
Let the time horizon $T = 2K+1$ be odd with $K > 2$. We set the coefficients $a_t = 0.5 - \frac{t-1}{T}$ and $b_t =1 $ for all $t \in [T]$. 
\end{myInstance}
We set the comparator $u_t$ to be the minimizer of $f_t$, i.e, $u_t = x_t^* = \argmin_{x \in \X} f_t(x)$. Clearly, $u_t = 1$ for $t \in [K+1]$, and $u_t = -1$ for $t = K+2,\ldots,T$. A direct calculation shows
\begin{align*}
V_T ={}& \sum_{t=2}^T\sup_{x\in\X}\vert(a_{t-1}^2-a_t^2)x-(a_{t-1}-a_t)\vert^2 = \sum_{t=2}^T\sup_{x\in\X}\left\vert\left(\frac{T-2t+3}{T^2}\right)\cdot x-\frac{1}{T}\right\vert^2\\
={}&\sum_{t=2}^{K+2}\left(\frac{2T-(2t-3)}{T^2}\right)^2+\sum_{t=K+3}^T\left(\frac{2t-3}{T^2}\right)^2 \leq \sum_{t=2}^T \left(\frac{2}{T}\right)^2 = \O(1).
\end{align*}
\begin{align*}
  F_T = {} & \sum_{t=1}^{T} \frac{1}{2}(a_t u_t -b_t)^2 = \sum_{t=1}^{K+1} \frac{1}{2}\left(0.5-\frac{t-1}{T} - 1\right)^2 + \sum_{t=K+2}^{T} \frac{1}{2}\left(-0.5+\frac{t-1}{T} - 1\right)^2 = \Theta(T).
\end{align*}
We can observe that the gradient variation $V_T =\O(1)$ is significantly smaller than the small-loss quantity $F_T = \Theta(T)$ in this problem instance; and meanwhile, the path length is $P_T = \O(1)$. Then, the minimax dynamic regret bound is $\O(\sqrt{T(1+P_T)}) = \O(\sqrt{T})$; the small-loss bound is $\O(\sqrt{(1+P_T+F_T)(1+P_T)}) = \O(\sqrt{T})$; and the gradient-variation bound is $\O(\sqrt{(1+P_T+V_T)(1+P_T)}) = \O(1)$. As a result, by exploiting the problem's structure, Sword++ can enjoy a \emph{constant} dynamic regret against $u_1,\ldots,u_T$  in this scenario, significantly improving upon the problem-independent bound of order $\O(\sqrt{T})$.

\begin{myInstance}[{$F_T \ll V_T$}]
Let the time horizon $T = 2K$ be even. During the first half iterations, $(a_t,b_t)$ is set as $(1,1)$ on odd rounds and $(0.5,0.5)$ on even rounds. During the remaining iterations, $(a_t,b_t)$ is set as $(1,-1)$ on odd rounds and $(0.5,-0.5)$ on even rounds.
\end{myInstance}
We set the comparator $u_t$ to be the minimizer of $f_t$, i.e, $u_t = x_t^* = \argmin_{x \in \X} f_t(x)$. Clearly, $u_t = 1$ for $t \in [K]$, and $u_t = -1$ for $t = K+1,\ldots,T$. A direct calculation shows
\begin{align*}
    V_T =\sum_{t=2}^{T} \sup_{x\in \X} \abs{(a_{t-1}^2 - a_t^2)x - (a_{t-1}b_{t-1} - a_tb_t)}^2 = \Theta(T),\qquad F_T = 0.
\end{align*}
We can see that the small-loss quantity $F_T = 0$ is considerably smaller than the gradient variation $V_T = \Theta(T)$ in this scenario; and meanwhile, the path length is $P_T = \O(1)$. Then, the minimax dynamic regret bound is $\O(\sqrt{T(1+P_T)}) = \O(\sqrt{T})$; the gradient-variation bound is $\O(\sqrt{(1+P_T+V_T)(1+P_T)}) = \O(\sqrt{T})$; and the small-loss bound is $\O(\sqrt{(1+P_T+F_T)(1+P_T)}) = \O(1)$. As a result, by exploiting the problem's structure, Sword++ can enjoy a \emph{constant} dynamic regret against $u_1,\ldots,u_T$ in this scenario, significantly improving upon the problem-independent bound of order $\O(\sqrt{T})$. 

\subsection{A Lower Bound}
\label{sec:lower-bound}
We here present a lower bound for dynamic regret of convex and smooth functions. 

\begin{myThm}
\label{thm:lower-bound}
For any online algorithm $\mathcal{A}$, there always exists a sequence of convex and smooth functions $f_1,\ldots,f_T$ and a sequence of comparators $\u_1,\ldots,\u_T$, such that, for any constant $c>0$,
\begin{equation}
    \label{eq:lower-bound}
    \sum_{t=1}^{T} f_t(\x_t) - \sum_{t=1}^T f_t(\u_t) >c \sqrt{(1+F_T)(1+P_T)}.
\end{equation}
when the time horizon $T$ is sufficiently large.
\end{myThm}

For the static regret bound, the worst-case minimax rate $\O(\sqrt{T})$ can be improved to $\O({\sqrt{F_T}})$ or $\O({\sqrt{V_T}})$ by substituting the dependence on $T$ with problem-dependent quantities. A natural question for universal dynamic regret is whether it is possible to also attain an $\O(\sqrt{(1 + \min\{V_T, F_T\})(1 + P_T)})$ bound that improves the minimax rate $\O(\sqrt{T(1 + P_T)})$. Theorem~\ref{thm:lower-bound} shows that no algorithm can achieve the $\O(\sqrt{(1 + F_T)(1 + P_T)})$ universal dynamic regret bound. We provide the proof in Section~\ref{sec:proof-lower-bound}, where the probabilistic method is applied to show the contradiction. In the constructed problem instance, the small-loss quantity is always $F_T  = 0$, and there exist a certain online function sequence $\{f_t\}_{t=1}^T$ such that the dynamic regret is lower bound by $\Omega(T)$. Therefore, the $\O(\sqrt{(1 + F_T)(1 + P_T)})$ upper bound would violate this lower bound, rendering it unfeasible. Nevertheless, as the gradient variation $V_T = \sum_{t=2}^T \sup_{\x \in \X} \norm{\nabla f_t(\x) - \nabla f_{t-1}(\x)}_2^2$ is larger than $0$ in this instance, we cannot rule out the possibility of the $\O(\sqrt{(1+V_T)(1+P_T)})$ upper bound.
\section{Proofs}
\label{sec:appendix-analysis}
This section presents the proofs of main results, including Theorem~\ref{thm:dynamic-regret-OMD-generic} of Section~\ref{sec:problem-setup-OMD}, Theorems~\ref{thm:variation-meta-regret}--\ref{thm:variation-one-gradient} of Section~\ref{sec:gradient-variation}, Theorem~\ref{thm:general-online-ensemble} of Section~\ref{sec:framework-collaborative-OE}, and Theorems~\ref{thm:small-loss-one-gradient}--\ref{thm:lower-bound} of Section~\ref{sec:implications}.

\subsection{{Proof of Theorem~\ref{thm:dynamic-regret-OMD-generic}}} 
\begin{proof}
The instantaneous dynamic regret can be upper bounded and decomposed as
\begin{align*}
    f_t(\x_t) - f_t(\u_t) \leq \underbrace{\inner{\nabla f_t(\x_t) - M_t}{\x_t - \xh_{t+1}}}_{\term{a}} + \underbrace{\inner{M_t}{\x_t - \xh_{t+1}}}_{\term{b}} + \underbrace{\inner{\nabla f_t(\x_t)}{\xh_{t+1} - \u_t}}_{\term{c}}.
\end{align*}
In the following, we use the stability lemma (Lemma~\ref{lemma:stability-OMD}) to bound term (a) and appeal to the Bregman proximal inequality (Lemma~\ref{lemma:bregman-divergence}) to bound term (b) and term (c).

We first investigate term (a). Intuitively, the prediction $\x_t$ should be close the $\xh_{t+1}$ when the optimistic vector $M_t$ is close to the gradient of the next iteration $\nabla f_t(\x_t)$. The intuition is formalized in the stability lemma~\citep[Propostion 7]{COLT'12:variation-Yang}, as restated in Lemma~\ref{lemma:stability-OMD} of Appendix~\ref{sec:appendix-technical-lemmas}, which implies $\norm{\x_t - \xh_{t+1}} \leq \eta_t \norm{\nabla f_t(\x_t) - M_t}_*$ and consequently,
\begin{equation*}
    \begin{split}   
    \term{a}  \leq \norm{\nabla f_t(\x_t) - M_t}_* \norm{\x_t - \xh_{t+1}} \leq \eta_t \norm{\nabla f_t(\x_t) - M_t}_*^2.
    \end{split}
\end{equation*}

We now analyze term (b) and term (c). By the Bregman proximal inequality (Lemma~\ref{lemma:bregman-divergence}) and the \textsc{Optimistic OMD} update step $\x_{t} =\argmin_{\x \in \X} \{\eta_t \inner{M_t}{\x} + \Div{\x}{\xh_{t}}\}$, we have 
\begin{equation*}
    \term{b} = \inner{M_t}{\x_t - \xh_{t+1}} \leq \frac{1}{\eta_t} \Big( \Div{\xh_{t+1}}{\xh_t} - \Div{\xh_{t+1}}{\x_t} - \Div{\x_{t}}{\xh_t}\Big).
\end{equation*}
Similarly, the update $\xh_{t+1} = \argmin_{\x \in \X} \{\eta_{t}\inner{\nabla f_{t}(\x_{t})}{\x} + \Div{\x}{\xh_{t}}\}$ implies
\begin{equation*}
    \term{c} = \inner{\nabla f_t(\x_t)}{\xh_{t+1} - \u_t} \leq \frac{1}{\eta_t} \Big( \Div{\u_t}{\xh_t} - \Div{\u_t}{\xh_{t+1}} - \Div{\xh_{t+1}}{\xh_t}\Big).
\end{equation*}
Combining the three upper bounds completes the proof.
\end{proof}

\subsection{Proof of Theorem~\ref{thm:variation-meta-regret}}
\label{sec:proof-Sword-meta}
\begin{proof}
Substituting the definitions of feedback loss and optimism into Lemma~\ref{lemma:OptimisticHedge} yields
\begin{equation*}
\sum_{t=1}^T\inner{\nabla f_t(\x_t)}{\x_t-\x_{t,i}}\leq \epsilon D^2 \norm{\nabla f_t(\x_t) - \nabla f_{t-1}(\xb_t)}_2^2 + \frac{\ln N}{\epsilon} - \frac{1}{4\epsilon}\sum_{t=2}^{T} \norm{\p_{t} - \p_{t-1}}_1^2.
\end{equation*}
Together with the derivations in~\eqref{eq:Sword-swicthing-cost} and~\eqref{eq:Sword-adaptivity-expand}, this implies
\begin{equation*}
\sum_{t=1}^T\inner{\nabla f_t(\x_t)}{\x_t-\x_{t,i}}\leq  2\epsilon D^2(G^2 + V_T) + \frac{\ln N}{\epsilon} + \left(2D^4L^2\epsilon- \frac{1}{4\epsilon}\right)\sum_{t=2}^{T} \norm{\p_{t} - \p_{t-1}}_1^2.
\end{equation*}
Setting the learning rate as $\epsilon = \min\{1/(4D^2L),\sqrt{(\ln N)/(2D^2(G^2 + V_T))}\}$, by Lemma~\ref{lemma:inequality} we obtain an $2D\sqrt{2 (G^2 + V_T)\ln N} + 8D^2L\ln N$ upper bound, which ends the proof as $f_t(\x_t) - f_t(\x_{t,i}) \leq \inner{\nabla f_t(\x_t)}{\x_t-\x_{t,i}}$ holds due to the convexity of online functions. 
\end{proof}

\subsection{Proof of Theorem~\ref{thm:dynamic-var}}
\label{sec:proof-Sword-dynamic-regret}
\begin{proof} As stated in~\eqref{eq:meta-base-decompose}, dynamic regret can be decomposed into the meta-regret and base-regret, and the decomposition holds for any base-learner's index $i \in [N]$. 

\paragraph{Upper bound of meta-regret.} Theorem~\ref{thm:variation-meta-regret} shows that for any $i \in [N]$,
\begin{equation}
\label{eq:Sword-meta-regret-final}
  \meta = \sum_{t=1}^T f_t(\x_t) - \sum_{t=1}^Tf_t(\x_{t,i}) \leq 2D\sqrt{2(4G^2 + V_T)\ln N} + 8D^2L\ln N.
\end{equation}

\paragraph{Upper bound of base-regret.} Lemma~\ref{lemma:OEGD-variation} indicates that for any index $i \in [N]$,
\begin{equation}
\label{eq:proof-thm5-base-D}
\base = \sum_{t=1}^{T} f_t(\x_{t,i}) - \sum_{t=1}^{T} f_t(\u_t) \leq \eta_i (G^2 + 2V_T) + \frac{1}{2\eta_i}(D^2 + 2DP_T),
\end{equation}
where $\eta_i \in \H$ is the step size associated with the $i$-th base-learner. Recall in Lemma~\ref{lemma:OEGD-variation}, we require the step size $\eta_i \leq 1/(4L)$ to leverage the negative term in the regret analysis. Denote by $\eta^* = \sqrt{(D^2 + 2DP_T)/(G^2 + 2V_T)}$ the optimal step size without considering the constraint and by $\eta^\dag = \min\{{1}/{(4L)},\eta^*\}$ the clipped one. Notice that we have $\eta_1 = \sqrt{D^2/(8G^2T)}$, $\eta_N = 1/(4L)$, and $\eta_1 \leq \eta^\dag \leq \eta_N$, due to path length $P_T \in [0,DT]$ and gradient variation $V_T \leq 4G^2(T-1)$ by Assumption~\ref{assumption:bounded-gradient} and Assumption~\ref{assumption:bounded-domain}. More importantly, owing to the construction of the step size pool $\H$ in~\eqref{eq:step-size-pool-variation}, we can assure that there exists an index $i^* \in [N]$ such that $\eta_{i^*} \leq \eta^\dag \leq \eta_{i^*+1} = 2\eta_{i^*}$. As a result, we pick $i = i^*$ in~\eqref{eq:proof-thm5-base-D} and get  
\begin{align}
  \base \leq {} & \eta_{i^*} (G^2+ 2V_T) + \frac{D^2+2DP_T}{2\eta_{i^*}}  \nonumber \\
  \leq {} & \eta^\dag (G^2 + 2V_T) + \frac{D^2 + 2DP_T}{\eta^\dag} \label{eq:upper-step-1-variation}\\
  \leq {} & 2 \sqrt{(G^2+2V_T)(D^2 + 2DP_T)} + 8L(D^2 + 2DP_T) \label{eq:upper-step-2-variation}\\
  \leq {} & \O\Big( \sqrt{(1 + P_T + V_T)(1 + P_T)} \Big). \label{eq:base-regret-variation}  
\end{align}
In above,~\eqref{eq:upper-step-2-variation} holds because $\eta^\dag$ is either $\eta^*$ or $1/(4L)$ and 
\begin{itemize}
  \item when $\eta^\dag = \eta^*$, the right hand side of~\eqref{eq:upper-step-1-variation} $ = 2 \sqrt{(G^2 + 2V_T)(D^2 + 2DP_T)}$;
  \item when $\eta^\dag = 1/(4L)$, we have $\eta^* = \sqrt{(D^2 + 2DP_T)/(G^2 + 2V_T)}\geq \frac{1}{4L}$, which implies that $\frac{1}{4L}(G^2+2V_T)\leq 4L (D^2 + 2DP_T)$. Under such a case, the right hand side of~\eqref{eq:upper-step-1-variation} $ = 4L(D^2 + 2DP_T) + \frac{1}{4L}(G^2 + 2V_T) \leq 8L(D^2 + 2DP_T)$.
\end{itemize} 
Combining two upper bounds yields~\eqref{eq:upper-step-2-variation} and further obtains~\eqref{eq:base-regret-variation}.

\paragraph{Upper bound of overall dynamic regret.}
Note that the meta-base regret decomposition~\eqref{eq:meta-base-decompose} and meta-regret upper bound~\eqref{eq:Sword-meta-regret-final} hold for any index $i \in [N]$. Hence, we can choose the index as $i^*$ as specified above and further combine the base-regret upper bound~\eqref{eq:proof-thm5-base-D} to achieve the final desired result. Hence, we complete the proof of Theorem~\ref{thm:dynamic-var}.
\end{proof}

\subsection{Proof of Theorem~\ref{thm:variation-one-gradient}}
\label{sec:proof-Swordpp-dynamic-regret}
\begin{proof}
Since Sword++ is essentially an instantiation of the collaborative online ensemble framework, we prove its dynamic regret building upon the general result of Theorem~\ref{thm:general-online-ensemble}.

We substitute $M_t = \nabla f_{t-1}(\x_{t-1})$ into~\eqref{eq:general-dynamic-M_T} of Theorem~\ref{thm:general-online-ensemble} and notice that
\begin{align}
A_T\leq{}& G^2+ 2\sum_{t=2}^T \norm{\nabla f_t(\x_t) - \nabla f_{t-1}(\x_{t})}_2^2 + 2\sum_{t=2}^T \norm{\nabla f_{t-1}(\x_t) - \nabla f_{t-1}(\x_{t-1})}_2^2\notag\\
\leq {}&G^2 + 2\sup_{\x\in\X}\sum_{t=2}^T \norm{\nabla f_t(\x) - \nabla f_{t-1}(\x)}_2^2 + 2L^2\sum_{t=2}^T \norm{\x_t - \x_{t-1}}_2^2\notag\\
\leq{}& G^2 + 2V_T + 4L^2 S_{\mathrm{mix}} + 4D^2L^2 S_p.\label{eq:proof-fixvar-1}
\end{align}
As a result, the first term of~\eqref{eq:general-dynamic-M_T} of Theorem~\ref{thm:general-online-ensemble} can be further bounded by
\begin{align}
{}&2\sqrt{D^2(\ln N)A_T}\notag\\
\leq{}&2\sqrt{D^2(\ln N)\left(G^2+2V_T + 4L^2 S_{\mathrm{mix}} + 4D^2L^2 S_p \right)}\notag\\
\leq{}&2\sqrt{D^2(\ln N)\left(G^2+2V_T\right)} + 2\sqrt{D^2(\ln N) (4L^2 S_{\mathrm{mix}} + 4D^2L^2 S_p )}\notag\\
\leq{}&2\sqrt{D^2(\ln N)\left(G^2+2V_T\right)} + \frac{2\ln N}{\bar{\epsilon}} + 8\bar{\epsilon}D^2L^2S_{\mathrm{mix}} + 8\bar{\epsilon}D^4L^2S_p\label{eq:proof-fixvar-meta},
\end{align}
where the last inequality is a consequence of the AM-GM inequality. Using a similar argument, we can bound the second term of~\eqref{eq:general-dynamic-M_T} of Theorem~\ref{thm:general-online-ensemble} by
\begin{align}
&2 \sqrt{(D^2+2DP_T) A_T}\notag\\
\leq{}&2 \sqrt{(D^2+2DP_T)(G^2+2V_T + 4L^2S_{\mathrm{mix}} + 4D^2L^2S_p)}\notag\\
\leq{}&2 \sqrt{(D^2+2DP_T)(G^2+2V_T)} + \frac{2D^2+4DP_T}{\bar{\eta}} + 8\bar{\eta}L^2S_{\mathrm{mix}} + 8\bar{\eta}D^2L^2S_p\label{eq:proof-fixvar-expert}.
\end{align}

Plugging~\eqref{eq:proof-fixvar-meta} and~\eqref{eq:proof-fixvar-expert} into~\eqref{eq:general-dynamic-M_T}, we get the following dynamic regret bound,
\begin{align}
{} & \sum_{t=1}^T f_t(\x_t) - \sum_{t=1}^T f_t(\u_t)\notag\\
\leq{}&2\sqrt{\ln N\left(G^2D^2+2D^2V_T \right)}+ 2 \sqrt{(D^2+2DP_T)(G^2+2V_T)} + \frac{4\ln N}{\bar{\epsilon}} + \frac{4(D^2+2DP_T)}{\bar{\eta}} \notag \\
&\quad  + \left(\lambda-\frac{1}{4\bar{\eta}}\right) S_{x,i} + \bigg( 8\bar{\eta}D^2L^2+ 8\bar{\epsilon}D^4L^2 - \frac{1}{4\bar{\epsilon}}\bigg) S_p + \bigg( 8\bar{\eta}L^2+ 8\bar{\epsilon}D^2L^2 - \lambda\bigg) S_{\mathrm{mix}} \label{eq:exhibit-collaboration}.
\end{align} 
We complete the proof by dropping the last three non-positive terms, which is ensured by the parameter configurations $\lambda = 2L$, $\bar{\eta} = 1/(8L)$ and $\bar{\epsilon} = 1/ (8D^2L)$. We finally mention that the $\ln N = \O(\log \log T)$ term is treated as a constant throughout the paper. Actually, this term can be improved to $\log \log P_T$ by imposing a non-uniform prior over base-learners~\citep[Proof of Theorem 3]{NIPS'18:Zhang-Ader}. Details are omitted here.
\end{proof}

\begin{myRemark}[Collaboration in Regret Analysis]
\label{remark:collaboration-proof}
The derivation uses a fixed learning rate for the meta-algorithm, which not only simplifies the proof but also more effectively illustrates the collaboration of meta-base two layers in the analysis. The analysis in~\eqref{eq:exhibit-collaboration} highlights the crucial role of \emph{collaboration} between meta and base layers. The positive terms --- base stability $S_{x,i}$, meta stability $S_p$, and  mixed stability $S_{\mathrm{mix}}$ --- \emph{cannot} be cancelled solely by negative terms within their respective layer. Instead, they necessitate additional negative terms, either from regret analysis or algorithmic corrections, to help cancel out.
\endenv \end{myRemark}

\subsection{Proof of Theorem~\ref{thm:general-online-ensemble}}
\label{sec:proof-theorem-ensemble}
\begin{proof}
The proof shares the same spirit with that of Theorem~\ref{thm:dynamic-var}, where we decompose the overall dynamic regret into the meta-regret and base-regret. The difference is that we now use a linearized surrogate loss function to substitute the original loss function. Indeed, 
\begin{equation}
\label{eq:dynamic-regret-surrogate-decompose-general}
\begin{split}
\sum_{t=1}^T f_t(\x_t) - \sum_{t=1}^T f_t(\u_t)  \leq  \underbrace{\sum_{t=1}^T \inner{\nabla f_t(\x_t)}{\x_t - \x_{t,i}}}_{\meta} - \underbrace{\sum_{t=1}^T \inner{\nabla f_t(\x_t)}{\x_{t,i} - \u_t}}_{\base}.
\end{split}
\end{equation}
Notably, the above meta-base regret decomposition holds for any base-learner's index $i \in [N]$. In the following, we upper bound these two terms respectively. 

\paragraph{Upper bound of meta-regret.} According to the definitions of the feedback loss $\ellb_t$ and the optimism $\m_{t}$, see the definition below~\eqref{eq:general-framework-meta}, we can rewrite the meta-regret as
\begin{align}
\meta = {} & \sum_{t=1}^{T} \inner{\p_t}{\ellb_t} - \sum_{t=1}^{T} \ell_{t,i} - \lambda\sum_{t=1}^T\sum_{i=1}^Np_{t,i}\norm{\x_{t,i}-\x_{t-1,i}}_2^2 + \lambda\sum_{t=1}^T\norm{\x_{t,i}-\x_{t-1,i}}_2^2.\notag
\end{align}
We use Lemma~\ref{lemma:OptimisticHedge} and the setting of step size $\epsilon = \min\{ \bar{\epsilon}, \sqrt{(\ln N)/(D^2 A_T)}\}$ to get 
\begin{align}
\sum_{t=1}^{T} \inner{\p_t}{\ellb_t} - \sum_{t=1}^{T} \ell_{t,i} \leq{}&\epsilon D^2\sum_{t=1}^{T} \norm{\nabla f_t(\x_t)-M_t}_{2}^2 + \frac{\ln N}{\epsilon} - \frac{1}{4\epsilon} \sum_{t=2}^{T}  \norm{\p_{t} - \p_{t-1}}_1^2\notag\\
\leq {} & 2\sqrt{D^2 (\ln N)A_T} + \frac{2\ln N}{\bar{\epsilon}} - \frac{1}{4\bar{\epsilon}}\sum_{t=2}^{T}  \norm{\p_{t} - \p_{t-1}}_1^2.\notag
\end{align}
Combining above two inequalities, we obtain
\begin{align}
\meta\leq{}& 2\sqrt{D^2 (\ln N)A_T } + \frac{2\ln N}{\bar{\epsilon}}- \frac{1}{4\bar{\epsilon}}\sum_{t=2}^{T}  \norm{\p_{t}- \p_{t-1}}_1^2  \notag\\
 {}& - \lambda\sum_{t=1}^T\sum_{i=1}^Np_{t,i}\norm{\x_{t,i}-\x_{t-1,i}}_2^2 + \lambda\sum_{t=1}^T\norm{\x_{t,i}-\x_{t-1,i}}_2^2.\label{eq:proof-thmfix-meta}
\end{align}

\paragraph{Upper bound of base-regret.} By Lemma~\ref{lemma:OEGD-variation}, we obtain the base-regret for any $i\in[N]$,
\begin{align}
\base \leq \eta_{i} A_T+ \frac{D^2 + 2DP_T}{2\eta_{i}} - \frac{1}{4\bar{\eta}}\sum_{t=2}^{T} \norm{\x_{t,i} - \x_{t-1,i}}_2^2\label{eq:proof-thmfix-base}.
\end{align}

\paragraph{Upper bound of overall dynamic regret.} Combining the meta-regret~\eqref{eq:proof-thmfix-meta} and the base-regret~\eqref{eq:proof-thmfix-base} with the decomposition~\eqref{eq:dynamic-regret-surrogate-decompose-general}, for any $i\in[N]$, we arrive at
\begin{align}
{} & \sum_{t=1}^T f_t(\x_t) - \sum_{t=1}^T f_t(\u_t)\nonumber \\
{}&\leq2\sqrt{D^2(\ln N)A_T }+ \eta_{i}A_T + \frac{D^2 + 2DP_T}{2\eta_{i}}  + \frac{2\ln N}{\bar{\epsilon}} \notag\\
&+ \left(\lambda-\frac{1}{4\bar{\eta}}\right)\sum_{t=2}^T\norm{\x_{t,i}-\x_{t-1,i}}_2^2 - \frac{1}{4\bar{\epsilon}} \sum_{t=2}^{T}  \norm{\p_{t} - \p_{t-1}}_1^2 - \lambda\sum_{t=1}^T\sum_{i=1}^Np_{t,i}\norm{\x_{t,i}-\x_{t-1,i}}_2^2.\label{eq:proof-thmfix-overall-A}
\end{align}
Here, we remain to choose the best base-learner to make the term $\eta_i A_T + \frac{D^2+2DP_T}{2\eta_i}$ tightest possible. Note that the optimal step size is $\eta^* = \sqrt{(D^2+2DP_T)/A_T}$, but nevertheless, the step size we should identify is $\eta^\dag = \min\{\eta^*,\bar{\eta}\}$ due to the threshold in the step size pool~\eqref{eq:step-size-pool-one-gradient-general}. It can be verified that candidate step sizes range from $\eta_{1} = \sqrt{\frac{D^2}{8G^2T}}$ to $\eta_{N} = \bar{\eta}$.
\begin{itemize}
    \item when $\eta^\dag = \eta^*$, there must be an index $i^*$ satisfying $\eta_{i^*} \leq \eta^* \leq \eta_{i^*+1} = 2\eta_{i^*}$. We choose $i = i^*$ and obtain $ \eta_{i^*} A_T + \frac{D^2 + 2D P_T}{2\eta_{i^*}} \leq \eta^* A_T + \frac{D^2 + 2D P_T}{\eta^*} = 2 \sqrt{(D^2+2DP_T)A_T}$;
    \item when $\eta^\dag = \bar{\eta}$, we will choose the compared index as $i = N$ and obtain that  $\eta_N A_T + \frac{D^2 + 2D P_T}{2\eta_N} = \bar{\eta} A_T + \frac{D^2 + 2D P_T}{2\bar{\eta}} \leq (2D^2+4DP_T)/\bar{\eta}$.
\end{itemize}
As a result, taking both cases into account completes the proof.
\end{proof}
\subsection{Proof of Theorem~\ref{thm:small-loss-one-gradient}}
\label{sec:proof-small-loss-bobw}
\begin{proof}
The proof shares the same spirit as that of Theorem~\ref{thm:variation-one-gradient}, whereas we upper bound the adaptivity term in a different way to achieve the small-loss bound. Specifically, we  convert the adaptivity term to the cumulative loss of decisions defined by $F_T^X = \sum_{t=1}^T f_t(\x_t)$.
\begin{align}
  A_T \leq{}& \norm{\nabla f_1(\x_1)}_2^2 + 2\sum_{t=2}^T\norm{\nabla f_t(\x_t)}_2^2 + 2\sum_{t=2}^T\norm{\nabla f_{t-1}(\x_{t-1})}_2^2\notag\\
\leq {}&  8L\sum_{t=1}^Tf_t(\x_t) + 8L\sum_{t=2}^T f_{t-1}(\x_{t-1}) \leq 16L \sum_{t=1}^T f_t(\x_t) = 16LF_T^X \notag,
\end{align}
where the second inequality comes from the self-bounding property of smooth and non-negative functions as shown in Lemma~\ref{lem:smooth}. Then, a direct application of Theorem~\ref{thm:general-online-ensemble} with the parameter configurations $\lambda =2L$, $\bar{\eta} = 1/(8L)$ and $\bar{\epsilon} = 1/(8D^2L)$ indicates 
\begin{align}
\sum_{t=1}^T f_t(\x_t) - \sum_{t=1}^T f_t(\u_t)\leq{}&2\sqrt{16LD^2\ln NF_T^X }+ 2 \sqrt{16L(D^2+2DP_T)F_T^X}\notag\\
&\quad + {16D^2L\ln N} + {16L(D^2+2DP_T)}\notag.
\end{align}

According to the definition of $F_T$ and $F_T^X$, the above inequality implies that
\begin{align}
F_T^X - F_T \leq{}& 2\sqrt{16L(D^2\ln N +D^2+2DP_T)F_T^X} + 16L(D^2\ln N + D^2 +2D P_T)\notag\\
\leq{}& 2\sqrt{16L(D^2\ln N + D^2 +2D P_T)(F_T+ 16L(D^2\ln N + D^2 +2D P_T))}\notag\\
&+80L(D^2\ln N + D^2 +2D P_T)\notag\\
={}& \O(\sqrt{(1+P_T+F_T)(1+P_T)}) +\O(1+P_T)\notag\\
={}& \O(\sqrt{(1+P_T+F_T)(1+P_T)})\label{eq:oco-small-loss-conversion},
\end{align}
where the second inequality is by the converting trick in Lemma~\ref{lemma:inquality-shai-cor}. This ends the proof.
\end{proof}

\subsection{Proof of Theorem~\ref{thm:implication-worst-case-bound}}
\label{sec:implication-worst-case-bound}
\begin{proof}
By the universal dynamic regret bound $\mbox{D-Regret}_T(\u_1,\dots,\u_T)\leq \sqrt{A_T(D+P_T)}$ and choosing $\u_t = \x_t^*$, we directly obtain the path-length worst-case dynamic regret bound:
\begin{equation}
\mbox{D-Regret}_T(\x^*_1,\dots,\x^*_T)\leq\sqrt{A_T(D+P^*_T)}\label{eq:worst-casepath-length}.
\end{equation}

In the following, we focus on the function-variation type bound. This is achieved following the argument of~\citet{UAI'20:simple}, we introduce a \emph{virtual} piece-wise stationary comparator sequence that only changes every $\Delta\in[1,T]$ iterations. Specifically, denoting by $\I_m = [(m-1)\Delta+1,\min\{m\Delta,T\}]\subseteq[1,T]$ the $m$-th interval, we define the comparator over the interval $\I_m$ as $\x_{\I_m}^* \in \argmin_{\x\in\X} \sum_{t\in\I_m} f_t(\x)$. There are in total $M = \lceil T/\Delta\rceil$ intervals. Then, we can decompose the worst-case dynamic regret as
\begin{align*}
\mbox{D-Regret}_T(\x^*_1,\dots,\x^*_T) =  \underbrace{\sum_{t=1}^T f_t(\x_t) - \sum_{m=1}^M \sum_{t\in \I_m} f_t(\x_{\I_m}^*)}_{\term{a}} + \underbrace{\sum_{m=1}^M \sum_{t\in \I_m} f_t(\x_{\I_m}^*) - \sum_{t=1}^T f_t(\x_t^*)}_{\term{b}}.
\end{align*}
For term~(a), since the piece-wise stationary comparator sequence only changes $M-1$ times, its path length is at most $D(M-1)$. Thus, the universal dynamic regret~\eqref{eq:implication-UDR} ensures
\begin{align*}
\term{a} \leq \sqrt{A_T \left(D+D(M-1)\right)} \leq \sqrt{DA_T\left(1+\frac{T}{\Delta}\right)} \leq\sqrt{DA_T} + \sqrt{\frac{DTA_T}{\Delta}}.
\end{align*}
Moreover, the argument in~\citet[Proposition 2]{OR'15:dynamic-function-VT} shows that $\term{b} \leq 2\Delta V^f_T$. Combining the upper bounds for term~(a) and term~(b), we obtain
\begin{align*}
\mbox{D-Regret}_T(\x_1^*,\dots,\x_T^*) \leq \sqrt{DA_T} + \sqrt{\frac{DTA_T}{\Delta}} + 2\Delta V^f_T.
\end{align*}
The optimal interval length is $\Delta_* \triangleq (DTA_T)^{1/3}(V^f_T)^{-2/3}$, which will lead to an $\O(\sqrt{A_T} + A_T^{\frac{1}{3}}T^{\frac{1}{3}}(V_T^f)^{\frac{1}{3}})$ worst-case dynamic regret. However, a caveat is that the interval length $\Delta\in[T]$ should be a positive integer, so we use the clipped version $\Delta_\dag \triangleq \min\left\{\left\lceil \Delta_*\right\rceil,T\right\}$. We show that~\eqref{eq:implication-WDR} is achievable with $\Delta_\dag$ by considering the following three cases.
\begin{itemize}
\item \textbf{Case 1} ($1 \leq \Delta_*\leq T$): in such a case, $\Delta_\dag = \lceil \Delta_*\rceil$ and thus $\Delta_*\leq\Delta_\dag\leq 2\Delta_*$,
\begin{equation*}
\mbox{D-Regret}_T(\x_1^*,\dots,\x_T^*) \leq \sqrt{DA_T} + \sqrt{\frac{DTA_T}{\Delta_*}} + 4\Delta_*V_T^f \leq \sqrt{D A_T} + 5D^{\frac{1}{3}}A_T^{\frac{1}{3}}T^{\frac{1}{3}}(V_T^f)^{\frac{1}{3}}. 
\end{equation*}

\item \textbf{Case 2} ($\Delta_*>T$): in such a case, $\Delta_\dag = T$ and $\sqrt{DA_T}\geq TV^f_T$. Then, we have
\begin{equation*}
\mbox{D-Regret}_T(\x_1^*,\dots,\x_T^*) \leq \sqrt{DA_T} + \sqrt{{DA_T}} + 2TV_T^f \leq 3\sqrt{DA_T}.
\end{equation*}

\item \textbf{Case 3} ($\Delta_*\leq 1$): in such a case, $\Delta_\dag = 1$ and $\sqrt{DA_TT}\leq V_T^f$. Since $P^*_T\leq DT$, we have $\sqrt{A_TP^*_T}\leq \sqrt{DA_TT}\leq D^{\frac{1}{3}}A_T^{\frac{1}{3}}T^{\frac{1}{3}}(V_T^f)^{\frac{1}{3}}$, indicating that the path-length bound~\eqref{eq:worst-casepath-length} is tighter than the desired result~\eqref{eq:implication-WDR}.
\end{itemize}
The proof is completed by combining above three cases and the path-length bound~\eqref{eq:worst-casepath-length}.
\end{proof}

\subsection{Proof of Theorem~\ref{thm:lower-bound}}
\label{sec:proof-lower-bound}
\begin{proof}
The theorem is proved by the probabilistic method, following the proof of~\citet[Theorem 5]{NIPS'17:zhang-dynamic-sc-smooth}. For iterations $t = 1,\ldots,T$, we randomly sample a convex and smooth function $f_t: \R^d \mapsto \R$ from the distribution $\mathcal{P}$.

Specifically, we construct the function as $f_t(\x) = \norm{\x - \sigma \boldsymbol{\epsilon}_t}_2^2$, where $\sigma > 0$ and $\boldsymbol{\epsilon}_t \in \R^d$ is a random vector with components sampled independently from the Rademacher distribution, i.e., $\boldsymbol{\epsilon}_t(i) =  1$ or $-1$ with equal probability of $50\%$. We further set the comparator $\u_t = \x_t^* =\argmin_{\x \in \X} f_t(\x) = \sigma \boldsymbol{\epsilon}_t$. Denote by $\x_t$ the decision returned by any deterministic online algorithm $\mathcal{A}$. Then the expected dynamic regret is defined as $\E[\mbox{D-Regret}_T] =\E\left[\sum_{t=1}^{T} f_t(\x_t) - \sum_{t=1}^{T} f_t(\u_t)\right]$ with expectation taken over the randomness of the online function $f_t$. Then, we show that $\E[\mbox{D-Regret}_T] \geq \sqrt{T/2}\cdot\E[\sqrt{(1+F_T)(1+P_T)}]$. 

On one hand, noticing that $\E[\sigma\inner{\x_t}{\epsilon_t}]= 0$ and $\E[\sigma^2\norm{\boldsymbol{\epsilon}_t}_2^2]\geq d\sigma^2$ for any $t\geq 1$, we have
\begin{align*}
    \E[\mbox{D-Regret}_T] = \sum_{t=1}^{T}\E[\norm{\x_t - \sigma\boldsymbol{\epsilon}_t}_2^2] = {} & \sum_{t=1}^{T}\E[\norm{\x_t}_2^2 + 2\sigma\inner{\x_t}{\boldsymbol{\epsilon}_t} + \sigma^2 \norm{\boldsymbol{\epsilon}_t}_2^2] \geq dT\sigma^2,
\end{align*}
On the other hand, let $\bm{\delta}_t(i) = \boldsymbol{\epsilon}_t(i) - \boldsymbol{\epsilon}_{t-1}(i)$. We have 
\begin{align}
    \E[P_T(\u_1,\ldots,\u_T)] = \sigma \sum_{t=2}^{T} \E\left[ \sqrt{\sum_{i=1}^{d} \boldsymbol{\delta}^2_t(i)} \right]\leq \sigma \sum_{t=2}^{T} \sqrt{\sum_{i=1}^{d} \E\left[ \boldsymbol{\delta}^2_t(i) \right]}  \leq \sqrt{2d} T \sigma,\label{eq:lower-bound-1}
\end{align}
where the first inequality is due to the Jensen's inequality and the second inequality is by the fact that $\E\left[ \boldsymbol{\delta}^2_t(i) \right] = 2$ for any $t\in[T]$ and $i\in[d]$.
The above equation leads to
\begin{align*}
\E\bigg[\sqrt{(1+F_T)(1+P_T)}\bigg] = \E[\sqrt{1+P_T}] \leq \sqrt{1+\E[P_T]} \leq (2d)^{\frac{1}{4}}\sqrt{T\sigma}.
\end{align*}
By choosing $\sigma = 1$, we can ensure that $\E[\mbox{D-Regret}]\geq \sqrt{T/2}\cdot\E\big[\sqrt{(1+F_T)(1+P_T)}\big]$. We note that the choice of $\sigma$ might lead to a violation of the bounded domain assumption, which can be easily fixed by rescaling. Then, the probabilistic argument implies that for any algorithm $\mathcal{A}$ there exists a sequence of online functions $f_1,\ldots,f_T$ and comparators $\{\u_t = \x_t^*\}_{t=1}^T$ such that $\mbox{D-Regret}_T \geq \sqrt{T/2}\cdot\sqrt{(1+F_T)(1+P_T)}$. This ends the proof.
\end{proof}
\section{Experiments}
\label{sec:experiemnts}
This section provides empirical studies to validate the effectiveness of our algorithms. 

\paragraph{Settings.} We simulate the online  environments as follows. The player \emph{sequentially} receives the feature of an instance and is then required to make the prediction. We focus on the online regression problem, where at each round an instance $(\psib_t ,y_t)$ is received with $\psib_t \in \Psib \subseteq \R^d$ being the feature and $y_t \in \Y \subseteq \R$ being the corresponding label. At each round, the player first receives the feature $\psib_t$ and is required to make the prediction by $\hat{y}_t = \psib_t^\T \x_t$ based on the learned model $\x_t \in \X \subseteq \R^d$; then, the ground-truth label $y_t \in \R$ is revealed and the player suffers a loss of $\ell(y_t,\hat{y}_t)$, where in the simulation we choose the Huber loss defined as
\begin{align*}
  \ell(y, \hat{y}) = 
  \begin{cases}
    \frac{1}{2}(y - \hat{y})^2, & \text{ for } \abs{y - \hat{y}} \leq \delta,\\
    \delta(\abs{y - \hat{y}} - \frac{1}{2} \delta), & \text{ otherwise}.    
  \end{cases}
\end{align*}
As a result, the online function can be regarded as a composition of the loss function and the data item, that is, $f_t: \X \mapsto \R$ with $f_t(\x) = \ell(y_t,\psib_t^\T \x)$. It can be verified that the online functions are convex, and satisfy the condition of non-negativity and smoothness. 

\paragraph{Datasets.} We compare the performance on both synthetic and real-world datasets. First, the synthetic data are generated as follows: at each round, the feature $\x_t \in \R^d$ is randomly generated from a ball with a radius of $\Gamma$, i.e., $\mathfrak{B} = \{\psib \in \R^d \mid \norm{\psib}_2 \leq \Gamma\}$; the associated label is set as $y_t = \psib_t^\T \x_t^* + \epsilon_t$, where $\epsilon_t$ is the random noise drawn from $[0,0.1]$ and $\x_t^* \in \R^d$ is the underlying model. The underlying model $\x_t^*$ is randomly sampled from a ball with a radius of $D/2$ (recall that $D$ is the diameter of the feasible domain throughout the paper), and it is forced to be stationary within a stage and will be changed every $S$ rounds to simulate non-stationary environments with abrupt changes. In our simulation, we set $\Gamma = 1$, $D = 2$, $d=5$, $T = 50000$, $S = 1000$, and $\delta = 2$. Next, we employ a real-world dataset called Sulfur recovery unit (SRU)~\citep{journals/csur/GamaZBPB14,TKDE'21:DFOP}, which is a regression dataset with slowly evolving distribution changes. There are in total 10,081 data samples representing the records of gas diffusion, where the feature consists of five different chemical and physical indexes and the label is the concentration of $\text{SO}_2$. 

\paragraph{Contenders.} We compare the performance of the following algorithms: (i) OGD~\citep{ICML'03:zinkvich}, online gradient descent, which is an OCO algorithm designed for static regret minimization; (ii) Ader~\citep{NIPS'18:Zhang-Ader}, an OCO algorithm designed for optimizing dynamic regret yet with only problem-independent guarantee; (iii) \textsf{Sword}, the algorithm proposed in Section~\ref{sec:solution1}, which achieves problem-dependent dynamic regret guarantees requiring multiple gradients per iteration; and (iv) \textsf{Sword++}, the algorithm proposed in Section~\ref{sec:solution2}, which achieves the same dynamic regret with only one gradient query per round. The implementations of all algorithms are based on \textsf{PyNOL} package~\citep{PyNOL}.

\begin{figure}[!t]
    \centering
    \subfigure[synthetic data]{ \label{fig:para-p}
      \includegraphics[clip, trim=0.6cm 0.3cm 1.5cm 1.2cm,height=0.34\textwidth]{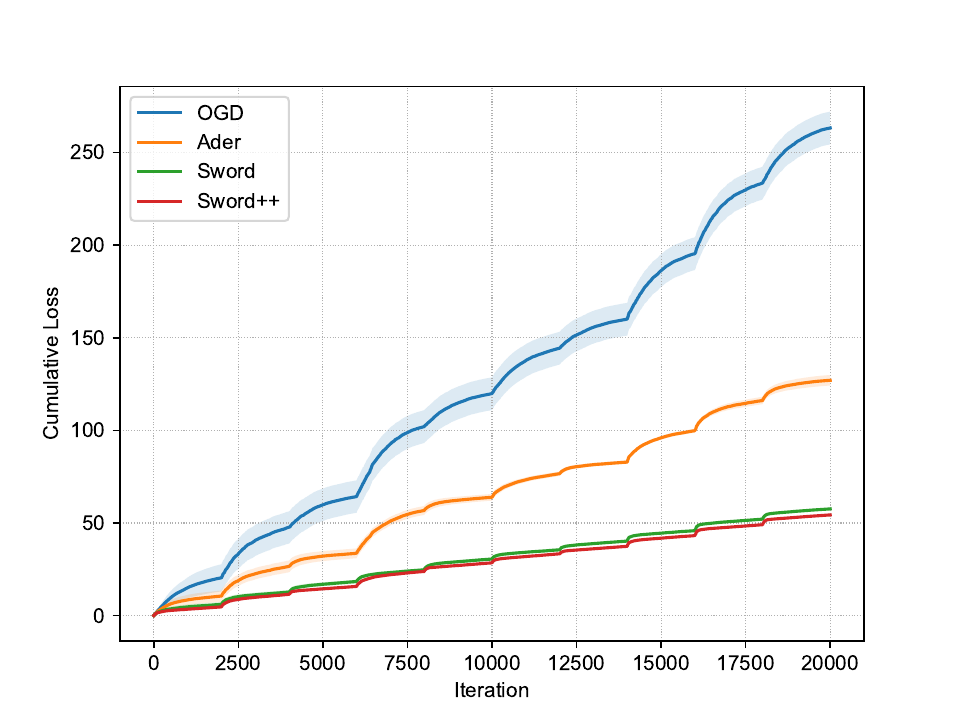}} \hspace{3mm}
    \subfigure[real-world data]{ \label{fig:para-p}
      \includegraphics[clip, trim=0.6cm 0.3cm 1.5cm 1.2cm,height=0.34\textwidth]{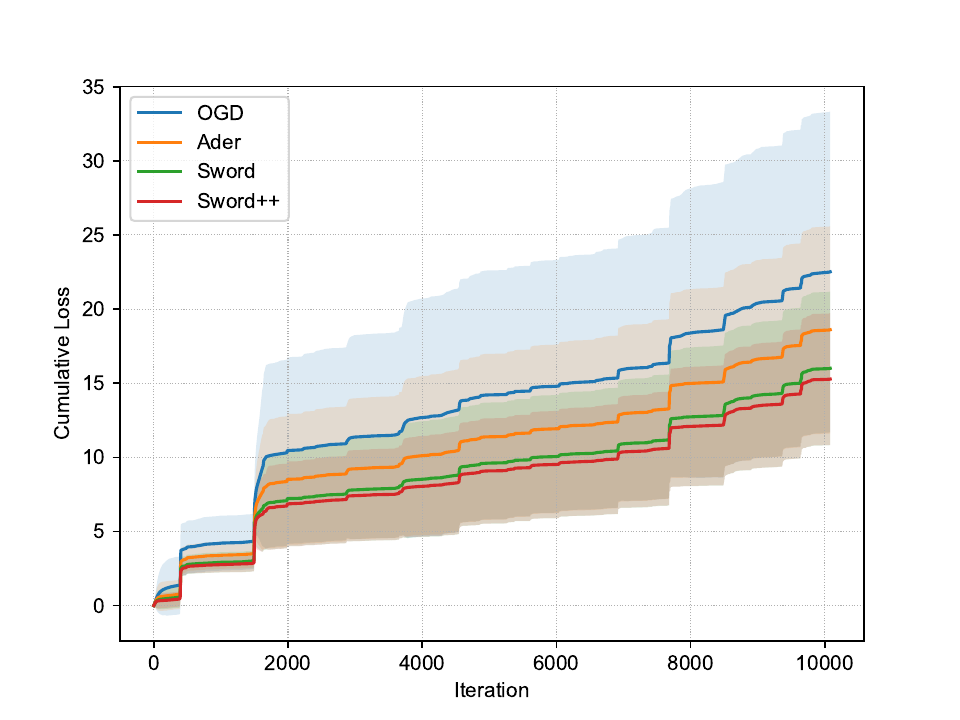}}
    \vspace{-4mm}
    \caption{Performance comparisons of all algorithms in terms of cumulative loss.}
  \label{fig:comparison-loss}\vspace{-7mm}
\end{figure}

\begin{figure}[!t]
    \centering
    \subfigure[synthetic data]{ \label{fig:para-p}
      \includegraphics[clip, trim=0.6cm 0.3cm 1.5cm 1.2cm,height=0.34\textwidth]{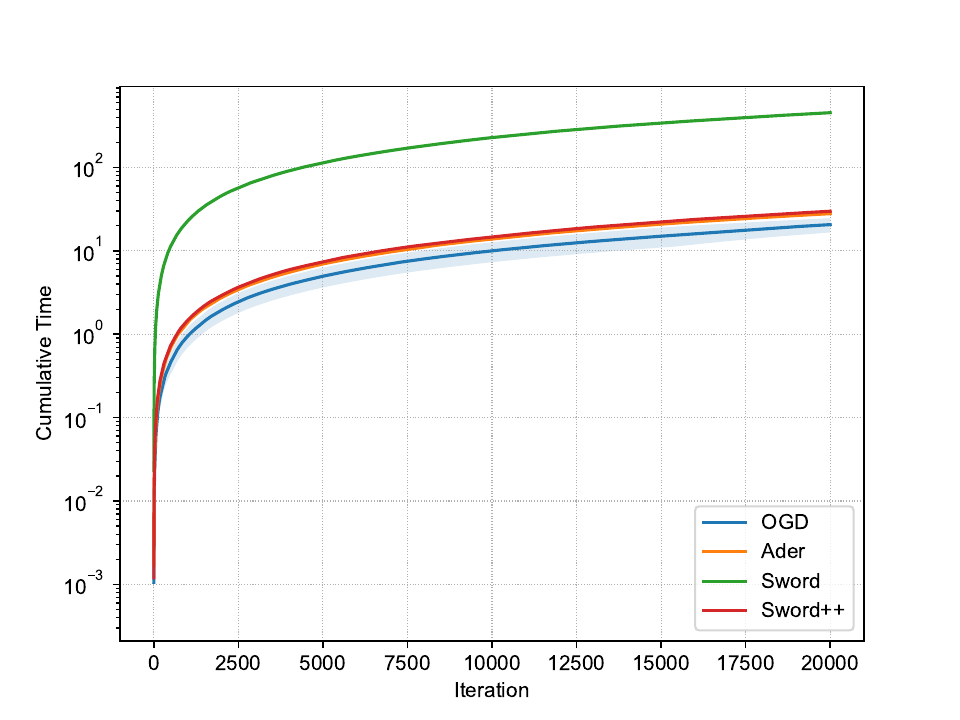}}\hspace{3mm}
    \subfigure[real-world data]{ \label{fig:para-p}
      \includegraphics[clip, trim=0.6cm 0.3cm 1.5cm 1.2cm,height=0.34\textwidth]{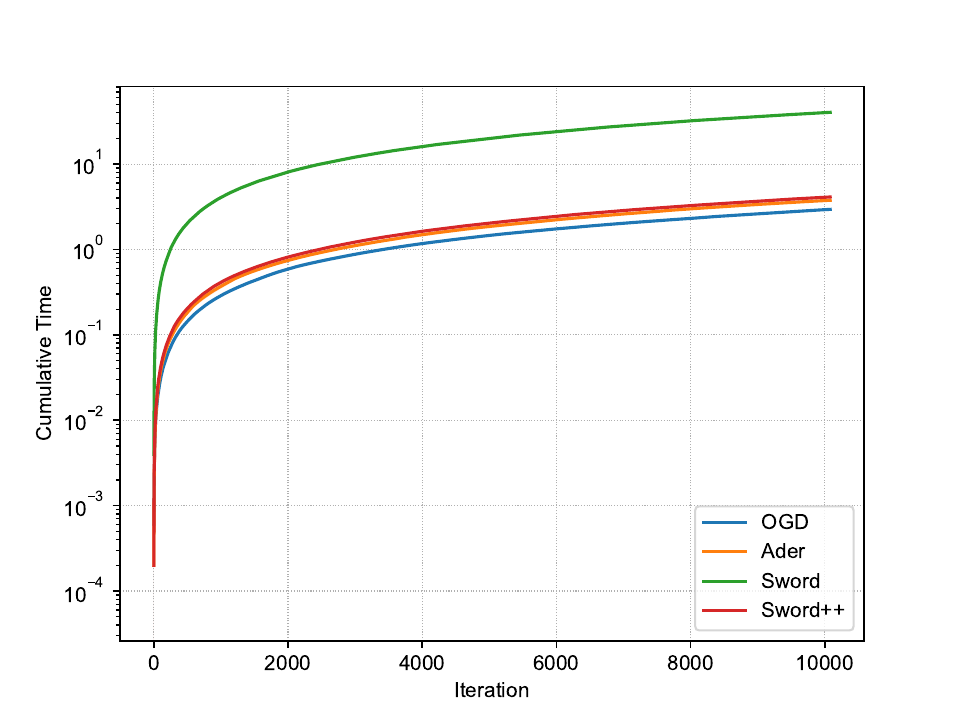}}
    \vspace{-4mm}
    \caption{Performance comparisons of all algorithms in terms of running time (in seconds).}
  \label{fig:comparison-time}\vspace{-2mm}
\end{figure}

\paragraph{Results.} We repeat the experiments five times and report the mean and the standard deviation in Figure~\ref{fig:comparison-loss} and Figure~\ref{fig:comparison-time}. In Figure~\ref{fig:comparison-loss}, we examine the performance in terms of cumulative loss. First, we can observe that OGD incurs a large cumulative loss over the horizon and is not able to effectively learn from the non-stationary environments. By contrast, both Ader and our approach (Sword, Sword++) achieve a satisfactory performance in the presence of distribution changes. Moreover, Sword and Sword++ exploit the adaptivity of the problem instance and thus achieve more encouraging empirical behavior than Ader, which demonstrates the empirical effectiveness. Figure~\ref{fig:comparison-time} reports the running time comparison, where the $y$-axis uses a logarithmic scale for a better presentation. We can observe that OGD is the most computationally efficient; besides, Ader and Sword++ are also comparable. By contrast, Sword requires significantly more running time. The result accords to our theory well, in that the gradient computation is the most time-consuming in our simulations. Theoretically, both Sword++ and Ader (with linearized surrogate loss) only require one gradient query per iteration, which shares the same gradient query complexity with OGD. On the contrary, Sword needs to query $N = \O(\log T)$ gradients at each round and is thus much more computationally inefficient. In summary, the empirical results validate the advantages of our proposed algorithms, notably showing that Sword++ behaves well and is computationally light.
\section{Conclusion}
\label{sec:conclusion}
In this paper, we exploit the easiness of problem instances to enhance the universal dynamic regret. We propose two novel online ensemble algorithms, \textsf{Sword} and \textsf{Sword++}, for convex and smooth online learning. Both algorithms achieve  a best-of-both-worlds dynamic regret of order $\O(\sqrt{(1+P_T + \min\{V_T, F_T\})(1+P_T)})$, where $V_T$ measures the gradient variation and $F_T$ is the cumulative loss of comparators. These quantities are at most $\O(T)$ yet can be very small when the problem is easy, hence reflecting the difficulty of problem instance. Consequently, our bounds can outperform the $\O(\sqrt{T(1+P_T)})$ minimax dynamic regret~\citep{NIPS'18:Zhang-Ader} by exploiting smoothness. Our results are accomplished by several crucial technical ingredients. We adopt optimistic online mirror descent as a unified building block for both base and meta algorithms, and carefully exploit the negative terms in the regret analysis. Moreover, in the design of Sword++, we introduce the framework of \emph{collaborative online ensemble}. This framework emphasizes the importance of jointly using negative terms in the regret analysis and integrating additional correction terms in the algorithmic design, which facilitates effective collaboration between the meta and base layers. By incorporating these elements, we can finally achieve favorable problem-dependent dynamic regret guarantees under the one-gradient feedback model.

All of attained dynamic regret bounds are universal in the sense that they hold against \emph{any} feasible comparator sequence, making the algorithms adaptive to non-stationary environments. An important future work is to investigate the optimality of our attained problem-dependent dynamic regret bounds. We now only have very preliminary understandings for small-loss dynamic regret (see the lower bound in Theorem~\ref{thm:lower-bound}  and conjectures regarding the gradient-variation dynamic regret in Remark~\ref{remark:discussion-VT-lower-bound}). However, a complete understanding requires refined problem-dependent lower bounds.

\section*{Acknowledgment}
This research was supported by NSFC (61921006, 62361146852), JiangsuSF (BK20220776), National Postdoctoral Program for Innovative Talent, and China Postdoctoral Science Foundation (2023M731597). Peng Zhao was supported in part by the Xiaomi Foundation. The authors would like to thank Mengxiao Zhang and Ashok Cutkosky for helpful discussions. We are also grateful for anonymous reviewers and the action editor for their invaluable comments, in particular, we sincerely thank Reviewer~\#2 of JMLR for carefully reviewing the paper and providing many constructive suggestions.

\appendix

\section{Proofs of Lemma~\ref{lemma:OEGD-variation} and Lemma~\ref{lemma:OptimisticHedge}}
\label{sec:appendix-proof-OMD}
The versatility of \textsc{Optimistic OMD} makes it very general to derive many existing results in a unified view. We elucidate two implications of Theorem~\ref{thm:dynamic-regret-OMD-generic}, which serve as proofs of Lemma~\ref{lemma:OEGD-variation} (dynamic regret of OEGD) and Lemma~\ref{lemma:OptimisticHedge} (static regret of Optimistic Hedge).\\

\begin{proof}[of Lemma~\ref{lemma:OEGD-variation}]
We first show a general result for Optimistic OMD with  an arbitrary optimistic vector $M_t\in\R^d$ and the regularizer $\psi(\x) = \frac{1}{2}\norm{\x}_2^2$, and then prove Lemma~\ref{lemma:OEGD-variation} by simply substituting $M_t = \nabla f_{t-1}(\x_{t-1})$ into it.

It is well-known that $\psi(\x) = \frac{1}{2}\norm{\x}_2^2$ is 1-strongly convex with respect to the Euclidean norm $\|\cdot\|_2$ and $\Div{\x}{\y} = \frac{1}{2}\norm{\x - \y}_2^2$. Thus, Theorem~\ref{thm:dynamic-regret-OMD-generic} implies
\begin{align*}
  \sum_{t=1}^{T} f_t(\x_t) - \sum_{t=1}^{T} f_t(\u_t) \leq \eta \sum_{t=1}^{T} \norm{\nabla f_t(\x_t) - M_t}_2^2 {} & +  \frac{1}{2\eta}\sum_{t=1}^{T} \Big( \norm{\u_t - \xh_t}_2^2 - \norm{\u_t - \xh_{t+1}}_2^2\Big) \\
{} & -  \frac{1}{2\eta} \sum_{t=1}^{T} \Big( \norm{\xh_{t+1} - \x_t}_2^2 + \norm{\xh_t - \x_t}_2^2\Big).
\end{align*}
Notice that the second term can be upper bounded as follows.
\begin{align*}
  {}& \sum_{t=1}^{T} \Big( \norm{\u_t - \xh_t}_2^2 - \norm{\u_t - \xh_{t+1}}_2^2\Big) \leq \norm{\u_1 - \xh_1}_2^2 + \sum_{t=2}^{T} \Big( \norm{\u_t - \xh_t}_2^2 - \norm{\u_{t-1} - \xh_t}_2^2\Big) \\
  \leq {} & \norm{\u_1 - \xh_1}_2^2 + \sum_{t=2}^{T} \norm{\u_t - \u_{t-1}}_2 \norm{\u_t - \xh_t + \u_{t-1} - \xh_t}_2 \leq D^2 + 2D\sum_{t=2}^{T} \norm{\u_t - \u_{t-1}}_2.
\end{align*}
We further evaluate the last term using the inequality of $a^2 + b^2 \geq (a+b)^2/2$ and obtain
\begin{align*}
\sum_{t=1}^{T} \Big( \norm{\xh_{t+1} - \x_t}_2^2 + \norm{\xh_t - \x_t}_2^2\Big) \geq \sum_{t=2}^{T} \Big( \norm{\xh_{t} - \x_{t-1}}_2^2 + \norm{\xh_t - \x_t}_2^2\Big) \geq \frac{1}{2}\sum_{t=2}^{T} \norm{\x_t - \x_{t-1}}_2^2.
\end{align*}
Hence, combining all the three inequalities, we get the following result:
\begin{align*}
\sum_{t=1}^T f_t(\x_t) - \sum_{t=1}^T f_t(\u_t)\leq \eta \sum_{t=1}^T\norm{\nabla f_t(\x_t) - M_t}_2^2  + \frac{1}{\eta} (D^2 + 2DP_T) - \frac{1}{4\eta}\sum_{t=2}^T\norm{\x_t-\x_{t-1}}_2^2.
\end{align*}

When choosing the optimism as the last-round gradient $M_{t} = \nabla f_{t-1}(\x_{t-1})$, the adaptivity term $\sum_{t=1}^{T} \norm{\nabla f_t(\x_t) - M_t}_2^2$ can be upper bounded in the following way:
\begin{align*}
  \sum_{t=1}^{T} \norm{\nabla f_t(\x_t) - M_t}_2^2 \leq {} &  G^2 + 2 \sum_{t=2}^{T} \left( \norm{\nabla f_t(\x_t) - \nabla f_{t-1}(\x_{t})}_2^2 + \norm{\nabla f_{t-1}(\x_t) - \nabla f_{t-1}(\x_{t-1})}_2^2 \right)\\
 \leq {} &  G^2 + 2 \sum_{t=2}^{T} \sup_{\x \in \X}\norm{\nabla f_t(\x) - \nabla f_{t-1}(\x)}_2^2 + 2 L^2 \sum_{t=2}^{T} \norm{\x_t -\x_{t-1}}_2^2,
\end{align*}
where the last step exploits $L$-smoothness of online functions. Therefore,
\begin{align*}
  \sum_{t=1}^{T} f_t(\x_t) - \sum_{t=1}^{T} f_t(\u_t) \leq {} & \eta (G^2 + 2V_T) + \frac{1}{2\eta}(D^2 + 2DP_T) + \left(2\eta L^2 - \frac{1}{4\eta}\right)\sum_{t=2}^{T} \norm{\x_{t} - \x_{t-1}}_2^2\\
  \leq {} & \eta (G^2 + 2V_T) + \frac{1}{2\eta}(D^2 + 2DP_T),
\end{align*}
where $\eta \leq 1/(4L)$ ensures the last term to be non-positive. This ends the proof.
\end{proof}

Next, by choosing the regularizer $\Rcal(\p) = \sum_{i=1}^{N} p_i \ln p_i$, the online function $f_t(\p) = \inner{\p}{\ellb_t}$ and optimism $M_t = \m_t$, Optimistic OMD recovers Optimistic Hedge~\citep{conf/colt/RakhlinS13}. Here, with a slight abuse of notations, we now use $\p \in \Delta_N$ to denote the variable in prediction with expert advice. Theorem~\ref{thm:dynamic-regret-OMD-generic} implies the static regret for Optimistic Hedge algorithm by choosing comparators as a fixed one in the simplex.
~\\
\begin{proof}[of Lemma~\ref{lemma:OptimisticHedge}] It is well-known that the negative entropy $\psi(\p) = \sum_{i=1}^{N} p_i \ln p_i$ is 1-strongly convex with respect to $\|\cdot\|_1$ and $\Div{\p}{\q} = \sum_{i=1}^{N} p_i \ln(p_i/q_i)$. By choosing comparators as $\mathbf{e}_i$, Theorem~\ref{thm:dynamic-regret-OMD-generic} indicates that the regret $\sum_{t=1}^{T} \left(\inner{\p_t}{\ellb_t} - \ell_{t,i}\right)$ is bounded by
\begin{align*}
  \epsilon \sum_{t=1}^{T} \norm{\ellb_t - \m_t}_\infty^2 +  \frac{1}{\epsilon}\sum_{t=1}^{T} \Big( \Div{\e_i}{\hat{\p}_t} - \Div{\e_i}{\hat{\p}_{t+1}}\Big) -  \frac{1}{\epsilon}\sum_{t=1}^{T} \Big( \Div{\hat{\p}_{t+1}}{\p_t} + \Div{\p_t}{\hat{\p}_{t}} \Big).
\end{align*}
The second term on the right hand side exhibits a telescoping structure, and thus
\[
  \frac{1}{\epsilon}\sum_{t=1}^{T} \Big( \Div{\e_i}{\hat{\p}_t} - \Div{\e_i}{\hat{\p}_{t+1}}\Big) \leq \frac{1}{\epsilon} \Div{\e_i}{\hat{\p}_1} = \frac{1}{\epsilon} \ln(1/p_{1,i}).
\]
By Pinsker's inequality $\Div{\p}{\q} = \mbox{KL}(\p,\q) \geq \frac{1}{2} \norm{\p - \q}_1^2$, we have 
\[
  \sum_{t=1}^{T} \Big( \Div{\hat{\p}_{t+1}}{\p_t} + \Div{\p_t}{\hat{\p}_{t}} \Big) \geq \frac{1}{2} \sum_{t=1}^{T} \left(\norm{\hat{\p}_{t+1} - \p_t}_1^2 + \norm{\p_t - \hat{\p}_{t}}_1^2\right) \geq \frac{1}{4} \sum_{t=2}^{T} \norm{\p_t - \p_{t-1}}_1^2,
\]
where the last inequality is got by regrouping the sum and applying triangle inequality. Combining all above inequalities ends the proof.
\end{proof}

The negative term in the regret bound~\eqref{eq:regret-optimistic-Hedge} of Lemma~\ref{lemma:OptimisticHedge} is very essential, which is quite useful in a variety of problems requiring adaptive bounds. Our analysis is based on the unified view of Optimistic OMD (Theorem~\ref{thm:dynamic-regret-OMD-generic}), and is much simpler than the original proof of~\citet{NIPS'15:fast-rate-game} using the mathematical induction  from the lens of FTRL.
\section{Adaptive Learning Rate Version}
\label{appendix:adaptive-LR}
In the main text, our proposed algorithms, Sword and Sword++, deliberately employ a \emph{fixed} learning rate for the meta-algorithm.This choice is made to simplify the presentation and the regret analysis, thereby helping reader to better understand the key idea of facilitating collaborations in meta and base levels, as highlighted in Remark~\ref{remark:collaboration-proof} of the analysis.

Nevertheless, the meta-algorithm learning rate tuning requires the knowledge of gradient variation $V_T = \sum_{t=2}^T \sup_{\x \in \X} \norm{\nabla f_t(\x) - \nabla f_{t-1}(\x)}_2^2$ (for Sword) or its variant $\bar{V}_T = \sum_{t=2}^T \norm{\nabla f_t(\x_t) - \nabla f_{t-1}(\x_{t-1})}_2^2$ (for Sword++), which is not desired. This section demonstrates an adaptive version using the self-confident tuning framework~\citep{JCSS'02:Auer-self-confident} such that the meta-algorithm does not require such information ahead of time.\footnote{For simplicity, we only present the adaptive version for Sword++, and the one for Sword can be similarly obtained (which is actually simpler). Moreover, an important note is that our adaptive version also only requires one gradient per iteration, hence still feasible for the one-gradient feedback model.} 

We first extend the collaborative online ensemble framework in Section~\ref{sec:framework-collaborative-OE} to an adaptive version, and subsequently use it to prove the gradient-variation and small-loss dynamic regret bounds for the adaptive version of Sword++.

\subsection{Adaptive Collaborative Online Ensemble}
In this part, we provide the adaptive learning rate version of the unified framework presented in Section~\ref{sec:framework-collaborative-OE}. Comparing with the fixed learning rate version, the only difference is that we run the optimistic Hedge with a time-varying learning rate for the meta-algorithm, 
\begin{equation}
    \label{eq:variation-Hedge-meta-adaptive}
    p_{t+1,i} \propto \exp\left(-\epsilon_t \Big(\sum_{s=1}^{t} \ell_{s,i} + m_{t+1,i} \Big)\right),
\end{equation}
where the loss vector $\ellb_t$ and $\m_t$ share the same configurations as~\eqref{eq:general-framework-meta}. For any $i$-th base-algorithm, we use the same update rule as the fixed learning rate version  
\begin{align}
\label{eq:general-framework-base-adaptive}
 \x_{t,i} = \Pi_{\X}\left[\xh_{t,i} - \eta_{i} M_t\right],~~ \xh_{t+1,i} = \Pi_{\X}\left[\xh_{t,i}-\eta_{i}\nabla f_{t}(\x_{t})\right],
\end{align}
Then, we can generate the prediction by $\x_t = \sum_{i=1}^N p_{t,i}\x_{t,i}$ and have the following guarantee. 
\begin{myThm}
\label{thm:general-online-ensemble-adaptive}
Under the same assumptions and parameter configurations as Theorem~\ref{thm:general-online-ensemble} and setting the learning rate of the meta-algorithm as 
\begin{equation}
  \label{eq:lr-meta-solution2-adaptve}
  \epsilon_t = \min\left\{ \bar{\epsilon}, \sqrt{\frac{\ln N}{D^2 \sum_{s=1}^t\norm{\nabla f_s(\x_s) - M_s}_2^2}} \right\},
\end{equation}
we have the following dynamic regret bound for the decisions specified by~\eqref{eq:variation-Hedge-meta-adaptive} and~\eqref{eq:general-framework-base-adaptive} against any comparators $\u_1,\ldots,\u_T \in \X$,
\begin{align}
\sum_{t=1}^T f_t(\x_t) - \sum_{t=1}^T f_t(\u_t) {}& \leq 4\sqrt{D^2\ln NA_T }+ 2 \sqrt{(D^2+2DP_T)A_T}+ \frac{\ln N}{\bar{\epsilon}} + 2\bar{\epsilon} D^2\tilde{G}^2\notag\\
& + \frac{2(D^2+2DP_T)}{\bar{\eta}}  + \left(\lambda-\frac{1}{4\bar{\eta}}\right) S_{x,i} - \frac{1}{4\bar{\epsilon}} S_p - \lambda S_{\mathrm{mix}}. \label{eq:result-adaptiev-collaboration}
\end{align}
In above, $\tilde{G} = \max_{t\in[T]} \norm{\nabla f_t(\x)-M_t}_2$, $A_T = \sum_{t=1}^{T} \norm{\nabla f_t(\x_t)-M_t}_2^2$, $P_T = \sum_{t=2}^{T} \norm{\u_{t-1} - \u_{t}}_2$, $S_{x,i} = \sum_{t=2}^T\norm{\x_{t,i}-\x_{t-1,i}}_2^2$, $S_p = \sum_{t=2}^T\norm{\p_t-\p_{t-1}}_1^2$, and $S_{\mathrm{mix}}= \sum_{t=2}^T\sum_{i=1}^Np_{t,i}\norm{\x_{t,i}-\x_{t-1,i}}_2^2$ are base stability, meta stability, and mixed stability.
\end{myThm}

\begin{myRemark}[Optimistic Hedge with Time-varying Learning Rates]
\label{remark:optimisticHedge-LR}
We remark that, in the fixed learning rate case, one can show that the Optimistic Hedge~\eqref{eq:variation-Hedge-meta} is identical to Optimistic OMD with the negative-entropy regularizer. However, the adaptive learning rate version~\eqref{eq:variation-Hedge-meta-adaptive} can only be interpreted as an FTRL algorithm. Thus, it is hard to directly apply Theorem~\ref{thm:dynamic-regret-OMD-generic} to obtain the meta-regret. We choose a FTRL-type meta-algorithm instead of an OMD-type algorithm, in that OMD with  time-varying learning rates would suffer linear regret in the worst case when using the negative-entropy regularizer. While this can be remedied by the dual stabilization~\citep{ICML'20:OMD-Stabilization}, we just use FTRL for simplicity.  \endenv
\end{myRemark}

\begin{proof}[of Theorem~\ref{thm:general-online-ensemble-adaptive}]
The proof is almost the same to that of Theorem~\ref{thm:general-online-ensemble}. The main difference is that we use a counterpart of Lemma~\ref{lemma:OptimisticHedge} to bound the meta-regret for the adaptive learning rate version~\eqref{eq:variation-Hedge-meta-adaptive}. Specifically, since~\eqref{eq:variation-Hedge-meta-adaptive} is identical to Optimistic FTRL $\p_{t+1} = \argmin_{\p\in\Delta_N}  \inner{\p}{\sum_{s=1}^t\ellb_s + \m_{t+1}} + \psi_{t+1}(\p)$ with regularizer $\psi_{t+1}(\p) = \frac{1}{\epsilon_t}(\sum_{i=1}^N p_i\ln p_i + \ln N)$,\footnote{Here, we add an additional constant $\ln N$ in the regularizer, which will not effect the solution of the optimization problem and meanwhile make the regret analysis more convenient.} a direct application of~\citet[Theorem 7.35]{arxiv'19:online-learning-modern-intro} leads to the following lemma.
\begin{myLemma}[{Theorem 7.35 of~\citet{arxiv'19:online-learning-modern-intro}}]
\label{lemma:OptHedge-adatLR}
The regret of Optimistic Hedge with a time-varying learning rate $\epsilon_t > 0 $ (see the update specified in~\eqref{eq:variation-Hedge-meta-adaptive}) to any expert $i \in [N]$ satisfies
\begin{align*}
\sum_{t=1}^{T} \inner{\p_t}{\ellb_t} - \sum_{t=1}^{T} \ell_{t,i}\leq{}&\max_{\p\in\Delta_N}\psi_{T+1}(\p) + \sum_{t=1}^T \inner{\ellb_t-\m_t}{\p_{t}-\p_{t+1}} - \sum_{t=1}^T\frac{1}{2\epsilon_{t-1}}\norm{\p_t-\p_{t+1}}_1^2,
\end{align*}
where $\psi_{t+1}(\p) = \frac{1}{\epsilon_{t}}(\sum_{i=1}^N p_i\ln p_i + \ln N)$ is the regularizer.
\end{myLemma}
Then, based on this Lemma~\ref{lemma:OptHedge-adatLR}, we can bound the regret of the Optimistic Hedge by
\begin{align*}
&\sum_{t=1}^{T} \inner{\p_t}{\ellb_t} - \sum_{t=1}^{T} \ell_{t,i}\\
\leq{}&\frac{\ln N}{\epsilon_T} + \sum_{t=1}^T\epsilon_{t-1}\norm{\ellb_t-\m_t}^2_\infty + \sum_{t=1}^T \frac{1}{4\epsilon_{t-1}}\norm{\p_t-\p_{t+1}}_1^2- \sum_{t=1}^{T} \frac{1}{2\epsilon_{t-1}}\norm{\p_{t} - \p_{t+1}}_1^2\\
\leq{}&\frac{\ln N}{\epsilon_T} + D^2\sum_{t=1}^{T}\epsilon_{t-1} \norm{\nabla f_t(\x_t)-M_t}_{2}^2 - \sum_{t=2}^{T} \frac{1}{4\epsilon_{t-1}}\norm{\p_{t} - \p_{t-1}}_1^2\\
\leq{}& 2\bar{\epsilon}D^2G^2 + \frac{\ln N}{\bar{\epsilon}} + 4\sqrt{D^2\ln N\sum_{t=1}^T\norm{\nabla f_t(\x_t)-M_t}_{2}^2} - \sum_{t=2}^{T} \frac{1}{4\epsilon_{t-1}}\norm{\p_{t} - \p_{t-1}}_1^2,
\end{align*}
where the second inequality is due to the H\"{o}lder's inequality $\inner{\ellb_t-\m_t}{\p_{t}-\p_{t+1}}\leq \norm{\ellb_t-\m_t}_\infty\norm{\p_{t}-\p_{t+1}}_1$ and the fact that $ab \leq {\epsilon_{t-1}a^2}+\frac{b^2}{4\epsilon_{t-1}}$ holds for any $a,b,\epsilon_{t-1}>0$. The third inequality is by definitions of $\ellb_t$ and $\m_{t}$. The last inequality is a consequence of the inequality $\ln N/\epsilon_T\leq \ln N/\bar{\epsilon} + \sqrt{D^2\ln N\sum_{t=1}^T \norm{\nabla f_t(\x) -M_t}_2^2}$ by learning rate configuration~\eqref{eq:lr-meta-solution2-adaptve} and Lemma~\ref{lem:self-confident-adaptiveLR}, which provides a clipped version of the self-confident tuning.

By the same meta-regret analysis in the proof of Theorem~\ref{thm:general-online-ensemble}, see~\eqref{eq:proof-thmfix-meta}, we have  
\begin{align}
\meta\leq{}& 4\sqrt{D^2 (\ln N)A_T } + \frac{\ln N}{\bar{\epsilon}} + 2\bar{\epsilon}D^2G^2- \frac{1}{4\bar{\epsilon}}\sum_{t=2}^{T}  \norm{\p_{t}- \p_{t-1}}_1^2  \notag\\
 {}& - \lambda\sum_{t=1}^T\sum_{i=1}^Np_{t,i}\norm{\x_{t,i}-\x_{t-1,i}}_2^2 + \lambda\sum_{t=1}^T\norm{\x_{t,i}-\x_{t-1,i}}_2^2\label{eq:proof-thmadap-meta},
\end{align}
which holds from any base-algorithm $i\in[N]$. Following the same arguments in the proof of Theorem~\ref{thm:general-online-ensemble}, we can identify an optimal base-algorithm indexed by $i^*\in[N]$, whose base-regret is bounded as $\base\leq 2 \sqrt{(D^2+2DP_T)A_T} +  \frac{2(D^2+2DP_T)}{\bar{\eta}}$. Combining the meta-regret and the base-regret of the $i^*$-th base-learner yields the result in~\eqref{eq:result-adaptiev-collaboration}.
\end{proof}

\subsection{Adaptive Version of Sword++}
We show that the adaptive learning rate version of the framework~\eqref{eq:variation-Hedge-meta-adaptive} and~\eqref{eq:general-framework-base-adaptive} with $M_t = \nabla f_{t-1}(\x_{t-1})$ for $t \geq 2$ ($M_1 = \mathbf{0}$) achieves the same problem-dependent dynamic regret bound (up to constants) as that in Theorem~\ref{thm:small-loss-one-gradient}. 
\begin{myThm}
\label{thm:best-of-both-world-adaptive-LR}
Under the same assumptions and parameter configurations as Theorem~\ref{thm:small-loss-one-gradient} and set the learning rate of the meta-algorithm as 
\begin{equation*}
  \epsilon_t = \min\left\{ \bar{\epsilon}, \sqrt{\frac{\ln N}{D^2 \sum_{s=1}^t\norm{\nabla f_s(\x_s) - M_s}_2^2}} \right\}
\end{equation*}
with $M_t = \nabla f_{t-1}(\x_{t-1})$ for $t \geq 2$ ($M_1 = \mathbf{0}$). Then, decisions specified by~\eqref{eq:variation-Hedge-meta-adaptive} and~\eqref{eq:general-framework-base-adaptive} satisfy that for any comparators $\u_1,\ldots,\u_T \in \X$,
\begin{align}
\sum_{t=1}^T f_t(\x_t) - \sum_{t=1}^T f_t(\u_t)\leq \O\left(\sqrt{(1+P_T+ \min\{V_T, F_T\})(1+P_T)}\right).
\end{align}
\end{myThm}
\begin{proof}
Under the parameter configurations $\lambda = 2L$, $\bar{\eta} = 1/(8L)$ and $\bar{\epsilon} = 1/ (8D^2L)$, the dynamic regret bound of the unified algorithm with the adaptive learning rate (c.f. Theorem~\ref{thm:general-online-ensemble-adaptive}) is almost the same as that of the fixed learning rate (c.f. Theorem~\ref{thm:general-online-ensemble}) up to constant factors. Thus, the same arguments in the proof of Theorem~\ref{thm:small-loss-one-gradient} lead to the desired result.
\end{proof}
\section{Technical Lemmas}
\label{sec:appendix-technical-lemmas}
This section collects several useful technical lemmas frequently used in the proofs. The first one is the Bregman proximal inequality, which is crucial in the analysis of first-order optimization methods based on Bregman divergence.
\begin{myLemma}[{Bregman proximal inequality~\citep[Lemma 3.2]{OPT'93:Bregman}}]
\label{lemma:bregman-divergence}
Let $\X$ be a convex set in a Banach space. Let $f: \X \mapsto \R$ be a closed proper convex function on $\X$. Given a convex regularizer $\Rcal:\X \mapsto \R$, we denote its induced Bregman divergence by $\D_\Rcal(\cdot,\cdot)$. Then, any update of the form
\[
  \x_k = \argmin_{\x \in \X} \{ f(\x) + \D_\Rcal(\x,\x_{k-1})\}
\]
satisfies the following inequality for any $\u \in \X$,
\begin{equation}
  f(\x_k) - f(\u) \leq \D_\Rcal(\u, \x_{k-1}) - \D_\Rcal(\u, \x_{k}) - \D_\Rcal(\x_k, \x_{k-1}).
\end{equation}
\end{myLemma}

The second one is the stability lemma, which is very useful in analyzing online algorithms based on FTRL or OMD frameworks. 
\begin{myLemma}[{stability lemma~\citep[Proposition 7]{COLT'12:variation-Yang}}]
\label{lemma:stability-OMD}
Consider the following two updates: (i) $\x_*=\argmin_{\x\in\X}\ \inner{\a}{\x}+\D_\Rcal(\x,\c)$, and (ii) $\x'_* = \argmin_{\x\in\X}\ \inner{\a'}{\x}+\D_\Rcal(\x,\c)$. When the regularizer $\Rcal:\X\mapsto\R$ is a 1-strongly convex function with respect to the norm $\| \cdot \|$, we have $\norm{\x_*-\x'_*} \leq \norm{\a-\a'}_*$.
\end{myLemma}

The self-bounding property of smooth functions is crucial and frequently used in proving small-loss bounds for convex and smooth functions.
\begin{myLemma}[{self-bounding property~\citep[Lemma 3.1]{NIPS'10:smooth}}]
\label{lem:smooth}
For an $L$-smooth and non-negative function $f: \X \mapsto \R_+$, we have $\norm{\nabla f(\x)}_2 \leq \sqrt{4 L f(\x)}, \ \forall \x \in \X$.
\end{myLemma}

Finally, we present several useful inequalities.
\begin{myLemma}
\label{lemma:inequality}
Let $a, b >0$ and $x_0 > 0$ be three positive values. Suppose that $L \leq ax + \frac{b}{x}$ holds for any $x \in (0,x_0]$. Then, by taking $x^* = \min\{ \sqrt{b/a}, x_0 \}$, we have $L \leq 2\sqrt{ab} + \frac{2b}{x_0}$.
\end{myLemma}
\begin{proof}
Suppose $\sqrt{b/a} \leq x_0$, then $x^* = \sqrt{b/a}$ and we have $L \leq ax^* + \frac{b}{x^*} = 2\sqrt{ab}$. 
Otherwise, $x^* = x_0$ and we have $L \leq a x^* + \frac{b}{x^*} = a x_0 + \frac{b}{x_0}$. Notice that in latter case $x_0 \leq \sqrt{b/a}$ holds, which implies $a x_0 \leq \frac{b}{x_0}$ and hence $ax_0 + \frac{b}{x_0} \leq \frac{2b}{x_0}$. Combining two cases ends the proof.
\end{proof}

\begin{myLemma}[Lemma 19 of~\citet{thesis:shai2007}]
\label{lemma:inquality-shai}
For any $x,y,a \in \R_+$ satisfying $x-y \leq \sqrt{ax}$, it holds that $x-y \leq a+\sqrt{ay}$.
\end{myLemma}

\begin{myLemma}
\label{lemma:inquality-shai-cor}
For any $x,y,a, b\in \R_+$ satisfying $x-y \leq \sqrt{ax} + b$, it holds that $x-y \leq a + b + \sqrt{ay + ab}$.
\end{myLemma}

\begin{myLemma}[{Lemma 3.5 of~\citet{JCSS'02:Auer-self-confident}}]
\label{lem:self-confident}
  Let $a_1, a_2, \ldots, a_T$ and $\delta$ be non-negative real numbers. Then, it holds that $\sum_{t=1}^{T} \frac{a_{t}}{\sqrt{\delta+\sum_{s=1}^{t} a_{s}}} \leq 2\sqrt{\delta+\sum_{t=1}^T a_t}$, where $0/\sqrt{0} = 0$.
\end{myLemma}

\begin{myLemma}
\label{lem:self-confident-adaptiveLR}
Let $a_1,a_2,\dots,a_T, b$ and $\bar{c}$ be non-negative real numbers and $a_t\in[0,B]$ for any $t\in[T]$. Let the step size be $c_t = \min\left\{\bar{c},\sqrt{\frac{b}{\sum_{s=1}^{t}a_s}}\right\}$ and $c_0 = \bar{c}$. Then, we have
\begin{equation}
\sum_{t=1}^T c_{t-1} a_t\leq 2\bar{c}B + 4\sqrt{b\sum_{t=1}^T a_t}.
\end{equation}
\end{myLemma}
\begin{proof}
This proof shares the same spirit with that of~\citet[Lemma~4.8]{UAI'19:FIRST-ORDER}. We assume $\sum_{t=1}^T a_t> B$, otherwise we can directly have $\sum_{t=1}^T c_{t-1}a_t \leq \bar{c} B$. When $\sum_{t=1}^T a_t> B$, let $T^\prime = \min\{t\in[T] \mid \sum_{s=1}^{t-1} a_s\geq B \}$. We can decompose the target by
\begin{align*}
\sum_{t=1}^T c_{t-1} a_t = {}&\sum_{t=1}^{T^\prime-1} c_{t-1} a_t + \sum_{t=T^\prime}^T c_{t-1} a_t.
\end{align*}
For the first term, $\sum_{t=1}^{T^\prime-1} c_{t-1} a_t  = \sum_{t=1}^{T^\prime-2} c_{t-1} a_t + c_{T^\prime-2} a_{T^\prime-1}\leq 2\bar{c}B$. For the second term, 
\begin{align*}
 \sum_{t = T'}^T c_{t-1} a_t \leq \sum_{t= T'}^T \frac{a_t\sqrt{b}}{\sqrt{\sum_{s=1}^{t-1}a_s}}\leq \sum_{t = T'}^T \frac{a_t\sqrt{b}}{\sqrt{\frac{1}{2}\sum_{s=1}^t a_s}} \leq \sum_{t = 1}^T \frac{a_t\sqrt{b}}{\sqrt{\frac{1}{2}\sum_{s=1}^t a_s}}\leq 3\sqrt{b\sum_{t=1}^T a_t},
 \end{align*}
 where the first inequality is by the definition of $c_t$ and the second inequality is due to $\sum_{s=1}^t a_s\leq B + \sum_{s=1}^{t-1} a_s \leq 2\sum_{s=1}^{t-1} a_s$ for all $t\geq T'$. The last inequality comes from Lemma~\ref{lem:self-confident}. We complete the proof by combining the two terms.
\end{proof}

\bibliography{online_learning}
\bibliographystyle{plainnat}
\end{document}